\documentclass[11pt,letterpaper,english]{article}

\usepackage[small,compact]{titlesec} 
\usepackage{amsmath}
\usepackage{amsthm}
\usepackage{amssymb}
\usepackage{graphicx}
\graphicspath{{Figures/}}
\usepackage{hyperref}
\usepackage{times}
\usepackage{setspace}
\usepackage{booktabs}
\usepackage{algorithm}
\usepackage{algpseudocode}
\usepackage{enumitem}
\usepackage{microtype}
\usepackage{tikz}
\usepackage{pgfplots}
\pgfplotsset{compat=1.17}
\usepackage{multirow}
\usepackage{soul}
\usepackage{caption}
\usepackage{subcaption}
\usepackage{float}
\usepackage{rotating}
\usepackage{array}
\usepackage{bm}
\PassOptionsToPackage{table}{xcolor}
\usepackage{xcolor}
\usepackage{colortbl} 
\definecolor{darkgrey}{gray}{0.25}

\usepackage{pifont}
\usepackage[bottom]{footmisc}
\usepackage{wrapfig}
\usepackage[makeroom]{cancel}
\usepackage[edges]{forest}
\usepackage{makecell}
\usepackage{longtable}
\usepackage{dsfont}
\usepackage{comment}
\usepackage{blkarray}
\usepackage{stackengine}
\usepackage{accents}
\usepackage[utf8]{inputenc}
\usepackage[T1]{fontenc}
\usepackage{etoolbox}
\DeclareUnicodeCharacter{0141}{\L}
\DeclareUnicodeCharacter{0142}{\l}
\DeclareUnicodeCharacter{2013}{--}
\DeclareUnicodeCharacter{2014}{---}
\DeclareUnicodeCharacter{2011}{-}
\DeclareUnicodeCharacter{2212}{-}
\usepackage{epstopdf}
\usepackage{xfrac}
\usepackage{fullpage}

\usetikzlibrary{arrows}
\usetikzlibrary{positioning}
\usepackage{bibunits}
\usepackage[sectionbib,round]{natbib}
\usepackage{ragged2e}
\def\bibsep{\smallskipamount}
\defaultbibliographystyle{abbrvnat} 
\defaultbibliography{bib} 
\usepackage{listings}
\bibpunct[, ]{(}{)}{;}{a}{}{,}

\apptocmd{\thebibliography}{\RaggedRight}{}{}

\newcommand{\cmark}{\ding{51}} % Checkmark
\newcommand{\lmark}{$\bigtriangleup$} 
\newcommand{\lbe}{\mathcal{L}_{BE}}

\newcommand*{\QED}[1][$\square$]{%
\leavevmode\unskip\penalty9999 \hbox{}\nobreak\hfill
    \quad\hbox{#1}%
}
\usepackage[toc,page]{appendix}

\DeclareMathOperator*{\argmin}{arg\,min}

\newcommand{\squishlist}{
   \begin{list}{$\bullet$}
    { \setlength{\itemsep}{0pt} \setlength{\parsep}{1pt}
      \setlength{\topsep}{1pt} \setlength{\partopsep}{1pt}
      \setlength{\leftmargin}{1.5em} \setlength{\labelwidth}{1em}
      \setlength{\labelsep}{0.5em} } }

\newcommand{\squishlisttwo}{
   \begin{list}{$\bullet$}
    { \setlength{\itemsep}{0pt} \setlength{\parsep}{0pt}
      \setlength{\topsep}{0pt} \setlength{\partopsep}{0pt}
      \setlength{\leftmargin}{1em} \setlength{\labelwidth}{1.5em}
      \setlength{\labelsep}{0.5em} } }

\newcommand{\squishend}{
    \end{list}  }

\newtheorem{thm}{Theorem}
\newtheorem{lem}[thm]{Lemma}
\newtheorem{cor}[thm]{Corollary}

\newtheorem{asmp}{Assumption}
\newtheorem{defn}{Definition}

\newtheorem{rem}{Remark}

\DeclareMathAlphabet{\mathcalligra}{T1}{calligra}{m}{n}

\usepackage{appendix}

\title{
An Empirical Risk Minimization Approach for \\Offline Inverse RL and Dynamic Discrete Choice Model}

\author{
Enoch H. Kang\thanks{We would like to thank John Rust, Kyoungseok Jang, Zikun Ye, and the anonymous reviewers at EC 2025 for their detailed feedback, which has significantly improved the paper. We also thank participants of the 2024 University of Washington Marketing PhD Workshop, the 2025 ACM Conference on Economics and Computation, 2025 AIM conference, and the Dynamic Structural Econometrics 2025 Summer School. Please ask all correspondance to: ehwkang@uw.edu, hemay@uw.edu and lalitj@uw.edu}\\
Foster School of Business, University of Washington
\and
Hema Yoganarasimhan\\
Foster School of Business, University of Washington
\and
Lalit Jain\\
Foster School of Business, University of Washington
}

\date{\today}

\begin{document}

\maketitle

\begin{abstract}
We study the problem of estimating Dynamic Discrete Choice (DDC) models, also known as offline Maximum Entropy-Regularized Inverse Reinforcement Learning (offline MaxEnt-IRL) in machine learning. The objective is to recover the reward function that governs agent behavior from offline behavior data. In this paper, we propose a globally convergent gradient-based method for solving these problems without the restrictive assumption of linearly parameterized rewards. The novelty of our approach lies in introducing the Empirical Risk Minimization (ERM) based IRL/DDC framework, which circumvents the need for explicit state transition probability estimation in the Bellman equation. Furthermore, our method is compatible with non-parametric estimation techniques such as neural networks. Therefore, the proposed method has the potential to be scaled to high-dimensional, infinite state spaces. A key theoretical insight underlying our approach is that the Bellman residual satisfies the Polyak-Łojasiewicz (PL) condition —- a property that, while weaker than strong convexity, is sufficient to ensure fast global convergence guarantees. Through a series of synthetic experiments, we demonstrate that our approach consistently outperforms benchmark methods and state-of-the-art alternatives.
\end{abstract}


% Main content sections
%\section{Introduction}

% ===== Begin input: intro.tex =====
\newpage
\section{Introduction}\label{sec:Intro}
Learning from previously collected datasets has become an essential paradigm in sequential decision-making problems where exploration during interactions with the environment is infeasible (e.g., self-driving cars, medical applications) or leveraging large-scale offline data is preferable (e.g., social science, recommendation systems, and industrial automation) \citep{levine2020offline}. However, in such cases, defining a reward function (a flow utility function) that accurately captures the underlying decision-making process is often challenging due to the unobservable/sparse rewards \citep{zolna2020offline} and complexity of real-world environments \citep{foster2021offline}. To circumvent these limitations, learning from expert demonstrations has gained prominence, motivating approaches such as Imitation Learning (IL) and offline Inverse Reinforcement Learning (offline IRL) or equivalently, Dynamic Discrete Choice (DDC) model estimation\footnote{Refer to Section \ref{sec:DDCIRLequiv} for the equivalence between Offline Maximum Entropy IRL (MaxEnt-IRL) and DDC.}.

While IL directly learns a policy by mimicking expert actions, it is susceptible to \textit{distribution shift}, i.e., when the testing environment (reward, transition function) is different from the training environment. On the other hand, offline IRL (or DDC) aims to infer the underlying reward function that best explains expert behavior. Given that this reward function is identified under a suitable normalization assumption, a new policy can be trained after a change in the environment's transition dynamics (e.g., modifications in recommendation systems) or in the reward function (e.g., marketing interventions). This capability enables offline IRL (or DDC) to be employed in counterfactual simulations, such as evaluating the effects of different policy decisions without direct experimentation. However, an imprecise reward function can lead to suboptimal policy learning and unreliable counterfactual analyses, ultimately undermining its practical utility. As a result, offline IRL (or DDC)'s key metric becomes the \textit{precision} of reward inference. %, which determines how well the inferred reward function generalizes across different environments and supports robust decision-making.

While the precise reward function estimation objective has been studied in recent offline IRL literature, theoretically guaranteed existing methods that satisfy the Bellman equation have been limited to explicitly learning a transition model \citep{zeng2023understanding}, which suffers exponential statistical complexity as the state dimension increases, or requiring deterministic transition for identification and estimation \citep{fu2017learning, cao2021identifiability, kang2026lecturenoteofflinerl}. The Dynamic Discrete Choice (DDC) literature in econometrics has separately explored the same problem \citep{rust1994structural, hotz1993conditional, aguirregabiria2007sequential, su2012constrained, adusumilli2019temporal, kristensen2021solving, geng2023data}. However, existing methodologies with theoretical precision guarantees suffer from the curse of dimensionality (computational or statistical complexity exponentially grows as state dimension increases \citep{ kristensen2021solving}) or algorithmic instabilities \citep{adusumilli2019temporal, kang2026lecturenoteofflinerl}. 
This motivates us to ask the following question:
\begin{center}
    \textit{Can we propose a scalable gradient-based method to infer rewards  (or Q$^*$ function) while provably ensuring global optimality with no assumption on reward structure/transition function knowledge?}
\end{center}

\noindent \textbf{Our contributions. }In this paper, we propose an Empirical Risk Minimization (ERM)--based gradient-based method for IRL/DDC as an inverse Q-learning method. This method provably finds the true parameter $\boldsymbol{\theta}$ for $Q^*$ estimation (up to statistical error, which diminishes at an $O(N^{-1/2})$ rate with $N$ samples) with $O(T^{-1/4})$ rate of convergence, where $T$ is the number of gradient iterations. In addition, the true reward function can be computed from the estimated $Q^\ast$ with no extra statistical or computational cost, given the estimated $Q^\ast$ function. In developing this method, we make the following technical contributions:
\begin{itemize}[leftmargin=0.3cm]
    \item We propose an empirical risk minimization (ERM) problem formulation, which we refer to as ERM-IRL in the IRL literature and ERM-DDC in the DDC literature, reflecting the shared problem. This formulation allows us to circumvent the need for explicit transition function estimation.\footnote{Transition function estimation can also be avoided in the counterfactual reasoning stage. Once the estimated reward function $\hat{r}$ is in hand, one can perform counterfactual policy evaluation and optimization by running an offline RL routine over a reward info-augmented data of $(s, a, \hat{r})$ tuples.} Notably, this formulation also allows us to conclude that imitation learning (IL) is a strictly easier problem than IRL/DDC estimation problem. 
    \item We show that the objective function of the ERM-IRL satisfies the Polyak-Łojasiewicz (PL) condition, which is a weaker but equally useful alternative to strong convexity for providing theoretical convergence guarantees. This is enabled by showing that each of its two components -- expected negative log-likelihood and mean squared Bellman error -- satisfies the PL condition\footnote{The sum of two PL functions is not necessarily PL; in the proof, we show that our case is an exception.}. 
    \item Since the mean squared Bellman error term is a solution to a strongly concave inner maximization problem \citep{dai2018sbeed, patterson2022generalized}, minimization of the ERM-IRL objective becomes a mini-max problem with two-sided PL condition \citep{yang2020global}. Using this idea, we propose an alternating gradient ascent-descent algorithm that provably converges to the true $Q^*$, which is the unique saddle point of the problem.
\end{itemize}
In addition to establishing theoretical global convergence guarantees, we demonstrate the empirical effectiveness of the algorithm through standard benchmark simulation experiments. 
Specifically, we evaluate using a series of simulations: (1) The Rust bus engine replacement problem \citep{rust1987optimal}, which is the standard framework for evaluation used in the dynamic discrete choice literature, and (2) A high-dimensional variant of the Rust bus-engine problem, where we allow a very large state space.
% and (3) OpenAI gym benchmark environment experiments with a discrete action space (Lunar Lander, Acrobot, and Cartpole) \cite{brockman2016openai}. 
In both settings, we show that our algorithm outperforms/matches the performance of existing approaches. It is particularly valuable in large state-space settings, where many of the standard algorithms become infeasible due to their need to estimate state-transition probabilities. We expect our approach to be applicable to a variety of business and economic problems where the state and action space are infinitely large, and firms/policy-makers do not have {\it a priori} knowledge of the parametric form of the reward function and/or state transitions. 

The remainder of the paper is organized as follows. In Section \ref{sec:Related}, we discuss related work in greater detail. Section \ref{sec:SetupBackgrounds} introduces the problem setup and provides the necessary background. In Section \ref{sec:ERM-IRL}, we present the ERM-IRL framework, followed by an algorithm for solving it in Section \ref{sec:Algorithm}. Section \ref{sec:Analysis} establishes the global convergence guarantees of the proposed algorithm. Finally, Section \ref{sec:Experiments} presents experimental results demonstrating the effectiveness of our approach.

% ===== End input: intro.tex =====

% ===== Begin input: related.tex =====
\section{Related works}
\label{sec:Related}

The formulations of DDC and (Maximum Entropy) IRL are fundamentally equivalent (see Section \ref{sec:DDCIRLequiv} for details). In the econometrics literature, stochastic decision-making behaviors are usually considered to come from the random utility model \citep{mcfadden2001economic}, which assumes that the effect of unobserved covariates appears in the form of additive and conditionally independent randomness in agent utilities \citep{rust1994structural}. On the other hand, in the computer science literature, stochastic decision-making behaviors are modeled as a `random choice'. That is, the assumption is that agents play a stochastic strategy (where they randomize their actions based on some probabilities). This model difference, however, is not a critical differentiator between the two literatures. The two modeling choices yield equivalent optimality equations, meaning that the inferred rewards are identical under both formulations \citep{ermon2015learning}.

The main difference between DDC and IRL methods stems from their distinct objectives. DDC's objective is to estimate the \textit{exact} reward function that can be used for subsequent counterfactual policy simulations (e.g., sending a coupon to a customer to change their reward function). Achieving such strong identifiability necessitates a strong anchor action assumption (Assumption \ref{ass:anchor}). A direct methodological consequence of this pursuit for an exact function is the requirement to solve the Bellman equation, which in turn causes significant scalability issues. On the other hand, IRL's objective is to identify \textit{a set of} reward functions that are compatible with the data. This allows for a weaker identification assumption \citep{ng1999policy} than DDC's assumption (Assumption \ref{ass:anchor}) and avoids the computational burden inherent in the DDC framework caused by the Bellman equation.

Table \ref{tab:lit} compares DDC and IRL methods based on several characteristics, which are defined here as they correspond to the table's columns. The first set of characteristics is typically satisfied by IRL methods but not by DDC methods with global optimality. \textit{One-shot optimization} indicates that a method operates on a single timescale without requiring an inner optimization loop or inner numerical integration like forward roll-outs. \textit{Transition Estimation-Free} signifies that the method avoids the explicit estimation of a transition function. A method is considered \textit{Gradient-based} if its primary optimization process relies on gradient descent, and \textit{Scalable} if it can handle state spaces of at least $20^{10}$.

Conversely, a different set of characteristics is often satisfied by DDC methods. The \textit{Bellman equation} criterion is met if the estimation procedure fits the estimated $r$-function or $Q^\ast$-function to the Bellman equation; this definition excludes occupancy-matching methods (e.g., IQ-Learn \citep{garg2021iq}, Clare \citep{yue2023clare}; see Appendix \ref{sec:occupancy}) and the semi-gradient method \citep{adusumilli2019temporal}, which minimizes the \textit{projected} squared Bellman error for linear value functions \citep{sutton2018reinforcement}. \textit{Global optimality} refers to a theoretical guarantee of convergence to the globally optimal $r$ or $Q^\ast$ beyond linear value function approximation. The $\triangle$ symbol indicates a conditional guarantee; for instance, Approximate Value Iteration (AVI) \citep{adusumilli2019temporal} is marked with a $\triangle$ due to the known instability of fitted fixed-point methods beyond the linear reward/value class \citep{wang2021instabilities,jiangoffline}. ML-IRL is also marked with a $\triangle$ because its guarantee applies only to linear reward functions. For methods that achieve global optimality, the table also lists their sample complexity. A rate of $1 / \sqrt{N}$ implies that the estimated parameter $\hat{\boldsymbol{\theta}}$ converges to the true parameter $\boldsymbol{\theta}^\ast$ at that rate, where $N$ is the sample size.

\begin{table}[h]
\centering
\caption{Comparison of DDC and IRL methods.}
\label{tab:lit}
\scalebox{0.69}{ % Adjust scale factor as needed
\renewcommand{\arraystretch}{1.5} % Good for spacing with newlines
\begin{tabular}{l c c c c c c c} % 8 columns
\toprule
\textbf{Method} & \makecell[c]{\textbf{One-shot}\\\textbf{Optimization}} & \makecell[c]{\textbf{Transition}\\\textbf{Estimation-Free}} & \makecell[c]{\textbf{Gradient-}\\\textbf{Based}} & \textbf{Scalability} & \makecell[c]{\textbf{Bellman}\\\textbf{Equation}} & \makecell[c]{\textbf{Global}\\\textbf{Optimality}} & \makecell[c]{\textbf{Statistical}\\\textbf{Complexity}} \\
\midrule
\multicolumn{8}{l}{\textit{DDC Methods}} \\
\cmidrule(r){1-8}
\makecell[l]{NFXP \\ \citep{rust1987optimal}} & & & & & \cmark & \cmark & $1/\sqrt{N}$ \\
\makecell[l]{CCP \\ \citep{hotz1993conditional}} & & & & & \cmark & \cmark & $1/\sqrt{N}$ \\
\makecell[l]{MPEC \\ \citep{su2012constrained}} & & & & & \cmark & \cmark & $1/\sqrt{N}$ \\
\makecell[l]{AVI \\ \citep{adusumilli2019temporal}} & & \cmark & & & \cmark & \lmark (unstable) & \\
\makecell[l]{Semi-gradient \\ \citep{adusumilli2019temporal}} & & \cmark & \cmark & \cmark & & & \\
\makecell[l]{RP \\ \citep{barzegary2022recursive}} & & & & \cmark & \cmark & & \\
\makecell[l]{SAmQ \\ \citep{geng2023data}} & & \cmark & & \cmark & \cmark & & \\
\addlinespace
\multicolumn{8}{l}{\textit{IRL Methods}} \\
\cmidrule(r){1-8}
\makecell[l]{BC \\ \citep{torabi2018behavioral}} & \cmark& \cmark & \cmark & \cmark & & & \\
\makecell[l]{IQ-Learn \\ \citep{garg2021iq}} & \cmark & \cmark & \cmark & \cmark & & & \\
\makecell[l]{Clare \\ \citep{yue2023clare}} & & & \cmark & \cmark & & & \\
\makecell[l]{ML-IRL \\ \citep{zeng2023understanding}} & & & \cmark & \cmark & \cmark & \lmark (Linear only) & \\
\makecell[l]{Model-enhanced AIRL \\ \citep{zhan2024model}} & & \cmark & & \cmark & & & \\
\midrule
\textbf{Ours} & \cmark & \cmark & \cmark & \cmark & \cmark & \cmark & $1/\sqrt{N}$ \\
\bottomrule
\end{tabular}
} % End scalebox
\end{table}

\subsection{Dynamic discrete choice model estimation literature}

The seminal paper by Rust \citep{rust1987optimal} pioneered this literature, demonstrating that a DDC model can be solved by solving a maximum likelihood estimation problem that runs above iterative dynamic programming. As previously discussed, this method is computationally intractable as the size of the state space increases. 

\cite{hotz1993conditional} introduced a method which is often called the two-step method conditional choice probability (CCP) method, where the CCPs and transition probabilities estimation step is followed by the reward estimation step. The reward estimation step avoids dynamic programming by combining simulation with the insight that differences in value function values can be directly inferred from data without solving Bellman equations. However, simulation methods are, in principle, trajectory-based numerical integration methods that also suffer scalability issues. Fortunately, we can sometimes avoid simulation altogether by utilizing the problem structure, such as regenerative/terminal actions (known as finite dependence \citep{arcidiacono2011conditional}). Still, this method requires explicit estimation of the transition function, which is not the case in our paper. This paper established an insight that there exists a one-to-one correspondence between the CCPs and the differences in $Q^*$-function values, which was formalized as the identification result by \cite{magnac2002identifying}. 

\cite{su2012constrained} propose that we can avoid dynamic programming or simulation by formulating a nested linear programming problem with Bellman equations as constraints of a linear program. This formulation is based on the observation that Bellman equations constitute a convex polyhedral constraint set. While this linear programming formulation significantly increases the computation speed, it is still not scalable in terms of state dimensions.

As the above methods suffer scalability issues, methods based on parametric/nonparametric approximation have been developed. Parametric policy iteration \citep{benitez2000comparison} and sieve value function iteration \citep{arcidiacono2013approximating} parametrize the value function by imposing a flexible functional form. \citet{kristensen2021solving} also proposed methods that combine smoothing of the Bellman operator with sieve-based approximations, targeting the more well-behaved integrated or expected value functions to improve computational performance. However, standard sieve methods that use tensor product basis functions, such as polynomials, can suffer from a computational curse of dimensionality, as the number of basis functions required for a given accuracy grows exponentially with the number of state variables. \cite{norets2012estimation} proposed that neural network-based function approximation reduces the computational burden of Markov Chain Monte Carlo (MCMC) estimation, thereby enhancing the efficiency and scalability. Also leveraging Bayesian MCMC techniques, \cite{imai2009bayesian} developed an algorithm that integrates the DP solution and estimation steps, reducing the computational cost associated with repeatedly solving the underlying dynamic programming problem to be comparable to static models, though it still requires transition probabilities for calculating expectations. \citet{arcidiacono2016estimation} formulates the problem in continuous time, where the sequential nature of state changes simplifies calculations. However, these continuous-time models still rely on specifying transition dynamics via intensity matrices. \cite{geng2020deep} proposed that the inversion principle of \cite{hotz1993conditional} enables us to avoid reward parameterization and directly (non-parametrically) estimate value functions, along with solving a much smaller number of soft-Bellman equations, which do not require reward parametrization to solve them. \cite{barzegary2022recursive} and \cite{geng2023data} independently proposed state aggregation/partition methods that significantly reduce the computational burden of running dynamic programming with the cost of optimality. While \cite{geng2023data} uses $k$-means clustering \citep{kodinariya2013review, sinaga2020unsupervised}, \citet{barzegary2022recursive} uses recursive partitioning (RP). As discussed earlier, combining approximation with dynamic programming induces unstable convergence except when the value function is linear in state \citep{jiangoffline}. 

\cite{adusumilli2019temporal} proposed how to adapt two popular temporal difference (TD)-based methods (an approximate dynamic programming-based method and a semi-gradient descent method based on \cite{tsitsiklis1996analysis}) for DDC. As discussed earlier, approximate dynamic programming-based methods are known to suffer from a lack of provable convergence beyond linear reward models \citep{jiangoffline, tsitsiklis1996feature, van2018deep,  wang2021instabilities}\footnote{In fact, \textit{``fitted fixed point methods can diverge even when all of the following hold: infinite data, perfect optimization with infinite computation, 1-dimensional function class (e.g., linear) with realizability (Assumption} \ref{ass:realizability}\textit{)}... \textit{(the instability is) not merely a theoretical construction: deep RL algorithms are known for their instability and training divergence...}'' \citep{jiangoffline}}; the semi-gradient method is a popular, efficient approximation method that has theoretical assurance of convergence to \textit{projected} squared Bellman error minimizers for linear reward/value functions \citep{sutton2018reinforcement}. \citet{feng2020global} showed global concavity of value function under certain transition functions and monotonicity of value functions in terms of one-dimensional state, both of which are easily satisfied for applications in social science problems. However, those conditions are limitedly satisfied for the problems with larger dimensional state space. 

\subsection{Offline inverse reinforcement learning literature} \; The most widely used inverse reinforcement learning model, Maximum-Entropy inverse reinforcement learning (MaxEnt-IRL), assumes that the random choice happens due to agents choosing the optimal policy after penalization of the policy by its Shannon entropy \citep{ermon2015learning}. In addition to the equivalence of MaxEnt-IRL to DDC (See \cite{ermon2015learning} and Web Appendix $\S$\ref{sec:DDCIRLequiv}), the identifiability condition for DDC \citep{magnac2002identifying} was rediscovered by \cite{cao2021identifiability} for MaxEnt-IRL. \cite{zeng2023understanding} proposes a two-step maximum likelihood-based method that can be considered as a conservative version of CCP method of \cite{hotz1993conditional}\footnote{When there is no uncertainty in the transition function, approximated trajectory gradient of Offline IRL method degenerates to forward simulation-based gradient in CCP estimator method of \cite{hotz1993conditional}.}.  
Despite that their method is proven to be convergent, it requires the explicit estimation of the transition function, and its global convergence was only proven for linear reward functions.

\cite{finn2016connection} showed that a myopic\footnote{See \cite{cao2021identifiability} for more discussion on this.} version of MaxEnt-IRL can be solved by the Generative Adversarial Network (GAN) training framework \citep{goodfellow2020generative}. \cite{fu2017learning} extended this framework to a non-myopic version, which is proven to recover the reward function and the value function (up to policy invariance) but only under deterministic transitions. Note that GAN approaches identify rewards only up to policy invariance \citep{ng1999policy}, which implies that counterfactual analysis is impossible. The GAN approach has been extended to $Q$-estimation methods that use fixed point iteration \citep{geng2020deep, geng2023data}. \cite{ni2021f} has shown that the idea of training an adversarial network can also be used to calculate the gradient direction for minimizing the myopic version of negative log likelihood\footnote{Minimizing negative log-likelihood is equivalent to minimizing KL divergence. See the Proof of Lemma \ref{lem:minMLE}.}. \cite{zhan2024model} proposed a GAN approach that is provably equivalent to negative log likelihood minimization objective without satisfying the Bellman equation (i.e., behavioral cloning \citep{torabi2018behavioral}). Indeed, GAN approaches are known to work well for Imitation Learning (IL) tasks \citep{zare2024survey}. However, as discussed earlier, solving for IL does not allow us to conduct a counterfactual simulation. 

A family of methods starting from \cite{ho2016generative} tries to address the inverse reinforcement learning problem from the perspective of occupancy matching, i.e., finding a policy that best matches the behavior of data. \cite{garg2021iq} proposed how to extend the occupancy matching approach of \cite{ho2016generative} to directly estimate $Q$-function instead of $r$. Given the assumption that the Bellman equation holds, this approach allows a simple gradient-based solution, as the occupancy matching objective function they maximize becomes concave. \cite{yue2023clare} modifies \cite{ho2016generative} to conservatively deal with the uncertainty of transition function. Recently, occupancy matching-based inverse reinforcement learning has been demonstrated at a planetary scale in Google Maps, delivering significant global routing improvements \cite{barnes2023massively}.  Despite their simplicity and scalability, one caveat of occupancy matching approaches is that the estimated $Q$ from solving the occupancy matching objective may not satisfy the Bellman equation (See Appendix \ref{sec:occupancy}).
This also implies that computing $r$ from $Q$ using the Bellman equation is not a valid approach.

% ===== End input: related.tex =====

% ===== Begin input: setup.tex =====
\section{Problem set-up and backgrounds}\label{sec:SetupBackgrounds}
%In this section, we introduce the Offline Inverse Reinforcement Learning (Offline IRL) problem and the Dynamic Discrete Choice (DDC) model estimation problem. We also establish their equivalence. 

We consider a single-agent Markov Decision Process (MDP) defined as a tuple $\left(\mathcal{S}, \mathcal{A}, P, \nu_0, r, \beta\right)$ where $\mathcal{S}$ denotes the state space and $\mathcal{A}$ denotes a finite action space, $P \in \Delta_{\mathcal{S}}^{\mathcal{S} \times \mathcal{A}}$ is a Markovian transition kernel, $\nu_0 \in \Delta_{\mathcal{S}}$ is the initial state distribution over $\mathcal{S}$,  $r \in \mathbb{R}^{\mathcal{S} \times \mathcal{A}}$ is a deterministic reward function and $\beta \in(0,1)$ a discount factor.  
Given a stationary Markov policy $\pi \in \Delta_{\mathcal{A}}^{\mathcal{S}}$, an agent starts from initial state $s_0$ and takes an action $a_h\in \mathcal{A}$ at state $s_h\in \mathcal{S}$ according to $a_h\sim\pi\left(\cdot \mid s_h\right)$ at each period $h$. Given an initial state $s_0\sim \nu_0$, we define the distribution of state-action sequences for policy $\pi$ over the sample space $(\mathcal{S} \times \mathcal{A})^{\infty}=\left\{\left(s_0, a_0, s_1, a_1, \ldots\right): s_h \in \mathcal{S}, a_h \in \mathcal{A}, h \in \mathbb{N}\right\}$ as $\mathbb{P}_{\nu_0,\pi}$. 
We also use $\mathbb{E}_{\nu_0,\pi}$ to denote the expectation with respect to $\mathbb{P}_{\nu_0,\pi}$.

\subsection{Setup: Maximum Entropy-Inverse Reinforcement Learning (MaxEnt-IRL) %\protect\footnote{Entropy regularized IRL's equivalence with dynamic discrete choice (DDC) model is discussed in Appendix \ref{sec:DDCIRLequiv}.}
} 
\label{sec:IRLIntro}

 Following existing literature \citep{geng2020deep, fu2017learning, ho2016generative}, we consider the \textit{entropy-regularized} optimal policy, which is defined as
\begin{equation}
\pi^*:=\operatorname{argmax}_{\pi \in \Delta_{\mathcal{A}}^{\mathcal{S}}}\mathbb{E}_\pi\bigl[\sum_{h=0}^{\infty} \beta^h \bigl(r(s_h, a_h)+\lambda\mathcal{H}(\pi(\cdot \mid s_h))\bigr)\bigr]
\end{equation}
where $\mathcal{H}$ denotes the Shannon entropy and $\lambda$ is the regularization coefficient. Throughout, we make the following assumption on the agent's decisions.

\begin{asmp}\label{ass:IRLoptimaldecision} When interacting with the MDP $\left(\mathcal{S}, \mathcal{A}, P, \nu_0, r, \beta\right)$, each agent follows the entropy-regularized optimal stationary policy $\pi^*$.
\end{asmp}

\noindent Throughout the paper, we use $\lambda = 1$, the setting which is equivalent to dynamic discrete choice (DDC) model with mean zero T1EV distribution (see Web Appendix $\S$\ref{sec:DDCIRLequiv} for details); all the results of this paper easily generalize to other values of $\lambda$. Given $\pi^*$, we define the \textit{value function} $V^*$ as:
$$
V^*(s) := \mathbb{E}_{\pi^*} \biggl[\sum_{h=0}^\infty \beta^h 
\bigl(r(s_h, a_h) + \mathcal{H}(\pi^*(\cdot \mid s_h))\bigr) \biggm| s_0 = s\biggr].
$$
Similarly, we define the $Q^*$ function as follows:
$$Q^*(s, a):=r\left(s, a\right)+\beta \cdot \mathbb{E}_{s^\prime \sim P(s,a)}\left[{V}^*\left(s^\prime\right)\mid s, a\right]$$
Given state $s$ and policy $\pi^*$, let $\mathbf{q} = [q_1 \ldots q_{|\mathcal{A}|}]$ denote the probability distribution over the action space $\mathcal{A}$, such that:
\begin{align}  
    q_a= \frac{\exp \left({Q^*(s, a)}\right)}{\sum_{a^\prime\in \mathcal{A}} \exp \left({Q^*(s, a^\prime)}\right)} \text{ for } a\in \mathcal{A} \notag
\end{align}
Then, according to Assumption \ref{ass:IRLoptimaldecision}, the value function $V^*$ must satisfy the recursive relationship defined by the \textit{Bellman equation} as follows:
\begin{align}
&V^*(s)=\max_{\mathbf{q} \in \triangle_\mathcal{A}}\{
\mathbb{E}_{a\sim\mathbf{q}}\bigl[r(s, a) +\beta \mathbb{E}_{s^\prime\sim P(s,a)}[V^*(s^\prime)\mid s, a]\bigr] +\mathcal{H}(\mathbf{q})\} \notag
\end{align}
Further, we can show that (see Web Appendix $\S$\ref{sec:IRLentropy}):
\begin{align}
    V^*(s)&=\ln \left[\sum_{a\in \mathcal{A}}\exp\left(Q^*(s, a)\right)\right] \notag
    \\
    \pi^*(a\mid s) &= \frac{\exp \left({Q^*(s, a)}\right)}{\sum_{a^\prime\in \mathcal{A}} \exp \left({Q^*(s, a^\prime)}\right)} \text{ for } a\in \mathcal{A} \notag
    \\
     Q^\ast(s,a)&=r(s, a)+\beta \cdot \mathbb{E}_{s^\prime \sim P(s, a)}\bigl[\log\sum_{a^\prime\in\mathcal{A}}\exp(Q^*(s^\prime,a^\prime)) \mid s, a\bigr]\label{eq:QBellman}
\end{align}
%That is, the optimal stationary policy $\pi^*$ satisfies $$$$ 
Throughout, we define a function $V_Q$ as 
\begin{align}
    V_Q(s)&:=\ln \left[\sum_{a\in \mathcal{A}}\exp\left(Q(s, a)\right)\right] \notag
\end{align}
where $ V_{Q^\ast} = V^\ast$.

\subsection{Setup: Dynamic Discrete Choice (DDC) model}
Following the literature \citep{rust1994structural, magnac2002identifying}, we assume that the reward the agent observes at state $s\in\mathcal{S}$ and $a\in\mathcal{A}$ can be expressed as $r(s,a) + \epsilon_a$, 
where $\epsilon_a\overset{i.i.d.}{\sim}  G(\delta, 1)$ is the random part of the reward, where $G$ is Type 1 Extreme Value (T1EV) distribution (i.e., Gumbel distribution)\footnote{This reward form is often referred to as additive and conditionally independent form.}. The mean of $G(\delta, 1)$ is $\delta + \gamma$, where $\gamma$ is the Euler constant. Throughout the paper, we use $\delta=-\gamma$, which makes $G$ a mean 0 distribution.\footnote{All the results of this paper easily generalize to other values of $\delta$.} Under this setup, we consider the optimal stationary policy and its corresponding value function defined as
$$\pi^\ast:=\operatorname{argmax}_{\pi \in \Delta_{\mathcal{A}}^{\mathcal{S}}}\mathbb{E}_{\nu_0,\pi,G}\left[\sum_{h=0}^{\infty} \beta^h( r\left(s_h, a_h\right)+\epsilon_{a_h})\right]$$
$$V^\ast(s):=\operatorname{max}_{\pi \in \Delta_{\mathcal{A}}^{\mathcal{S}}}\mathbb{E}_{\nu_0,\pi,G}\left[\sum_{h=0}^{\infty} \beta^h( r\left(s_h, a_h\right)+\epsilon_{a_h})\mid s_0=s\right]$$
Throughout, we make the following assumption on agent's decisions.

\begin{asmp}\label{ass:DDCoptimaldecision} When interacting with the MDP $\left(\mathcal{S}, \mathcal{A}, P, \nu_0, r, \beta\right)$, agent follows the optimal stationary policy $\pi^*$.
\end{asmp}
According to Assumption \ref{ass:DDCoptimaldecision}, the value function $V^\ast$ must satisfy the recursive relationship, often called the Bellman equation, as follows:
\begin{align}
V^\ast(s)&=\mathbb{E}_{\boldsymbol{\epsilon}}\left[\max _{a \in \mathcal{A}}\left\{r(s, a)+\epsilon_a+\beta \cdot \mathbb{E}\left[V^\ast\left(s^{\prime}\right) \mid s, a\right]\right\}\right] \notag
\\
&=\ln \left[\sum_{a\in \mathcal{A}}\exp\left(r\left(s, a\right)+\beta \cdot \mathbb{E}\left[V^\ast\left(s^\prime\right)\mid s, a\right]\right)\right]  \notag
\end{align}
where the second equality is from Lemma \ref{lem:GumbelMax}.
We further define the $Q^\ast$ function as
$$
Q^*(s, a):=r(s, a)+\beta \cdot \mathbb{E}_{s^{\prime} \sim P(s, a)}\left[V^*\left(s^{\prime}\right) \mid s, a\right]
$$
We can show that (see Web Appendix $\S$\ref{sec:SingleDDC}):
\begin{align}
    V^*(s)&=\ln \left[\sum_{a\in \mathcal{A}}\exp\left(Q^*(s, a)\right)\right] \notag
    \\
    \pi^*(a\mid s) &= \frac{\exp \left({Q^*(s, a)}\right)}{\sum_{a^\prime\in \mathcal{A}} \exp \left({Q^*(s, a^\prime)}\right)} \text{ for } a\in \mathcal{A} \notag
    \\
      Q^*(s,a) &=r(s, a)+\beta \cdot \mathbb{E}_{s^\prime \sim P(s, a)}    \bigl[\log\sum_{a^\prime\in\mathcal{A}}\exp(Q^*(s^\prime,a^\prime)) \mid s, a\bigr]\label{eq:QBellmanDDC}
\end{align}

\subsection{DDC -- MaxEnt-IRL Equivalence and unified problem statement}\label{sec:DDCIRLequiv}

The Bellman equations of MaxEnt-IRL with \(\lambda = 1\) (Equation \eqref{eq:QBellman}) and DDC with \(\delta = -\gamma\) (Equation \eqref{eq:QBellmanDDC}) are equivalent. Consequently, the optimal \(Q^*\) values obtained from solving these Bellman equations are the same for both MaxEnt-IRL and DDC. Furthermore, the optimal policy induced by \(Q^*\) is identical in both frameworks. Therefore, we can infer that solving one problem is equivalent to solving the other. Throughout, all the discussions we make for \(\lambda = 1\) in MaxEnt-IRL and \(\delta = -\gamma\) in DDC extend directly to any \(\lambda \neq 1\) and \(\delta \neq -\gamma\), respectively. This equivalence is a folk theorem that was first observed in \citep{ermon2015learning}.

In both settings, the goal is to recover the underlying reward function $r$ that explains an agent's demonstrated behavior. Given the equivalence between them, we can now formulate a \textit{unified problem statement} that encompasses both Offline Maximum Entropy Inverse Reinforcement Learning (Offline MaxEnt-IRL) and the Dynamic Discrete Choice (DDC) model estimation. 

To formalize this, consider a dataset consisting of state-action-next state sequences collected from an agent's behavior:  
$\mathcal{D}:=\left(\left(s_0, a_0, s_0^\prime\right), \left(s_1, a_1,s_1^\prime\right), \ldots,  \left(s_N, a_N,s_N^\prime\right)\right)$. 
Following Assumption \ref{ass:IRLoptimaldecision}, we assume that the data was generated by the agent playing the optimal policy $\pi^*$ when interacting with the MDP $\left(\mathcal{S}, \mathcal{A}, P, \nu_0, r, \beta\right)$.
\begin{defn}[The unified problem statement]\label{def:IRLproblem}

    The objective of offline MaxEnt-IRL and DDC can be defined as learning a function $\hat{r}\in\mathcal{R} \subseteq\mathbb{R}^{\bar{\mathcal{S}}\times \mathcal{A}}$ that minimizes the mean squared prediction error with respect to data distribution (i.e., expert policy's state-action distribution) from offline data $\mathcal{D}$ such that: 
\begin{align}
    \underset{\hat{r}\in \mathcal{R} \subseteq \mathbb{R}^{\bar{S} \times \mathcal{A}}}{\operatorname{argmin}} \; \mathbb{E}_{(s,a)\sim \nu_0, \pi^*}[(\hat{r}(s,a)-r(s,a))^2] \label{eq:rObjective}    
\end{align}
where 
$$\overline{\mathcal{S}}:=\left\{s \in \mathcal{S} \mid \mathbb{P}_{\nu_0, \pi}\left(\{s_h=s\}\right)>0 \text { for some } h \geq 0\right\}$$ defines the expert policy's coverage, which consists of states that are reachable with nonzero probability under the expert's optimal policy \(\pi^*\). \footnote{For every \(s \in \bar{\mathcal{S}}\), every action \(a \in \mathcal{A}\) occurs with probability strictly greater than zero, ensuring that the data sufficiently covers the relevant decision-making space.}  
\end{defn}
\noindent Restricting to \(\bar{\mathcal{S}}\) is essential, as the dataset \(\mathcal{D}\) only contains information about states visited under \(\pi^*\). Inferring rewards beyond this set would be ill-posed due to a lack of data, making \(\bar{\mathcal{S}}\) the natural domain for learning. Similarly, Computing MSE using the expert policy's state-action distribution is natural since the goal is to recover the reward function that explains the expert’s behavior. 
\\
\;
\\
\noindent\textbf{Remark (Counterfactual policy optimization without estimating $P$).}
Given logged interaction data \(\mathcal{D} = \{(s_i, a_i, s'_i)\}_{i=1}^N\) and a recovered reward function \(\hat{r}\), suppose the analyst wishes to evaluate the effect of a counterfactual intervention on the reward—e.g., modifying incentives or preferences—which induces a new estimated reward function \(\hat{r}_{\mathrm{cf}}(s, a)\), such as $
\hat{r}_{\mathrm{cf}}(s, a)=\hat{r}(s, a)+\Delta(s, a)
$, where $\Delta$ encodes the intervention. Using the augmented dataset
$$
\hat{\mathcal{D}}_{\mathrm{cf}}=\left\{\left(s_i, a_i, \hat{r}_{\mathrm{cf}}\left(s_i, a_i\right), s_i^{\prime}\right)\right\}_{i=1}^N,
$$
any standard modern offline-RL algorithm (e.g., Conservative Q-Learning \citep{kumar2020conservative}) can be applied to $\hat{\mathcal{D}}_{\mathrm{cf}}$ to obtain the counterfactual optimal policy \textbf{without requiring an explicit estimate of the transition function $P$}. These methods operate via empirical Bellman backups over observed transitions and are fully model-free.

In contrast, when the counterfactual intervention alters the $P$ itself, one must estimate $P$, model the intervention accordingly, and perform model-based planning (e.g., \citep{sutton2018reinforcement}).

\subsection{Identification}
 
As we defined in Definition \ref{def:IRLproblem}, our goal is to learn the agent's reward function $r(s, a)$ given offline data $\mathcal{D}$.
However, without additional assumptions on the reward structure, this problem is ill-defined because many reward functions can explain the optimal policy \citep{fu2017learning, ng1999policy}. To address this issue, following the DDC literature  \citep{rust1994structural, magnac2002identifying, hotz1993conditional} and recent IRL literature \citep{geng2020deep}, we assume that there is an \textit{anchor action} $a_s$ in each state $s$, such that the reward for each of state-anchor action combination is known.

\begin{asmp}\label{ass:anchor} For all $s\in\mathcal{S}$, there exists an action $a_s\in \mathcal{A}$ such that $r(s,a_s)$ is known.
\end{asmp} 

\noindent Note that the optimal policy remains the same irrespective of the choice of the anchor action $a_s$ and the reward value at the anchor action $r(s, a_s)$ (at any given $s$). As such, Assumption \ref{ass:anchor} only helps with identification and does not materially affect the estimation procedure. That is, for the sake of choosing an arbitrary reward function that is compatible with the optimal policy (i.e., the IRL objective), an arbitrary choice of $a_s$ and the $r(s,a_s)$ value is justified.

In Lemma \ref{thm:MagnacThesmar}, we formally establish that Assumptions \ref{ass:IRLoptimaldecision} and \ref{ass:anchor} uniquely identify $Q^\ast$ and $r$. See Web Appendix $\S$\ref{sec:PfMagnac} for the proof.

\begin{lem}[\cite{magnac2002identifying}]\label{thm:MagnacThesmar}%\todo{Add proof}

Given discount factor $\beta$, transition kernel $P\in \Delta_\mathcal{S}^{\mathcal{S}\times \mathcal{A}}$ and optimal policy $\pi^*\in\Delta_\mathcal{A}^{\mathcal{S}}$, under Assumptions \ref{ass:IRLoptimaldecision} and \ref{ass:anchor}, the solution to the following system of equations:
\begin{equation}
\left\{
\begin{array}{l}
    \dfrac{\exp({Q}\left(s,a\right))}{\sum_{a^\prime\in \mathcal{A}} \exp({Q}\left(s,a^\prime\right))} = \pi^*(\;a
    \mid s) \;\;\; \forall s\in \bar{\mathcal{S}}, a\in\mathcal{A}
    \\[1em]
    r(s, a_s)+\beta \cdot \mathbb{E}_{s^{\prime} \sim P(s, a_s)}\left[V_Q(s^\prime) \mid s, a_s\right]  = Q(s, a_s)          \;\; \; \forall s\in \bar{\mathcal{S}} \notag
\end{array}
\right.
\label{eq:HotzMillereqs}
\end{equation}
identifies $Q^*$ up to $s \in \overline{\mathcal{S}}, a \in \mathcal{A}$. Furthermore, $r$ is obtained up to $\forall s \in \overline{\mathcal{S}}, a \in \mathcal{A}$ by solving:
\begin{align}
    &r(s,a) = Q^\ast(s, a) - \beta \cdot \mathbb{E}_{s^{\prime} \sim P(s, a)}\bigl[V_{Q^\ast}(s^\prime) \mid s, a\bigr]. \label{eq:rbyBellman}
\end{align}
for all $s \in \overline{\mathcal{S}}, a \in \mathcal{A}$.
\end{lem}

In the first part of the theorem, we show that, after constraining the reward functions for anchor actions, we can recover the unique $Q^\ast$-function for the optimal policy from the observed choices and the Bellman equation for the anchor-action (written in terms of log-sum-exp of $Q$-values). The second step follows naturally, where we can show that reward functions are then uniquely recovered from $Q^*$-functions using the Bellman equation.

\subsection{Bellman error and Temporal difference (TD) error}\label{sec:BE&TD}
There are two key concepts used for describing a gradient-based algorithm for IRL/DDC: the Bellman error and the Temporal Difference (TD) error. In this section, we define each of them and discuss their relationship. We start by  defining $\mathcal{Q}=\left\{Q: \mathcal{S} \times \mathcal{A} \rightarrow \mathbb{R} \mid\|Q\|_{\infty}<\infty\right\}$. By \citet{rust1994structural}, $\beta<1$ (discount factor less than one) implies $Q^* \in \mathcal{Q}$. Next, we define the \textit{Bellman operator} as $\mathcal{T}: \mathcal{Q} \mapsto \mathcal{Q}$ as follows:
$$\mathcal{T}Q(s, a) \notag := r(s, a) + \beta \cdot \mathbb{E}_{s' \sim P(s, a)} \bigl[ V_Q(s') \bigr]
$$
According to the Bellman equation shown in Equation \eqref{eq:QBellman}, $Q^*$ satisfies $\mathcal{T}Q^*(s, a) - Q^*(s, a)=0$; in fact, $Q^*$ is the unique solution to $\mathcal{T}Q(s, a) - Q(s, a)=0$; see \citep{rust1994structural}. 
Based on this observation, we define the following notions of error. 

\begin{defn}\label{def:BE}
    We define the \textit{Bellman error} for $Q\in\mathcal{Q}$ at $(s,a)$ as $\mathcal{T}Q(s, a) - Q(s, a)$. Furthermore, we define the \textit{Squared Bellman error} and the \textit{Expected squared Bellman error}  as $$\mathcal{L}_{\text{BE}}(Q)(s, a) := \left( \mathcal{T}Q(s, a) - Q(s, a) \right)^2$$
$$
\overline{
\mathcal{L}_{\text{BE}}}(Q) = \mathbb{E}_{(s, a) \sim \pi^*,\, \nu_0} \left[ \mathcal{L}_{\text{BE}}(Q)(s, a) \right]
$$
\end{defn}

\noindent In practice, we don't have direct access to $\mathcal{T}$ unless we know (or have a consistent estimate of) the transition kernel $P\in\Delta_\mathcal{S}^{\mathcal{S}\times\mathcal{A}}$. Instead, we can compute an empirical \textit{Sampled Bellman operator} $\hat{\mathcal{T}}$, defined as 
$$
\hat{\mathcal{T}}Q(s, a, s') = r(s, a) + \beta \cdot V_Q(s^\prime).
$$
\begin{defn}
    We define \textit{Temporal-Difference (TD) error} for $Q$ at the transition $(s, a, s^\prime)$, \textit{Squared TD error}, and \textit{Expected squared TD error} as follows:
\begin{align}
    &\delta_Q(s, a, s'):=\hat{\mathcal{T}}Q(s, a, s') - Q(s, a) \notag
    \\
 &\mathcal{L}_{\text{TD}}(Q)(s,a,s^\prime):= \left(\hat{\mathcal{T}}Q(s, a, s') - Q(s, a)\right)^2 \notag
 \\
 &\overline{
\mathcal{L}_{\text{TD}}}(Q) := \mathbb{E}_{(s, a) \sim \pi^*,\, \nu_0} \left[ \mathbb{E}_{s' \sim P(s, a)} \left[ \mathcal{L}_{\text{TD}}(Q)(s, a, s') \right] \right] \notag
\end{align}
\end{defn}

\noindent Lemma \ref{lem:expTD=BE} states the relationship between the TD error terms and Bellman error terms.
\begin{lem}[Expectation of TD error is equivalent to BE error]
\label{lem:expTD=BE}
    $$ \mathbb{E}_{s' \sim P(s, a)} \left[ \hat{\mathcal{T}}Q(s, a, s') \right] = \mathcal{T}Q(s, a) 
    $$
    $$
    \mathbb{E}_{s' \sim P(s, a)} \left[ \delta_Q(s, a, s')\right] =\mathcal{T}Q(s, a) - Q(s, a). $$
\end{lem}

%%%%%%%%%%%%%%%%%%%%%%%%%%%%%%%%%%%%%%%%%%%%%%

% ===== End input: setup.tex =====

% ===== Begin input: ERM.tex =====
\section{ERM-IRL (ERM-DDC) framework}\label{sec:ERM-IRL}
\subsection{Identification via expected risk minimization}
Given Lemma \ref{thm:MagnacThesmar}, we would like to find the unique $\hat{Q}$ that
satisfies  
\begin{equation}
\left\{
\begin{array}{l}
    \dfrac{\exp({\hat{Q}}\left(s,a\right))}{\sum_{a^\prime\in \mathcal{A}} \exp({\hat{Q}}\left(s,a^\prime\right))} = \pi^*(a
    \mid s) \; \; \; \forall s\in \bar{\mathcal{S}}, a \in \mathcal{A}
    \\[1em]
    r(s, a_s)+\beta \cdot \mathbb{E}_{s^{\prime} \sim P(s, a_s)}\left[\log(\sum_{a^\prime\in\mathcal{A}}\exp \hat{Q}(s^\prime, a^\prime)) \mid s, a_s\right]-\hat{Q}(s, a_s)=0 \;\;\; \forall s\in \bar{\mathcal{S}}
    
\end{array}\tag{Equation \eqref{eq:HotzMillereqs}}
\right. 
\end{equation}
where $\bar{\mathcal{S}}$ (the reachable states from $\nu_0$, $\pi^\ast$) was defined as:
$$
\bar{\mathcal{S}}=\left\{s \in \mathcal{S} \mid \operatorname{Pr}\left(s_t=s \mid s_0 \sim \nu_0, \pi^*\right)>0 \text { for some } t \geq 0\right\} 
$$
Now note that:
  \begin{align}
     &\left\{Q \in \mathcal{Q} \mid\hat{p}_{Q}(\;\cdot
    \mid s) = \pi^*(\;\cdot
    \mid s)\quad  \forall s\in\bar{\mathcal{S}}\quad\text{a.e.}\right\} \notag
    \\
    &=\underset{Q\in \mathcal{Q}}{\arg\max } \; \;\mathbb{E}_{(s, a)\sim \pi^*, \nu_0}  \left[\log\left(\hat{p}_{Q}(\;\cdot
    \mid s)\right)\right] \tag{$\because$ Lemma \ref{lem:minMLE}}
    \\
    &=\underset{Q\in \mathcal{Q}}{\arg\min } \; \;\mathbb{E}_{(s, a)\sim \pi^*, \nu_0}  \left[-\log\left(\hat{p}_{Q}(\;\cdot
    \mid s)\right)\right] \notag
    \end{align}
and 
    \begin{align}
     &\left\{Q \in \mathcal{Q} \mid\mathcal{L}_{BE}(Q)(s,a_s) = 0\quad  \forall s\in\bar{\mathcal{S}}\right\} \notag
    \\
    &=\underset{Q\in \mathcal{Q}}{\arg\min } \; \;\mathbb{E}_{(s, a)\sim \pi^*, \nu_0}  \left[\mathbf{1}_{a = a_s} \mathcal{L}_{BE}(Q)(s,a)\right] \notag
    \end{align}
Therefore equations \eqref{eq:HotzMillereqs} becomes 

\begin{equation}
\left\{
\begin{array}{l}
    \hat{Q} \in \underset{Q\in \mathcal{Q}}{\arg\min } \; \;\mathbb{E}_{(s, a)\sim \pi^*, \nu_0}  \left[-\log\left(\hat{p}_{Q}(\;\cdot
    \mid s)\right)\right] 
    \\[1em]
     \hat{Q} \in \underset{Q\in \mathcal{Q}}{\arg\min } \; \;\mathbb{E}_{(s, a)\sim \pi^*, \nu_0}  \left[\mathbf{1}_{a = a_s} \mathcal{L}_{BE}(Q)(s,a)\right]
    
\end{array}\label{eq:modifiedHotz}
\right. 
\end{equation}
We now propose a one-shot Empirical Risk Minimization framework (ERM-IRL/ERM-DDC) to solve the Equation \eqref{eq:modifiedHotz}. First, we recast the IRL problem as the following \textit{expected risk} minimization problem under an infinite data regime. 
\begin{defn}[Expected risk minimization problem] The expected risk minimization problem is defined as the problem of finding $Q$ that minimizes  the expected risk $\mathcal{R}_{exp}(Q)$, which is defined as
\begin{align}
  &\mathcal{R}_{exp}(Q):= \mathbb{E}_{(s, a)\sim \pi^*, \nu_0}  \left[\mathcal{L}_{NLL}(Q)(s,a) + \mathbf{1}_{a = a_s} \mathcal{L}_{BE}(Q)(s,a)\right] \! \notag \noindent
  \\
  & = \mathbb{E}_{(s, a)\sim \pi^*, \nu_0}\bigl[-\log\left(\hat{p}_{Q}(a
\mid s)\right) +  \mathbf{1}_{a = a_s} \left( \mathcal{T}Q(s, a) - Q(s, a) \right)^2 \bigr] \label{eq:mainopt}
\end{align}
\noindent where $a_s$ is defined in Assumption \ref{ass:anchor}.
\end{defn}

\noindent\textbf{Remark.} The joint minimization of the Negative Log Likelihood (NLL) term and Bellman Error (BE) term is the key novelty in our approach. Prior work on the IRL and DDC literature \citep{hotz1993conditional, zeng2023understanding} typically minimizes the log-likelihood of the observed choice probabilities (the NLL term), given observed or estimated state transition probabilities. The standard solution is to first estimate/assume state transition probabilities, then obtain estimates of future value functions, plug them into the choice probability, and then minimize the NLL term. In contrast, our recast problem avoids the estimation of state-transition probabilities and instead jointly minimizes the NLL term along with the Bellman error term. This is particularly helpful in large state spaces since the estimation of state-transition probabilities can be infeasible/costly in such settings. In Theorem \ref{thm:mainopt}, we show that the solution to our recast problem in Equation \eqref{eq:mainopt} identifies the reward function. 

\begin{thm}[Identification via expected risk minimization]
\label{thm:mainopt} 
\;
\\
The solution to the expected risk minimization problem (Equation \eqref{eq:mainopt})
uniquely identifies $Q^\ast$ up to $s\in\bar{\mathcal{S}}
$ and $a \in \mathcal{A}$, i.e., finds $\widehat{Q}$ that satisfies $\widehat{Q}(s,a)=Q^\ast(s,a)$ for  $s\in\bar{\mathcal{S}}
$ and $a \in \mathcal{A}$. Furthermore, we can uniquely identify $r$ up to $s\in\bar{\mathcal{S}}$ and $a \in \mathcal{A}$ by $r(s, a)= \widehat{Q}(s, a) -  \beta \cdot \mathbb{E}_{s^{\prime} \sim P(s, a)} 
    \bigl[ V_{\widehat{Q}} \bigr]$.
\end{thm}

\begin{proof}[Proof of Theorem \ref{thm:mainopt}]
Define $\widehat{Q}$ as the solution to the expected risk minimization problem
\begin{align}
    \widehat{Q} &\in \underset{Q\in \mathcal{Q}}{\arg\min } \; \;\mathbb{E}_{(s, a)\sim \pi^*, \nu_0}  \left[-\log\left(\hat{p}_{Q}(a
\mid s)\right)\right] + \mathbb{E}_{(s, a)\sim \pi^*, \nu_0}\left[ \mathbf{1}_{a = a_s} \mathcal{L}_{BE}(Q)(s,a)\right] \tag{Equation \ref{eq:mainopt}}.
\end{align}
Note that Equation \eqref{eq:mainopt} minimizes the sum of two terms that are jointly minimized in Equation \eqref{eq:modifiedHotz}. Since the solution set to Equation \eqref{eq:modifiedHotz} is nonempty by Lemma \ref{thm:MagnacThesmar}, by Lemma \ref{lem:sharingsol}, $\widehat{Q}$ minimizes Equation \eqref{eq:modifiedHotz}. Again, by Lemma \ref{thm:MagnacThesmar}, this implies that $\widehat{Q}$ is equivalent to $Q^\ast$.
\end{proof}
\noindent Essentially, Theorem \ref{thm:mainopt} ensures that solving Equation \eqref{eq:mainopt} gives the exact $r$ and $Q^\ast$ up to 
$\bar{\mathcal{S}}$ and thus provides the solution to the DDC/IRL problem defined in Definition \ref{def:IRLproblem}. 

\subsubsection*{Remark. Comparison with Imitation Learning} 
Having established the identification guarantees for the ERM-IRL/DDC framework, it is natural to compare this formulation to the identification properties of Imitation Learning (IL) \citep{torabi2018behavioral, rajaraman2020toward, foster2024behavior}. Unlike IRL, which seeks to infer the underlying reward function that explains expert behavior, IL directly aims to recover the expert policy without modeling the transition dynamics. The objective of imitation learning is often defined as finding a policy $\hat{p}$ with 
$\min _{\hat{p}} \mathbb{E}_{(s, a) \sim \pi^\ast, \nu_0}\left[\ell\left(\hat{p}(a \mid s), \pi^\ast(a \mid s)\right)\right]$, $\ell$ is the cross-entropy loss or equivalently, 
\begin{align}
\min _{\hat{p}} \mathbb{E}_{(s, a) \sim \pi^\ast , \nu_0}\left[-\log \hat{p}(a|s)\right]\label{eq:mleEqBC}    
\end{align}
Equation \eqref{eq:mleEqBC} is exactly what a typical Behavioral Cloning (BC) \citep{torabi2018behavioral} minimizes under entropy regularization, as the objective of BC is
\begin{align}
  & \!\!\underset{Q\in \mathcal{Q}}{\min }  \;\mathbb{E}_{(s, a)\sim \pi^*, \nu_0}  \left[-\log \hat{p}_Q(a|s)\right] \label{eq:BC}
\end{align}

\noindent where $\hat{p}_Q(a\mid s) = \frac{Q(s,a)}{\sum_{\tilde{a}\in\mathcal{A} }Q(s,\tilde{a})}$. Note that the solution set of Equation \eqref{eq:BC} fully contains the solution set of Equation \eqref{eq:modifiedHotz}, which identifies the $Q^\ast$ for offline IRL/DDC. This means that any solution to the offline IRL/DDC problem also minimizes the imitation learning objective, but not necessarily vice versa. Consequently, under entropy regularization, the IL objective is fundamentally easier to solve than the offline IRL/DDC problem, as it only requires minimizing the negative log-likelihood term without enforcing Bellman consistency. One of the key contributions of this paper is to formally establish and clarify this distinction: IL operates within a strictly simpler optimization landscape than the offline IRL/DDC, making it a computationally and statistically more tractable problem. This distinction further underscores the advantage of Behavioral Cloning (BC) over ERM-IRL/DDC for imitation learning (IL) tasks-—since BC does not require modeling transition dynamics or solving an optimization problem involving the Bellman residual, it benefits from significantly lower computational and statistical complexity, making it a more efficient approach for IL.

While behavioral cloning (BC) is often sufficient for IL \citep{foster2024behavior}, i.e., reproducing the expert’s actions under the \emph{same} dynamics and incentive structure, it fundamentally lacks the ingredients needed for \emph{counterfactual} reasoning. Because BC learns a direct mapping $s \mapsto a$ without ever inferring the latent reward $r(s,a)$, it has no principled way to predict what an expert \emph{would} do if (i) the transition kernel $P$ were perturbed (e.g., a new recommendation algorithm) or (ii) the reward landscape itself were altered (e.g., a firm introduces monetary incentives). In such scenarios the state–action pairs generated by the expert no longer follow the original occupancy measure, so a cloned policy is forced to extrapolate outside its training distribution—precisely where imitation learning is known to fail.

Recovering the reward resolves this limitation. Once we have a consistent estimate $\hat r$ (or, equivalently, $\hat Q^{\ast}$), we can \emph{decouple} policy evaluation from the historical data: any hypothetical change to $P$ or the reward can be encoded and handed to a standard offline RL, which recomputes the optimal policy for the \emph{new} MDP without further demonstrations. In other words, rewards serve as a portable, mechanism‑level summary of preferences that supports robust counterfactual simulation, policy optimization, and welfare analysis—capabilities that pure imitation methods cannot provide.

\subsection{Estimation via minimax-formulated empirical risk minimization}
\label{sec:DoubleSampling}
While the idea of expected risk minimization -- minimizing Equation \eqref{eq:mainopt} -- is straightforward, empirically approximating $\mathcal{L}_{B E}(Q)(s, a) = (\mathcal{T} Q(s, a)-Q(s, a))^2$ and its gradient is quite challenging. 
As discussed in Section \ref{sec:BE&TD}, $\mathcal{T}Q$ is not available unless we know the transition function. As a result, we have to rely on an estimate of $\mathcal{T}$. A natural choice, common in TD-methods, is $\hat{\mathcal{T}} Q\left(s, a, s^{\prime}\right)=r(s, a)+\beta \cdot V_Q(s^\prime)$ which is computable given $Q$ and data $\mathcal{D}$. Thus, a natural proxy objective to minimize is:
\begin{align}
    &\mathbb{E}_{s' \sim P(s, a)} [\mathcal{L}_{\mathrm{TD}}(Q)\left(s, a, s^{\prime}\right)] :=\mathbb{E}_{s' \sim P(s, a)} [(\hat{\mathcal{T}} Q\left(s, a, s^{\prime}\right)-Q(s, a))^2] \notag
\end{align}
Temporal Difference (TD) methods typically use stochastic approximation to obtain an estimate of this proxy objective \citep{tesauro1995temporal, adusumilli2019temporal}. However, the issue with TD methods is that minimizing the proxy objective will not minimize the Bellman error in general (see Web Appendix $\S$\ref{sec:BiconjProofs} for details), because of the extra variance term, as shown below. 
\begin{align}
&\mathbb{E}_{s' \sim P(s, a)} 
\bigl[\mathcal{L}_{TD}(Q)(s, a, s^\prime)\bigr] = \mathcal{L}_{BE}(Q)(s, a) + \mathbb{E}_{s' \sim P(s, a)} 
\bigl[(\mathcal{T}Q(s, a) - \hat{\mathcal{T}}Q(s, a, s^\prime))^2\bigr]    \notag
\end{align}
As defined, $\hat{\mathcal{T}}$ is a one-step estimator, and the second term in the above equation does not vanish even in infinite data regimes. So, simply using the TD approach to approximate squared Bellman error provides a biased estimate. Intuitively, this problem happens because expectation and square are not exchangeable, i.e., 
$\mathbb{E}_{ s^\prime \sim P(s, a)}\left[\delta_{Q}\left(s,a, s^\prime\right)\mid s, a\right]^2 \neq \mathbb{E}_{ s^\prime \sim P(s, a)}\left[\delta_{Q}\left(s,a, s^\prime\right)^2\mid s, a\right]$. This bias in the TD approach is a well-known problem called the \textit{double sampling problem} \citep{ munos2003error, lagoudakis2003least,sutton2018reinforcement, jiangoffline}.  To remove this problematic square term, we employ an approach often referred to as the ``Bi-Conjugate Trick'' \citep{antos2008learning, dai2018sbeed, patterson2022generalized} which replaces a square function with a linear function called the bi-conjugate:
\begin{align}
     \mathcal{L}_{BE}(s,a)(Q)&:=\mathbb{E}_{ s^\prime \sim P(s, a)}\left[\delta_{Q}\left(s,a, s^\prime\right)\mid s, a\right]^2\notag
     \\
     &=\max_{h\in \mathbb{R}}2\cdot\mathbb{E}_{ s^\prime \sim P(s, a)}\left[\delta_{Q}\left(s,a, s^\prime\right)\mid s, a\right]\cdot h-h^2 \notag
\end{align}
\noindent By further re-parametrizing $h$ using $\zeta = h - r(s,a) + Q(s,a)$, after some algebra, we arrive at Lemma \ref{lem:OurBiconj}. See Web Appendix $\S$\ref{sec:BiconjProofs} for the detailed derivation.

\begin{lem}
\label{lem:OurBiconj}
\;\\
(a) We can express the squared Bellman error as
\begin{align}
    \mathcal{L}_{BE}(Q)(s, a)&:=(\mathcal{T} Q(s, a)-Q(s, a))^2 \notag
    \\
    &=\mathbb{E}_{s^\prime \sim P(s, a)} 
    \bigl[\mathcal{L}_{TD}(s, a, s^\prime)(Q)\bigr] - \beta^2 D(Q)(s, a) \label{eq:OurBiconj}
\end{align}
\begin{align}
\textrm{where}  \;\;\;   D(Q)(s, a):=\min _{\zeta \in \mathbb{R}} \mathbb{E}_{s^{\prime} \sim P(s, a)}\left[\left(V_{Q}\left(s^{\prime}\right)-\zeta\right)^2 \mid s, a\right]\label{eq:D(Q)}
\end{align}
(b) Define the minimizer (over all states and actions) of objective \eqref{eq:D(Q)} as
$$ \zeta^*: (s,a)\mapsto \arg\min_{\zeta\in \mathbb{R}}\mathbb{E}_{ s^\prime \sim P(s, a)}\left[\left(V^\ast(s^\prime)- \zeta\right)^2\mid s, a\right]$$
then
$r(s,a) = Q^*(s,a)-\beta \zeta^*(s,a)$.
\end{lem}
\noindent The reformulation of $\lbe$ proposed in Lemma \ref{lem:OurBiconj} enjoys the advantage of minimizing the squared TD-error ($\mathcal{L}_{TD}$) but without bias. Combining Theorem \ref{thm:mainopt} and Lemma \ref{lem:OurBiconj}, we arrive at the following Theorem \ref{thm:algoEQ}, which gives the expected risk minimization formulation of IRL we propose.
\begin{thm}
\label{thm:algoEQ} $Q^*$ is uniquely identified by
\begin{align}
     &\underset{Q\in \mathcal{Q}}{\min }  \;\mathcal{R}_{exp}(Q) \notag
     \\
     &=\min _{Q \in \mathcal{Q}} \mathbb{E}_{(s, a) \sim \pi^*, \nu_0}\bigl[{\mathcal{L}_{N L L}(Q)(s, a)} + \mathbf{1}_{a=a_s} \bigl\{\mathbb{E}_{s^{\prime} \sim P(s, a)}\left[{\mathcal{L}_{T D}(Q)\left(s, a, s^{\prime}\right)}\right]-\beta^2 {D(Q)(s, a)}\bigr\} \notag
     \\
    &=\min _{Q \in \mathcal{Q}}\max_{\zeta\in \mathbb{R}^{S\times A}} \mathbb{E}_{(s, a) \sim \pi^*, \nu_0, s^{\prime} \sim P(s, a)}\bigl[\underbrace{\textcolor{blue}{-\log \left(\hat{p}_Q(a \mid s)\right)}}_{1)} + \mathbf{1}_{a=a_s}\bigl\{\underbrace{\textcolor{red}{\bigl(\hat{\mathcal{T}} Q\left(s, a, s^{\prime}\right)-Q(s, a)\bigr)^2}}_{{2)}} \notag
    \\
    & \quad -\beta^2 \underbrace{ \bigl(\textcolor{orange}{\left(V_{Q}\left(s^{\prime}\right)-\zeta(s,a)\right)^2} }_{{3)}}\bigr\} \bigr]\label{eq:AlgoOpt}
\end{align}

Furthermore, $r(s, a)=Q^*(s, a) - \beta \zeta^*(s, a)$ where $\zeta^*$ is defined in Lemma \ref{lem:OurBiconj}.
\end{thm}

\noindent Equation \eqref{eq:AlgoOpt} in Theorem \ref{thm:algoEQ} is a mini-max problem in terms of $Q\in\mathcal{Q}$ and the introduced dual function $\zeta\in \mathbb{R}^{S\times A}$. To summarize, term 1) is the negative log-likelihood equation, term 2) is the TD error, and term 3) introduces a dual function $\zeta$.
The introduction of the dual function $\zeta$ in term 3) may seem a bit strange. 
In particular, note that $\arg\max_{\zeta \in \mathbb{R}} -\mathbb{E}_{s^{\prime} \sim P(s, a)}\left[\left(V_Q\left(s^{\prime}\right)-\zeta\right)^2 \mid s, a\right]$ is just $\zeta = \mathbb{E}_{s'\sim P(s,a)}[V\left(s^{\prime}\right)|s,a]$. 
However, we do not have access to the transition kernel and the state and action spaces may be large. 
Instead, we think of $\zeta$ as a function of states and actions, $\zeta(s,a)$ as introduced in Lemma~\ref{lem:OurBiconj}. This parametrization allows us to optimize over a class of functions containing $\zeta(s,a)$ directly. 

Given that the minimax resolution for the expected risk minimization problem in Theorem \ref{thm:algoEQ} finds $Q^\ast$ under an \textit{infinite number of data}, we now discuss the case when we are only given a \textit{finite dataset} $\mathcal{D}$ instead. For this, let's first rewrite the Equation \eqref{eq:AlgoOpt} as
\begin{align}
   &=\mathbb{E}_{(s, a) \sim \pi^*, \nu_0, s^{\prime} \sim P(s, a)}\bigl[\textcolor{blue}{-\log \left(\hat{p}_Q(a \mid s)\right)}\bigr]  + \mathbb{E}_{(s, a) \sim \pi^*, \nu_0, s^{\prime} \sim P(s, a)}\bigl[\mathbf{1}_{a=a_s}\bigl\{\textcolor{red}{\bigl(\hat{\mathcal{T}} Q\left(s, a, s^{\prime}\right)-Q(s, a)\bigr)^2}\bigr\}\bigr]
  \notag 
   \\
& \quad\quad -  \mathbb{E}_{(s, a) \sim \pi^*, \nu_0, s^{\prime} \sim P(s, a)}\bigl[\mathbf{1}_{a=a_s} \bigl\{\beta^2\textcolor{orange}{\left(V_{Q}\left(s^{\prime}\right)-\widetilde{\zeta}(s,a)\right)^2} \bigr\}\bigr] \label{eq:modifiedExpRM}
\end{align}
where $\widetilde{\zeta} \in  \operatorname{argmin}_{\zeta\in \mathbb{R}^{S\times A}}\mathbb{E}_{(s, a) \sim \pi^*, \nu_0, s^{\prime} \sim P(s, a)}\bigl[\mathbf{1}_{a=a_s} \bigl\{\left(V_{Q}\left(s^{\prime}\right)-\zeta(s,a)\right)^2 \bigr\}\bigr]$. But such $\widetilde{\zeta}$ also satisfies $\widetilde{\zeta}\in \operatorname{argmin}_{\zeta\in \mathbb{R}^{S\times A}}\mathbb{E}_{(s, a) \sim \pi^*, \nu_0, s^{\prime} \sim P(s, a)}\bigl[ \left(V_{Q}\left(s^{\prime}\right)-\zeta(s,a)\right)^2 \bigr] $, because
\begin{align}
\widetilde{\zeta}&\in\operatorname{argmin}_{\zeta\in \mathbb{R}^{S\times A}}\mathbb{E}_{(s, a) \sim \pi^*, \nu_0, s^{\prime} \sim P(s, a)}\bigl[\mathbf{1}_{a=a_s} \bigl\{\left(V_{Q}\left(s^{\prime}\right)-\zeta(s,a)\right)^2 \bigr\}\bigr] \notag
\\
&\in \operatorname{argmin}_{\zeta\in \mathbb{R}^{S\times A}}\mathbb{E}_{(s, a) \sim \pi^*, \nu_0, s^{\prime} \sim P(s, a)}\bigl[ \left(V_{Q}\left(s^{\prime}\right)-\zeta(s,a)\right)^2 \bigr] \label{eq:expzeta}
\end{align}
Substituting Equation \eqref{eq:modifiedExpRM} and Equation \eqref{eq:expzeta}'s  expectation over the data distribution with the empirical mean over the data, we arrive at the \textit{empirical risk minimization} formulation defined in Definition \ref{def:ERM}.

\begin{defn}[Empirical risk minimization]
\label{def:ERM}
Given $N := |\mathcal{D}|$ where $\mathcal{D}$ is a finite dataset. An empirical risk minimization problem is defined as the problem of finding $Q$ that minimizes the empirical risk $\mathcal{R}_{emp}(Q;\mathcal{D})$, which is defined as
\begin{align}
    &\mathcal{R}_{e m p}(Q ; \mathcal{D}):= \frac{1}{N} \sum_{\left(s, a, s^{\prime}\right) \in \mathcal{D}}\bigl[\textcolor{blue}{-\log \left(\hat{p}_Q(a \mid s)\right)}\bigr] \notag  
    \\
    & \quad\quad\quad+ \frac{1}{N} \sum_{\left(s, a, s^{\prime}\right) \in \mathcal{D}}\bigl[\mathbf{1}_{a=a_s}\bigl\{\textcolor{red}{\bigl(\hat{\mathcal{T}} Q\left(s, a, s^{\prime}\right)-Q(s, a)\bigr)^2} - \beta^2\textcolor{orange}{\left(V_{Q}\left(s^{\prime}\right)-\bar{\zeta}(s,a)\right)^2} \bigr\}\bigr]\label{eq:EmpiricalERMIRL}
\end{align}
where $\bar{\zeta}:=\operatorname{argmin}_{\zeta\in \mathbb{R}^{S\times A}}\frac{1}{N} \sum_{\left(s, a, s^{\prime}\right) \in \mathcal{D}}\bigl[ \left(V_{Q}\left(s^{\prime}\right)-\zeta(s,a)\right)^2 \bigr]$
\end{defn}
\noindent 
As formulated in Equation \eqref{eq:EmpiricalERMIRL}, our empirical risk minimization objective is a minimax optimization problem that involves inner maximization over $\zeta$ and outer minimization over $Q$. This structure is a direct consequence of introducing the dual function $\zeta$ to develop an unbiased estimator, thereby circumventing the double sampling problem discussed previously. In the next section, we introduce \textbf{GLADIUS}, a practical, alternating gradient ascent-descent algorithm specifically designed to solve this minimax problem and learn the underlying reward function from the data.
% ===== End input: ERM.tex =====

% ===== Begin input: algorithm.tex =====
\section{GLADIUS: Algorithm for ERM-IRL (ERM-DDC)}
\label{sec:Algorithm} 

\begin{algorithm}
\caption{\textbf{G}radient-based \textbf{L}earning with \textbf{A}scent-\textbf{D}escent for \textbf{I}nverse \textbf{U}tility learning from \textbf{S}amples (GLADIUS)}
\label{alg:estimation1}
\begin{algorithmic}[1]
\Require Offline dataset $\mathcal{D}=\{(s, a, s')\}$, time horizon $T$
\Ensure $\widehat{r}$, $\widehat{Q}$
\State Initialize $Q_{\boldsymbol{\theta_2}}$, $\zeta_{\boldsymbol{\theta_1}}$, $\text{iteration} \gets 1$
\While{$t \leq T$}
    \State Draw batches $B_1, B_2$ from $\mathcal{D}$
    
    \State \textbf{[Ascent Step: Update $ \zeta_{\boldsymbol{\theta_1}}$, fixing $Q_{\boldsymbol{\theta_2}}$ and $V_{\boldsymbol{\theta_2}}$]}
    \State $D_{\boldsymbol{\theta_1}} \gets \sum\limits_{(s, a, s^\prime)\in B_2} 
    \textcolor{orange}{\bigl(V_{\boldsymbol{\theta_2}}(s^\prime) 
    - \zeta_{\boldsymbol{\theta_1}}(s, a)\bigr)^2}$
    \State \textbf{where} $V_{\boldsymbol{\theta}}(s^\prime) := \log \sum_{\tilde{a} \in \mathcal{A}} \exp(Q_{\boldsymbol{\theta}}(s^\prime, \tilde{a}))$
    
    \State $\boldsymbol{\theta_1} \gets \boldsymbol{\theta_1} - \tau_1 \nabla_{\boldsymbol{\theta_1}} D_{\boldsymbol{\theta_1}}$
    
    \State \textbf{[Descent Step: Update $Q_{\boldsymbol{\theta_2}}$ and $V_{\boldsymbol{\theta_2}}$, fixing $\zeta_{\boldsymbol{\theta_1}}$]}
    
    \State $\overline{\mathcal{L}_{NLL}} \gets  \sum\limits_{(s, a, s^\prime)\in B_2}
    \textcolor{blue}{-\log\bigl(\hat{p}_{\boldsymbol{\theta_2}}(a \mid s)\bigr)}$
    
    \State $\overline{\mathcal{L}_{\mathrm{BE}}} \gets \sum\limits_{(s, a, s^\prime)\in B_1} 
    \mathbf{1}_{a = a^*_s} \bigl[\textcolor{red}{\mathcal{L}_{TD}(Q)\left(s, a, s^{\prime}\right)} - \beta^2 \textcolor{orange}{( V_{\boldsymbol{\theta_2}}(s^\prime)- \zeta_{\boldsymbol{\theta_1}}(s,a))^2}\bigr]$
    
    \State \textbf{where} $\mathcal{L}_{TD}(Q)\left(s, a, s^{\prime}\right) := \left(\hat{\mathcal{T}} Q\left(s, a, s^{\prime}\right)-Q(s, a)\right)^2$
    
    \State $\mathcal{L}_{\boldsymbol{\theta_2}} \gets \overline{\mathcal{L}_{\mathrm{NLL}}} 
    +  \overline{\mathcal{L}_{\mathrm{BE}}}$
    
    \State $\boldsymbol{\theta_2} \gets \boldsymbol{\theta_2} - \tau_1 \nabla_{\boldsymbol{\theta_2}} \mathcal{L}_{\boldsymbol{\theta_2}}$
    
    \State $\text{iteration} \gets \text{iteration} + 1$
\EndWhile
\State $\widehat{\zeta} \gets \zeta_{\boldsymbol{\theta_1}}$
\State $\widehat{Q} \gets Q_{\boldsymbol{\theta_2}}$
\State $\widehat{r}(s, a) \gets \widehat{Q}(s, a) - \beta \cdot \widehat{\zeta} (s, a)$
\end{algorithmic}
\end{algorithm}

\noindent  
Algorithm \ref{alg:estimation1} solves the empirical risk minimization problem in Definition \ref{def:ERM} through an alternating gradient ascent descent algorithm we call Gradient-based Learning with Ascent-Descent for Inverse Utility learning from Samples (GLADIUS). Given the function class $\mathcal{Q}$ of value functions, let $Q_{\boldsymbol{\theta}_2}\in\mathcal{Q}$ and $\zeta_{\boldsymbol{\theta}_1}\in\mathbb{R}^{S\times A}$ denote the functional representation of $Q$ and $\zeta$. Our goal is to learn the parameters $\boldsymbol{\theta}^*= \{\boldsymbol{\theta}_1^*, \boldsymbol{\theta}_2^* \}$, that together characterize $\hat{Q}$ and $\hat{\zeta}$. Each iteration in the GLADIUS algorithm consists of the following two steps: 
\begin{enumerate}[noitemsep]
    \item Gradient Ascent: For sampled batch data $B_1\subseteq \mathcal{D}$, take a gradient step for $\boldsymbol{\theta}_1$ while fixing $Q_{\boldsymbol{\theta}_2}$.
    \item Gradient Descent: For sampled batch data $B_2\subseteq \mathcal{D}$, take a gradient step for $\boldsymbol{\theta}_2$ while fixing $\zeta_{\boldsymbol{\theta}_1}$.
\end{enumerate}
Note that $\mathcal{D}$ can be used instead of using batches $B_1$ and $B_2$; the usage of batches is to keep the computational/memory complexity $O(|B_1=B_2|)$ instead of $O(|\mathcal{D}|)$.

After a fixed number of gradient steps of $Q_{\boldsymbol{\theta}_1}$ and $\zeta_{\boldsymbol{\theta}_2}$ (which we can denote as $\hat{Q}$ and $\hat{\zeta}$), we can compute the reward prediction $\hat{r}$ as 
$\hat{r}(s,a) = \hat{Q}(s,a) - \beta \hat{\zeta}(s,a)$ due to Theorem \ref{thm:algoEQ}.

\subsection*{Special Case: Deterministic Transitions}
\begin{algorithm}
\caption{GLADIUS under Deterministic Transitions}
\label{alg:estimation_deterministic}
\begin{algorithmic}[1]
\Require Offline dataset $\mathcal{D}=\{(s, a, s')\}$, time horizon $T$
\Ensure $\widehat{r}$, $\widehat{Q}$
\State Initialize $Q_{\boldsymbol{\theta}}$, $\text{iteration} \gets 1$
\While{$t \leq T$}
    \State Draw batch $B$ from $\mathcal{D}$
    
    \State $\overline{\mathcal{L}_{NLL}} \gets  \sum\limits_{(s, a, s^\prime)\in B}
    \textcolor{blue}{-\log\bigl(\hat{p}_{\boldsymbol{\theta}}(a \mid s)\bigr)}$
    
    \State $\overline{\mathcal{L}_{\mathrm{BE}}} \gets \sum\limits_{(s, a, s^\prime)\in B} 
    \mathbf{1}_{a= a^*_s} \textcolor{red}{\mathcal{L}_{TD}(Q)\left(s, a, s^{\prime}\right)}$
    
    \State \textbf{where} $\mathcal{L}_{TD}(Q)\left(s, a, s^{\prime}\right) := \left(\hat{\mathcal{T}} Q\left(s, a, s^{\prime}\right)-Q(s, a)\right)^2$
    
    \State $\mathcal{L}_{\boldsymbol{\theta}} \gets \overline{\mathcal{L}_{\mathrm{NLL}}} 
    +  \overline{\mathcal{L}_{\mathrm{BE}}}$
    
    \State $\boldsymbol{\theta} \gets \boldsymbol{\theta} - \tau \nabla_{\boldsymbol{\theta}} \mathcal{L}_{\boldsymbol{\theta}}$
    
    \State $\text{iteration} \gets \text{iteration} + 1$
\EndWhile
\State $\widehat{Q} \gets Q_{\boldsymbol{\theta}}$
\State $\widehat{r}(s, a) \gets \widehat{Q}(s, a) - \beta \log \sum_{\tilde{a} \in \mathcal{A}} \exp(\widehat{Q}(s^\prime, \tilde{a}))$
\end{algorithmic}
\end{algorithm}
When the transition function is \textit{deterministic} (e.g., in \citet{rafailov2024r, guo2025deepseek, zhong2024dpo}) meaning that for any state-action pair \((s, a)\), the next state \(s'\) is uniquely determined, the ascent step involving \(\zeta\) is no longer required. This is because the term \(\left(V_{Q}\left(s^{\prime}\right)-\zeta(s,a)\right)^2\) (highlighted in orange in Equation \eqref{eq:AlgoOpt} and \eqref{eq:EmpiricalERMIRL}) becomes redundant in the empirical ERM-IRL objective, because $\max_{\zeta \in \mathbb{R}} -\mathbb{E}_{s^{\prime} \sim P(s, a)}\left[\left(V_Q\left(s^{\prime}\right)-\zeta\right)^2 \mid s, a\right]$ is always $0$. Consequently, the optimization simplifies to:

\begin{align}
    \min _{Q \in \mathcal{Q}}&\frac{1}{N} \sum_{(s,a,s^\prime)\in \mathcal{D}}\bigl[\textcolor{blue}{-\log \left(\hat{p}_Q(a \mid s)\right)} + 
   \mathbf{1}_{a=a_s}\textcolor{red}{\bigl(\hat{\mathcal{T}} Q\left(s, a, s^{\prime}\right)-Q(s, a)\bigr)^2} \bigr]\label{eq:DeterministicERMIRL}
\end{align}

\noindent Under deterministic transitions, the GLADIUS algorithm reduces to gradient descents for \( Q_{\boldsymbol{\theta}} \), eliminating the need for the alternating ascent-descent update steps. Consequently, the estimated reward function is computed as:
\begin{equation}
    \hat{r}(s,a) = \hat{Q}(s,a) - \beta V_Q(s^\prime) \notag
\end{equation}
\noindent Key Differences in the Deterministic Case:
\begin{itemize}[leftmargin=*]
    \item No Ascent Step: The ascent step for \(\zeta\) is removed since the term \(\left(V_{Q}(s') - \zeta(s,a)\right)^2\) disappears.
    \item Gradient Descent: The algorithm updates \(Q_{\boldsymbol{\theta}}\) via a single gradient descent step per iteration.
    \item Reward Computation: The reward function is computed as \(\hat{r}(s, a) = \hat{Q}(s, a) - \beta V_Q(s^\prime)\).
\end{itemize}

\noindent This modification makes GLADIUS computationally more efficient when applied to deterministic environments while maintaining the correct theoretical formulation of the $Q^\ast$ and reward functions.

% ===== End input: algorithm.tex =====

% ===== Begin input: theory.tex =====
\section{Theoretical guarantees for GLADIUS}
\label{sec:Analysis}

This section establishes the optimization guarantees used by GLADIUS. The empirical objective in Equation~\eqref{eq:EmpiricalERMIRL} is a minimax problem: for fixed $Q$, the auxiliary block $\zeta$ is a least-squares regression target and the corresponding maximization problem is strongly concave, while the $Q$-block is generally nonconvex because of the Bellman residual term \citep{bas2021logistic}. Our analysis therefore does not rely on convexity of the $Q$-objective. Instead, we show that the empirical $Q$-objective and the auxiliary regression objective satisfy Polyak-\L{}ojasiewicz (PL) inequalities under an empirical Jacobian conditioning condition.

The result is formulated for continuous state spaces and finite action spaces. The state space $\mathcal S$ is not discretized. The finite-dimensional matrices that appear below arise only because the empirical objective evaluates $Q_\theta$ and $\zeta_\phi$ at finitely many state-action points observed in the dataset. The proof also avoids the stronger full-column-rank condition
\[
\inf_{\|h\|_2=1}\sup_{(s,a)\in\mathcal S\times\mathcal A}|h^\top\nabla_\theta Q_\theta(s,a)|>0.
\]
Instead, the analysis uses empirical-output Jacobian conditioning. The verification of this condition for linear and sufficiently wide neural-network parametrizations is deferred to Appendix~\ref{app:empirical_param_proofs}.

\subsection{Empirical PL conditions}

Let $\mathcal S$ be a finite-dimensional state space and let $\mathcal A$ be a finite action space. Consider the parametrized class
\[
\mathcal{Q}
=
\left\{Q_{\boldsymbol{\theta}}:\mathcal S\times\mathcal A\to\mathbb R
\mid
\boldsymbol{\theta}\in\Theta\subseteq\mathbb R^{d_\theta},\ Q_{\boldsymbol{\theta}}\in\mathcal F\right\}.
\]
The class $\mathcal F$ may be linear, polynomial, or a neural-network class.

\begin{asmp}[Realizability]
\label{ass:realizability}
The class $\mathcal{Q}$ contains an optimal function $Q^*$, meaning that there exists $\boldsymbol{\theta}^*\in\Theta$ such that $Q_{\boldsymbol{\theta}^*}=Q^*$.
\end{asmp}

Assumption~\ref{ass:realizability}, standard in offline RL and value-function approximation \citep{chen2019information,xie2021batch,zhan2022offline,zanette2023realizability,jiangoffline}, is used below only through the empirical evaluations generated by the observed continuous-state sample.

Under this parametrization, the population objective is
\begin{align}
  \mathbb{E}_{(s, a)\sim \pi^*, \nu_0}  \left[\mathcal{L}_{NLL}(Q_{\boldsymbol{\theta}})(s,a) +  \mathbf{1}_{a = a_s} \mathcal{L}_{BE}(Q_{\boldsymbol{\theta}})(s,a)\right], \label{eq:paraERMIRL}
\end{align}
and the empirical objective is
\begin{align}
    &\max_{\boldsymbol{\theta}_1}\frac{1}{N} \sum_{(s,a,s^\prime)\in \mathcal{D}}
    \bigl[{-\log \left(\hat{p}_{Q_{\boldsymbol{\theta_2}}}(a \mid s)\right)} 
    + \mathbf{1}_{a=a_s}\bigl\{\bigl(\hat{\mathcal{T}} Q_{\boldsymbol{\theta_2}}\left(s, a, s^{\prime}\right)-Q_{\boldsymbol{\theta_2}}(s, a)\bigr)^2 
    -\beta^2  \bigl(V_{Q_{\boldsymbol{\theta_2}}}\left(s^{\prime}\right)-\zeta_{\boldsymbol{\theta_1}}(s,a)\bigr)^2 \bigr\} \bigr].\label{eq:paraEmp}
\end{align}
Here $V_Q(s)=\log\sum_{b\in\mathcal A}\exp(Q(s,b))$ and $\hat p_Q(b\mid s)=\exp(Q(s,b))/\sum_{c\in\mathcal A}\exp(Q(s,c))$.

\begin{defn}[Polyak-\L{}ojasiewicz condition]
Given $\Theta\subseteq\mathbb R^d$, a differentiable function $g:\Theta\to\mathbb R$ satisfies the Polyak-\L{}ojasiewicz (PL) condition with respect to the Euclidean norm if the solution set is nonempty, $g$ has finite minimal value $g^*$, and there exists $c>0$ such that
\[
    \frac12\|\nabla g(\boldsymbol\theta)\|_2^2
    \ge
    c\bigl(g(\boldsymbol\theta)-g^*\bigr),
    \qquad
    \forall \boldsymbol\theta\in\Theta.
\]
\end{defn}

Let $\mathcal D$ denote the observed finite dataset. The empirical objective \eqref{eq:paraEmp} is computed from $\mathcal D$ and consists of the NLL term, the Bellman-error term, and the auxiliary $\zeta$-regression term. For the Bellman-error term, the action is the anchor action $a_s$. No sample splitting or independence between objective-specific index sets is required.

Let $\mathcal X_Q=\mathcal X_Q(\mathcal D)$ denote the finite set of sampled states at which the empirical objective evaluates $Q_\theta$. Since $\mathcal A$ is finite, define the empirical $Q$-evaluation set
\[
    Z_Q=Z_Q(\mathcal D):=\mathcal X_Q\times\mathcal A
    =\{z_1,\ldots,z_M\},
    \qquad M:=|Z_Q|.
\]
Let
\[
    Z_\zeta=Z_\zeta(\mathcal D)=\{\bar z_1,\ldots,\bar z_G\},
    \qquad G:=|Z_\zeta|,
\]
be the distinct state-action pairs at which the auxiliary network $\zeta_\phi$ is evaluated. The exact grouping of the empirical averages and weights is deferred to Appendix~\ref{app:empirical_grouping}. Define
\[
    \mathbf Q_\theta(Z_Q):=
    \bigl(Q_\theta(z_1),\ldots,Q_\theta(z_M)\bigr)^\top,
    \qquad
    J_Q(\theta;Z_Q):=D_\theta\mathbf Q_\theta(Z_Q),
\]
and
\[
    \boldsymbol\zeta_\phi(Z_\zeta):=
    \bigl(\zeta_\phi(\bar z_1),\ldots,\zeta_\phi(\bar z_G)\bigr)^\top,
    \qquad
    J_\zeta(\phi;Z_\zeta):=D_\phi\boldsymbol\zeta_\phi(Z_\zeta).
\]

\begin{asmp}[Empirical Jacobian conditioning]
\label{ass:nonSingularJac}
With the empirical evaluation sets and Jacobians defined above, let $B_Q=B(\theta_0,R_Q)$ and $B_\zeta=B(\phi_0,R_\zeta)$ be closed finite-radius balls containing the initial point, optimization iterates, and the empirical minimizers.
\begin{enumerate}[label=(\roman*)]
\item \textbf{Smoothness and bounded derivatives.} For every $z\in Z_Q$, $\theta\mapsto Q_\theta(z)$ is twice continuously differentiable on $B_Q$, and there exist $B_{Q,1},B_{Q,2}<\infty$ such that
\[
    \sup_{\theta\in B_Q}\sup_{z\in Z_Q}\|\nabla_\theta Q_\theta(z)\|_2\le B_{Q,1},
    \qquad
    \sup_{\theta\in B_Q}\sup_{z\in Z_Q}\|\nabla_\theta^2 Q_\theta(z)\|_{\rm op}\le B_{Q,2}.
\]
For every $\bar z\in Z_\zeta$, $\phi\mapsto\zeta_\phi(\bar z)$ is twice continuously differentiable on $B_\zeta$, and analogous finite constants $B_{\zeta,1},B_{\zeta,2}$ exist.

\item \textbf{Empirical-output $Q$-Jacobian conditioning.} There exists $\mu_Q>0$ such that, for every $\theta\in B_Q$,
\[
    J_Q(\theta;Z_Q)J_Q(\theta;Z_Q)^\top\succeq \mu_Q I_M.
\]
Equivalently, for every $v\in\mathbb R^M$,
\[
    \|J_Q(\theta;Z_Q)^\top v\|_2^2\ge \mu_Q\|v\|_2^2.
\]

\item \textbf{Empirical-output $\zeta$-Jacobian conditioning.} There exists $\mu_\zeta>0$ such that, for every $\phi\in B_\zeta$,
\[
    J_\zeta(\phi;Z_\zeta)J_\zeta(\phi;Z_\zeta)^\top\succeq \mu_\zeta I_G.
\]
\end{enumerate}
\end{asmp}

\begin{lem}[Linear parametrization satisfies Assumption~\ref{ass:nonSingularJac}]
\label{lem:linPolyNonsingular}
Let $Q_\theta(s,a)=\theta^\top\varphi(s,a)$ and let $\zeta_\phi(s,a)=\phi^\top\chi(s,a)$. If the empirical feature Gram matrices on $Z_Q(\mathcal D)$ and $Z_\zeta(\mathcal D)$ are positive definite, then Assumption~\ref{ass:nonSingularJac} holds for the linear parametrizations.
\end{lem}

\begin{lem}[Wide smooth neural networks satisfy Assumption~\ref{ass:nonSingularJac}]
\label{lem:NNenjoysPL}
For sufficiently wide feedforward neural networks with $C^2$ hidden activations, standard over-parameterized scaling, and random initialization, Assumption~\ref{ass:nonSingularJac} holds with high probability under the standard sufficient conditions stated and justified in Appendix~\ref{app:empirical_param_proofs}.
\end{lem}

%Technical output-space lemmas used to prove the next PL statements are stated and proved in Appendix~\ref{app:technical_output_lemmas}.

\begin{lem}[Empirical Bellman-error PL]
\label{lem:BE_PL}
Suppose Assumption~\ref{ass:nonSingularJac}(ii) holds. Then there exists a constant $c_{\rm BE}>0$ such that
\[
    {1\over2}\|\nabla_\theta\widehat L_{\rm BE}(\theta)\|_2^2
    \ge
    c_{\rm BE}\widehat L_{\rm BE}(\theta),
    \qquad
    \forall\theta\in B_Q.
\]
\end{lem}

\begin{lem}[Empirical NLL PL]
\label{lem:NLL_PL}
Suppose Assumption~\ref{ass:nonSingularJac}(i)--(ii) hold. Then there exists a constant $c_{\rm NLL}>0$ such that
\[
    {1\over2}\|\nabla_\theta\widehat L_{\rm NLL}(\theta)\|_2^2
    \ge
    c_{\rm NLL}
    \left(\widehat L_{\rm NLL}(\theta)-\widehat L_{\rm NLL}^*\right),
    \qquad
    \forall\theta\in B_Q.
\]
\end{lem}

\begin{lem}[Empirical auxiliary-regression PL]
\label{lem:zeta_empirical_pl}
Fix $\theta_2\in B_Q$. Suppose Assumption~\ref{ass:nonSingularJac}(iii) holds. Then there exists a constant $c_\zeta>0$ such that
\[
    {1\over2}\|\nabla_\phi\widehat F_\zeta(\phi;\theta_2)\|_2^2
    \ge
    c_\zeta
    \left(\widehat F_\zeta(\phi;\theta_2)-\widehat F_\zeta^*(\theta_2)\right),
    \qquad
    \forall\phi\in B_\zeta.
\]
Equivalently, the inner maximization objective $-\widehat F_\zeta$ satisfies the corresponding PL inequality for gradient ascent.
\end{lem}

\begin{thm}[Empirical PL guarantees]
\label{thm:bothPL}
Under Assumption~\ref{ass:nonSingularJac}, the empirical $Q$-objective
\[
    \widehat R_Q(\theta):=\widehat L_{\rm NLL}(\theta)+\widehat L_{\rm BE}(\theta)
\]
satisfies a PL inequality on $B_Q$: there exists $c_{\widehat R}>0$ such that
\[
    {1\over2}\|\nabla_\theta\widehat R_Q(\theta)\|_2^2
    \ge
    c_{\widehat R}
    \left(\widehat R_Q(\theta)-\widehat R_Q^*\right),
    \qquad
    \forall\theta\in B_Q.
\]
Moreover, for each fixed $\theta_2\in B_Q$, the auxiliary loss $\widehat F_\zeta(\cdot;\theta_2)$ satisfies a PL inequality on $B_\zeta$. Explicit admissible constants are given in the appendix proofs.
\end{thm}

\subsection{Global convergence of GLADIUS}
\label{sec:GlobalConv}

We now combine the empirical PL geometry from Theorem~\ref{thm:bothPL}
with the stability/generalization bound for stochastic gradient ascent-descent.
The goal is to separate the two sources of error in GLADIUS: the
\emph{optimization error} from running only finitely many gradient iterations, and
the \emph{statistical/generalization error} from replacing the population risk by
the empirical risk.

Let $\mathcal D_N=\{(s_i,a_i,s_i')\}_{i=1}^N$ denote the observed dataset.  Write
the empirical minimax objective optimized by GLADIUS as
\begin{align}
    &\widehat f_N(\theta,\phi)
    := \notag
    \\
    &{1\over N}\sum_{(s,a,s')\in\mathcal D_N}
    \biggl[
        -\log\hat p_{Q_\theta}(a\mid s)
        +\mathbf 1_{a=a_s}
        \biggl\{
            \bigl(\widehat{\mathcal T}Q_\theta(s,a,s')-Q_\theta(s,a)\bigr)^2
            -\beta^2
            \bigl(V_{Q_\theta}(s')-\zeta_\phi(s,a)\bigr)^2
        \biggr\}
    \biggr].
    \label{eq:empirical_minimax_gladius}
\end{align}
For fixed $\theta$, the maximization over $\phi$ estimates the conditional
mean correction in the biconjugate representation of the squared Bellman error.
Define the profiled empirical objective
\[
    \widehat g_N(\theta)
    :=
    \max_{\phi\in B_\zeta}\widehat f_N(\theta,\phi),
    \qquad
    \widehat g_N^*
    :=
    \min_{\theta\in B_Q}\widehat g_N(\theta),
\]
and let
\[
    \widehat S_N
    :=
    \argmin_{\theta\in B_Q}\widehat g_N(\theta)
\]
denote the set of empirical minimizers.  Up to constants that do not affect the
minimization over $\theta$, $\widehat g_N$ is the empirical $Q$-objective
$\widehat R_Q$ studied in Theorem~\ref{thm:bothPL}.

Let $(\widehat\theta_T,\widehat\phi_T)$ be the parameters returned by
Algorithm~\ref{alg:estimation1}.
The following theorem gives the empirical optimization rate and then lifts it to
a population excess-risk bound.

\begin{thm}[Global convergence of GLADIUS]
\label{prop:linConvergence}
Suppose Assumptions~\ref{ass:realizability} and~\ref{ass:nonSingularJac} hold.
Assume further that the stochastic gradients used by Algorithm~\ref{alg:estimation1} use stepsizes
$\eta_t=c_1/(c_2+t)$ with constants satisfying the two-sided-PL SGDA conditions
of \citet{yang2020global}. Then there exist constants $\nu,\gamma_0>0$ such that
\[
    \mathbb E\!\left[
        \widehat g_N(\widehat\theta_T)-\widehat g_N^*
    \right]
    \le
    {\nu\over \gamma_0+T}.
    \label{eq:empirical_optimization_rate}
\]
Moreover, there exists
$c_{\rm qg}>0$ such that
\[
    \mathbb E\!\left[
        \operatorname{dist}(\widehat\theta_T,\widehat S_N)^2
    \right]
    \le
    {\nu\over c_{\rm qg}(\gamma_0+T)}.
    \label{eq:empirical_parameter_distance}
\]
Finally, the population ERM-DDC/IRL excess risk satisfies
\[
    \mathbb E\!\left[
        \mathcal R_{\exp}(Q_{\widehat\theta_T})
        -
        \mathcal R_{\exp}(Q^*)
    \right]
    \le
    (1+L/\rho)G\,\varepsilon_{N,T}
    +
    {\nu\over \gamma_0+T},
    \label{eq:population_excess_risk_rate}
\]
where
\[
    \varepsilon_{N,T}
    =
    O\!\left((c_2+T)^{-\alpha}\right)
    +
    {C\over N},
    \qquad
    \alpha:=\min\left\{{1\over2},{3cc_1\over8}\right\}.
\]
Here $L$ is the smoothness constant of the empirical objective, $\rho$ is the
strong-concavity constant of the auxiliary regression problem, $G$ bounds the
stochastic gradient second moment, and $C,c,c_1,c_2$ are problem-dependent
constants.
\end{thm}
Theorem~\ref{prop:linConvergence} is a risk bound.  To translate it into
estimation error for the object of interest, we use the quantitative
identification property of the full ERM-DDC/IRL risk.  Let $d^*$ denote the
population state-action distribution generated by the expert policy $\pi^*$ and
the initial distribution $\nu_0$.  For any measurable $f:\mathcal S\times
\mathcal A\to\mathbb R$, define
\[
    \|f\|_{2,*}^2
    :=
    \mathbb E_{(s,a)\sim d^*}[f(s,a)^2].
\]

\begin{thm}[$Q^*$-error rate]
\label{thm:distributional_q_bound}
Suppose the conditions of Theorem~\ref{prop:linConvergence} hold.  Suppose also
that the population ERM-DDC/IRL risk satisfies the quantitative identification
bound in Lemma~\ref{lem:erm_ddc_quant_identification}; that is, there exists
$C_{\rm id}<\infty$ such that, for all $\theta\in B_Q$,
\[
    \|Q_\theta-Q^*\|_{2,*}^2
    \le
    C_{\rm id}
    \left[
        \mathcal R_{\exp}(Q_\theta)
        -
        \mathcal R_{\exp}(Q^*)
    \right].
\]
Then GLADIUS satisfies
\[
    \mathbb E\!\left[
        \|Q_{\widehat\theta_T}-Q^*\|_{2,*}^2
    \right]
    \le
    C_{\rm id}
    \left[
        (1+L/\rho)G\,\varepsilon_{N,T}
        +
        {\nu\over \gamma_0+T}
    \right].
    \label{eq:q_error_bound}
\]
Consequently,
\[
    \mathbb E\!\left[
        \|Q_{\widehat\theta_T}-Q^*\|_{2,*}^2
    \right]
    =
    O\!\left((c_2+T)^{-\alpha}\right)
    +
    O\!\left({1\over N}\right)
    +
    O\!\left({1\over T}\right).
\]
Equivalently, the $L_2(d^*)$ estimation error obeys
\[
    \left(
        \mathbb E\!\left[
            \|Q_{\widehat\theta_T}-Q^*\|_{2,*}^2
        \right]
    \right)^{1/2}
    =
    O\!\left((c_2+T)^{-\alpha/2}\right)
    +
    O\!\left({1\over \sqrt N}\right)
    +
    O\!\left({1\over \sqrt T}\right).
\]
In particular, when $\alpha=1/2$, the optimization contribution to the
$L_2(d^*)$ error is $O(T^{-1/4})$ and the statistical contribution is
$O(N^{-1/2})$.
\end{thm}

\noindent\textbf{Remark.}
To the best of our knowledge, no prior work has proposed an algorithm that
guarantees global-optimum convergence for a minimization problem involving the
Bellman-error term $\mathcal L_{\rm BE}(Q)(s,a)$.%
\footnote{Some studies, such as \citet{dai2018sbeed}, establish convergence to
a stationary point of the corresponding minimax problem.}
In this regard, Lemma~\ref{lem:BE_PL} has an important implication for offline
reinforcement learning.  Gradient-based offline RL methods
\citep{antos2008learning,dai2018sbeed} minimize a Bellman-error objective in
essentially the same way that GLADIUS does.  These methods have been shown to
converge, but global convergence guarantees have not yet been established in
the generality considered here.  Lemma~\ref{lem:BE_PL}, combined with a
suitable distributional assumption, may be useful for establishing that
gradient-based offline RL is globally convergent for important function classes,
including tabular, linear, polynomial, and neural-network classes.

\section{Offline IRL/DDC experiments}
\label{sec:Experiments}

We now present results from simulation experiments, comparing the performance of our approach with that of several benchmark algorithms. In this section, we use the high-dimensional version of the canonical bus engine replacement problem (\cite{rust1994structural}) as the setting for our experiments. Later, in $\S$\ref{sec:imitation}, we use the OpenAI gym benchmark environment experiments with a discrete action space (Lunar Lander, Acrobot, and Cartpole) \cite{brockman2016openai} as in \citet{garg2021iq} for the related but easier problem of imitation learning. 

\subsection{Experimental Setup}
This bus engine setting has been extensively used as the standard benchmark for the reward learning problem in the DDC literature in economics \citep{hotz1993conditional, aguirregabiria2002swapping,  kasahara2009nonparametric,  arcidiacono2011conditional,  arcidiacono2011practical, su2012constrained, norets2009inference, chiong2016duality, reich2018divide, chernozhukov2022locally, geng2023data, barzegary2022recursive, yang2024estimation}. 

The bus engine replacement problem \citep{rust1987optimal} is a simple regenerative optimal stopping problem. In this setting, the manager of a bus company operates many identical buses. As a bus accumulates mileage, its per-period maintenance cost increases. The manager can replace the engine in any period (which then becomes as good as new, and this replacement decision resets the mileage to one). However, the replacement decision comes with a high fixed cost. Each period, the manager makes a dynamic trade-off between replacing the engine and continuing with maintenance. We observe the manager's decisions for a fixed set of buses, i.e., a series of states, decisions, and state transitions. Our goal is to learn the manager's reward function from these observed trajectories under the assumption that he made these decisions optimally. 

\noindent \textbf{Dataset.} 
There are $N$ independent and identical buses (trajectories) indexed by $j$, each of which has $100$ periods over which we observe them, i.e., $h\in\{1\ldots 100\}$. Each bus's trajectory starts with an initial mileage of 1. The only reward-relevant state variable at period $h$ is the mileage of bus $x_{jh} \in \{1, 2, \ldots 20\}$. 

\noindent \textbf{Decisions and rewards.}\;\; There are two possible decisions at each period, replacement or continuation, denoted by $d_{jh} = \{0,1\}$. $d_{jh}=1$ denotes replacement, and there is a fixed cost $\theta_1$ of replacement. Replacement resets the mileage to 1, i.e., the engine is as good as new. $d_{jh}=0$ denotes maintenance, and the cost of maintaining the engine depends on the mileage as follows: $\theta_0 x_{jh}$. Intuitively, the manager can pay a high fixed cost $\theta_1$ for replacing an engine in this period but doing so reduces future maintenance costs since the mileage is reset to 1. In all our experiments, we set $\theta_0 = 1$ (maintenance cost) and $\theta_1 = 5$ (replacement cost). Additionally, we set the discount factor to $\beta = 0.95$.

\noindent \textbf{State transitions at each period.}\;\; If the manager chooses maintenance, the mileage advances by 1, 2, 3, or 4 with a $1/4$ probability each. If the manager chooses to replace the engine, then the mileage is reset to 1. That is, $\mathbb{P}(\{x_{j(h+1)}=x_{jh}+k\}\mid d_{jh}=0 ) = 1/4$, $k\in \{1,2,3,4\}$ and $\mathbb{P}\{x_{j(h+1)}=1\mid d_{jh}=1\}=1$. When the bus reaches the maximum mileage of 20, we assume that mileage remains at 20 even if the manager continues to choose maintenance.

%The cost for choosing action $d_{jh}$ depends only on how much mileage bus $j$ accumulated from its last replacement of the engine $x_{jt}\in \mathcal{S}=\{1, 2, \ldots 20\}$: when the manager decides to keep the engine with mileage $x_{jt}$, maintenance cost occurs with the amount $\theta_1 x_{jt}$; footnote{In DDC notation, this is equivalent to maintenance cost of $\theta_1 x_{j t}+\epsilon_{j t 0}$ and replacement cost of $\theta_2+\epsilon_{jt1}$, where $\epsilon_{jt0}, \epsilon_{jt1}$ are both independently and identically drawn from $G(0, 1)$ for $h\in [H]$ and $j\in [N]$, where $G$ is the Type 1 Extreme Value (T1EV) distribution. }. 

%This reward information is not revealed to the researcher. 

\noindent \textbf{High-dimensional setup.}\;\; In some simulations, we consider a high-dimensional version of the problem, where we now modify the basic set-up described above to include a set of $K$ high-dimensional state variables, similar to \citet{geng2023data}. Assume that we have access to an additional set of $K$ state variables $\{s^1_{jh}, s^2_{jh}, s^3_{jh} \ldots s^K_{jh}\}$, where each $s^k_{jh}$ is an i.i.d random draw from $\{-10, -9, \ldots, 9, 10\}$. We vary $K$ from 2 to 100 in our empirical experiments to test the sensitivity of our approach to the dimensionality of the problem.   
Further, we assume that these high-dimensional state variables $s_{jh}^k$s do not affect the reward function or the mileage transition probabilities. However, the researcher does not know this. So, they are included in the state space, and ideally, our algorithm should be able to infer that these state variables do not affect rewards and/or value function and recover the true reward function.

\noindent \textbf{Traing/testing split.}  Throughout, we keep 80\% of the trajectories in any experiment for training/learning the reward function, and the remaining 20\% is used for evaluation/testing. 

\noindent \textbf{Functional form.} For the oracle methods, we use the reward functions' true parametric form, i.e., a linear function. For other methods (including ours), we used a multi-layer perceptron (MLP) with two hidden layers and 10 perceptrons for each hidden layer for the estimation of $Q$-function. 

\subsection{Benchmark Algorithms}
We compare our algorithm against a series of standard, or state-of-art benchmark algorithms in the DDC and IRL settings.
\\
\textbf{Rust (Oracle)}
\; Rust is an oracle-like fixed point iteration baseline that uses the nested fixed point algorithm \citep{rust1987optimal}. It assumes the knowledge of: (1) linear parametrization of rewards by $\theta_1$ and $\theta_2$ as described above, and (2) the exact transition probabilities. 
\\
\textbf{ML-IRL (Oracle)}\; ML-IRL from \citet{zeng2023understanding} is the state-of-the-art offline IRL algorithm that minimizes negative log-likelihood of choice (i.e., the first term in Equation \eqref{eq:mainopt}). This method requires a separate estimation of transition probabilities, which is challenging in high-dimensional settings. So, we make the same oracle assumptions as we did for Rust (Oracle), i.e., assume that transition probabilities are known. Additionally, to further improve this method, we leverage the finite dependence property of the problem \citep{arcidiacono2011conditional}, which helps avoid roll-outs.  
\\
\textbf{SAmQ}\; SAmQ \cite{geng2023data} fits approximated soft-max Value Iteration (VI) to the observed data. We use the SAmQ implementation provided by the authors\footnote{\href{https://github.com/gengsinong/SAmQ}{https://github.com/gengsinong/SAmQ}}. However, their code did not scale due to a memory overflow issue, and did not work for scenarios with 2500 (i.e., 250,000 samples) trajectories or more. 
\\
\textbf{IQ-learn}\; IQ-learn is a popular gradient-based method, maximizing the occupancy matching objective (which does not guarantee that the Bellman equation is satisfied; see Web Appendix $S$\ref{sec:occupancy}) for details.
\\
\textbf{BC}\; Behavioral Cloning (BC) simply minimizes the expected negative log-likelihood. This simple algorithm outperforms \cite{zeng2023understanding, li2022rethinkingvaluedice} many recent algorithms such as ValueDICE \cite{kostrikov2019imitation}. See $\S$\ref{sec:ERM-IRL} for detailed discussions.

\subsection{Experiment results}
\subsubsection{Performance results for the standard bus engine setting}
Table \ref{fig:mse_r_estimation} provides a table of simulation experiment results without dummy variables, i.e., with only mileage ($x_{jh}$) as the relevant state variable. The performance of algorithms was compared in terms of \textit{mean absolute percentage error (MAPE)} of $r$ estimation, which is defined as $\frac{1}{N} \sum_{i=1}^N \left|\frac{\hat{r}_i-r_i}{r_i}\right| \times 100$, where $N$ is the total number of samples from expert policy $\pi^*$ and $\hat{r}_i$ is each algorithm's estimator for the true reward $r_i$.\footnote{In the simulation we consider, we don't have a state-action pair with true reward near 0.}\footnote{As we assume that the data was collected from agents following (entropy regularized) optimal policy $\pi^*$ (Assumption \ref{ass:IRLoptimaldecision}), the distribution of states and actions in the data is the best data distribution choice.} 

\begin{table*}[h]
    \centering
    \scalebox{0.8}{
    \begin{tabular}{l
            >{\centering\arraybackslash}p{2.5cm}
            >{\centering\arraybackslash}p{2.5cm}
            >{\centering\arraybackslash}p{2.5cm}
            >{\centering\arraybackslash}p{2.5cm}
            >{\centering\arraybackslash}p{2.5cm}
            >{\centering\arraybackslash}p{2.5cm}}
\toprule
\multirow{2}{*}{\parbox{1.7cm}{No. of Trajectories\\(H=100)}} 
& \multicolumn{2}{c}{Oracle Baselines} & \multicolumn{4}{c}{\makecell{Neural Network, No Knowledge of Transition Probabilities
}} \\
\cmidrule(lr){2-3} \cmidrule(lr){4-7}
& Rust & ML-IRL & \textbf{\ul{GLADIUS}} & SAmQ & IQ-learn & BC \\
\cmidrule(lr){2-7}
& MAPE (SE) & MAPE (SE) & MAPE (SE) & MAPE (SE) & MAPE (SE) & MAPE (SE) \\
\midrule
50   & 3.62 (1.70) & 3.62 (1.74) & \textbf{3.44} (1.28) & 4.92 (1.20)  & 114.13 (26.60) & 80.55 (12.82)  \\
250   & 1.37 (0.77) & 1.10 (0.78) & \textbf{0.84} (0.51) & 3.65 (1.00)  & 112.86 (27.31) & 72.04 (13.21) \\
500   & 0.90 (0.56) & 0.84 (0.59) & \textbf{0.55} (0.20) & 3.13 (0.86)  & 113.27 (25.54) & 71.92 (12.44) \\
1000  & 0.71 (0.49) & 0.64 (0.48) & \textbf{0.52} (0.22) & 1.55 (0.46)  & 112.98 (24.12) & 72.17 (12.11) \\
2500  & 0.68 (0.22) & 0.62 (0.35) & \textbf{0.13} (0.06) & N/A          & 111.77 (23.99) & 62.61 (10.75) \\
5000  & 0.40 (0.06) & 0.43 (0.26) & \textbf{0.12} (0.06) & N/A          & 119.18 (22.55) & 46.45 (8.22) \\
\bottomrule
\multicolumn{7}{l}{\footnotesize 
Based on 20 repetitions. Oracle baselines (Rust, MLIRL) were based on bootstrap repetition of 100.}
\end{tabular}
    }
\caption{Mean Absolute Percentage Error (MAPE) (\%) of $r$ Estimation.  (\# dummy = 0)}
\label{fig:mse_r_estimation}   
\end{table*}

We find that GLADIUS performs much better than non-Oracle baselines and performs at least on par with, or slightly better than, Oracle baselines. A natural question here is: why do the Oracle baselines that leverage the exact transition function and the precise linear parametrization not beat our approach? The main reason for this is the imbalance of state-action distribution from expert policy: (See Table \ref{tab:r_a0_1000} and Web Appendix $\S$\ref{sec:AppendixBus}) 
\begin{enumerate}[leftmargin=*]
    \item All trajectories start from mileage 1. In addition, the replacement action (action 0) resets the mileage to 1. Therefore, most states observed in the expert data are within mileage 1-5. When data collection policy visits a few states much more frequently than others, ``the use of a projection operator onto a space of function approximators with respect to a distribution induced by the behavior policy can result in poor performance if that distribution does not sufficiently cover the state space.'' \citep{tsitsiklis1996analysis} This makes Oracle baseline predictions for states with mileage 1-5 slightly worse than GLADIUS.
    \item Since we evaluate MAPE on the police played in the data, this implies that our evaluation mostly samples mileages 1--5, and GLADIUS's weakness in extrapolation for mileage 6-10 matters less than the slight imprecision of oracle parametric methods in mileages 1--5.
\end{enumerate}
\begin{table*}[h!]
    \centering
    \scalebox{0.75}{
    \begin{tabular}{c|cccccccccc}
        \toprule
      Mileages & 1 & 2 & 3 & 4 & 5 & 6 & 7 & 8 & 9 & 10 \\
        \midrule
        Frequency  & 7994 & 1409 & 1060 & 543 & 274 & 35 & 8 & 1 & 0 & 0 \\
        \midrule
         True reward  & -1.000 & -2.000 & -3.000 & -4.000 & -5.000 & -6.000 & -7.000 & -8.000 & -9.000 & -10.000 \\
        ML-IRL  & -1.013 & -2.026 & -3.039 & -4.052 & -5.065 & -6.078 & -7.091 & -8.104 & -9.117 & -10.130 \\
        Rust & -1.012 & -2.023 & -3.035 & -4.047 & -5.058 & -6.070 & -7.082 & -8.093 & -9.105 & -10.117 \\
       \textbf{\ul{GLADIUS}}  & -1.000 & -1.935 & -2.966 & -3.998 & -4.966 & -5.904 & -6.769 & -7.633 & -8.497 & -9.361 \\
        \bottomrule
    \end{tabular}
    }
    \caption{Estimated rewards and frequency values for 1,000 trajectories for action 0.}
    \label{tab:r_a0_1000}
\end{table*}

\noindent Finally, it is not surprising to see IQ-learn and BC underperform in the reward function estimation task since they do not require/ensure that the Bellman condition holds. See Appendix \ref{sec:imitation} for a detailed discussion. 

\subsubsection{Performance results for the high-dimensional set-up. }

Figure \ref{fig:dummy_r_estimation} (below) presents high-dimensional experiments, where states were appended with dummy variables. Each dummy variable is of dimension 20. Note that a state space of dimensionality $20^{10}$ (10 dummy variables with 20 possible values each) is equivalent to $10^{13}$, which is infeasible for existing exact methods (e.g., Rust) and methods that require transition probability estimation (e.g., ML-IRL). Therefore, we only present comparisons to the non-oracle methods.

\begin{figure}[h]
    \centering
    \begin{minipage}{0.48\textwidth} % Adjust the width as needed
        \centering
        %----- Table
        \scalebox{0.8}{
        \begin{tabular}{lcccc}
            \toprule
            \multirow{2}{*}{\parbox{1.3cm}{Dummy\\ Variables}} & 
            \multicolumn{1}{c}{\textbf{\ul{GLADIUS}}} & 
            \multicolumn{1}{c}{SAmQ} & 
            \multicolumn{1}{c}{IQ-learn} & 
            \multicolumn{1}{c}{BC} \\
            \cmidrule(lr){2-5}
            & MAPE (SE) & MAPE (SE) & MAPE (SE) & MAPE (SE) \\
            \midrule
             2    & \textbf{1.24} (0.45) & 1.79 (0.37) & 112.0 (14.8) & 150.9 (29.1) \\
             5    & \textbf{2.51} (1.19) & 2.77 (0.58)  & 192.2 (19.2) & 171.1 (37.3) \\
             20   & \textbf{6.07} (3.25) & N/A  & 180.1 (15.6) & 180.0 (33.7) \\
             50   & \textbf{9.76} (3.68) & N/A  & 282.2 (25.2) & 205.1 (35.3) \\
             100  & \textbf{11.35} (4.24) & N/A  & 321.1 (23.1) & 288.8 (42.9) \\
            \bottomrule
            \multicolumn{5}{l}{\footnotesize Based on 10 repetitions. For SAmQ, N/A means that the algorithm did not scale.} \\
        \end{tabular}
        }
    \end{minipage}
    \hfill
    \begin{minipage}{0.48\textwidth} % Adjust the width as needed
        \centering
        %----- Plot
        \begin{tikzpicture}
        \begin{semilogxaxis}[
            width=6.5cm, % Adjust width for double column
            height=5cm,
            xlabel={\# Dummy Variables (log scale)},
            ylabel={MAPE of $r$ Estimation},
            title={GLADIUS Performance},
            grid=both,
            xtick={2, 5, 20, 50, 100},
            xticklabels={2, 5, 20, 50, 100},
            ytick={0, 5, 10, 15},
            yticklabels={0, 5, 10, 15},
            mark options={solid},
            legend pos=north east,
            error bars/y explicit,
            error bars/error bar style={line width=1pt, dashed}
        ]

        % AGAD (Ours) with updated values
        \addplot+[mark=o, blue, thick, error bars/.cd, y dir=both, y explicit] coordinates {
            (2, 1.24) +- (0, 0.45)
            (5, 2.51) +- (0, 1.19)
            (20, 6.07) +- (0, 3.25)
            (50, 9.76) +- (0, 3.68)
            (100, 11.35) +- (0, 4.24)
        };
        \addlegendentry{GLADIUS (Ours)}

        \end{semilogxaxis}
        \end{tikzpicture}
    \end{minipage}
    \caption{MAPE of $r$ estimation. The left panel shows the MAPE values in a tabular format, and the right panel visualizes the GLADIUS's performance on a log-scaled x-axis. 1000 trajectories were used for all experiments. Smaller is better; the best value in each row is highlighted.}
    \label{fig:dummy_r_estimation}
\end{figure}
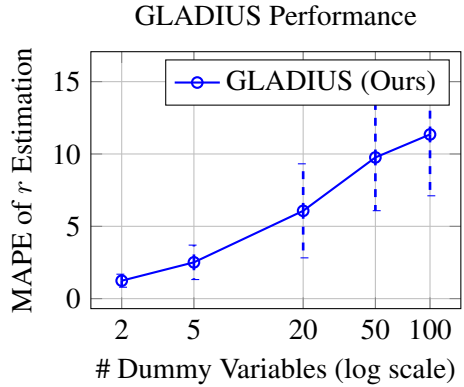
We find that our approach outperforms benchmark algorithms, including SAmQ, IQ‑learn, and BC (see Figure~\ref{fig:dummy_r_estimation}). Moreover, as illustrated in the right panel of Figure~\ref{fig:dummy_r_estimation}, the MAPE grows sub‑linearly with the size of the state space (note the logarithmic \(x\)-axis). Remarkably, even in the extreme setting with \(K=100\) dummy variables—an astronomically large discrete state space of \(20^{100}\!\approx\!10^{130}\) distinct configurations—GLADIUS’s MAPE rises only to about~10\%. This limited loss in precision due to such a massive dimensional expansion underscores the method’s robustness and practical scalability to very high-dimensional applications.

% ===== End input: experiments.tex =====

% ===== Begin input: ILexperiments.tex =====
\section{Imitation Learning experiments}
\label{sec:imitation}

One of the key contributions of this paper is the characterization of the relationship between imitation learning (IL) and inverse reinforcement learning (IRL)/Dynamic Discrete Choice (DDC) model, particularly through the ERM-IRL/DDC framework. Given that much of the IRL literature has historically focused on providing experimental results for IL tasks, we conduct a series of experiments to empirically validate our theoretical findings. Specifically, we aim to test our
prediction in Section \ref{sec:ERM-IRL} that \textit{behavioral cloning (BC) should outperform ERM-IRL for IL tasks}, as BC directly optimizes the negative log-likelihood objective without the additional complexity of Bellman error minimization. By comparing BC and ERM-IRL across various IL benchmark tasks, we demonstrate that BC consistently achieves better performance in terms of both computational efficiency and policy accuracy, reinforcing our claim that IL is a strictly easier problem than IRL.

\subsection{Experimental Setup}
As in \citet{garg2021iq}, we employ three OpenAI Gym environments for algorithms with discrete actions \citep{brockman2016openai}: Lunar Lander v2, Cartpole v1, and Acrobot v1. These environments are widely used in IL and RL research, providing well-defined optimal policies and performance metrics. 

\noindent \textbf{Dataset.}  
For each environment, we generate expert demonstrations using a pre-trained policy. We use publicly available expert policies\footnote{\url{https://huggingface.co/sb3/}} trained via Proximal Policy Optimization (PPO) \cite{schulman2017proximal}, as implemented in the Stable-Baselines3 library \citep{raffin2021stable}. Each expert policy is run to generate demonstration trajectories, and we vary the number of expert trajectories across experiments for training. For all experiments, we used the expert policy demonstration data from 10 episodes for testing.

\noindent \textbf{Performance Metric.}  
The primary evaluation metric is \% optimality, defined as:
\begin{align}
    \text{\% optimality of an episode} := \frac{\text{One episode's episodic reward of the algorithm}}{\text{Mean of 1,000 episodic rewards of the expert}} \times 100. \notag
\end{align}
For each baseline, we report the mean and standard deviation of 100 evaluation episodes after training. A higher \% optimality indicates that the algorithm's policy closely matches the expert. The 1000-episodic mean and standard deviation ([mean$\pm$std]) of the episodic reward of expert policy for each environment was $[232.77\pm73.77]$ for Lunar-Lander v2 (larger the better), $[-82.80\pm27.55]$ for Acrobot v1 (smaller the better), and $[500\pm 0]$ for Cartpole v1 (larger the better).

\noindent \textbf{Training Details.}  
All algorithms were trained for 5,000 epochs. Since our goal in this experiment is to show the superiority of BC for IL tasks, we only include ERM-IRL and IQ-learn \cite{garg2021iq} as baselines. Specifically, we exclude baselines such as Rust \citep{rust1987optimal} and ML-IRL \citep{zeng2023understanding}, which require explicit estimation of transition probabilities.

\subsection{Experiment results}
%\noindent \textbf{GLADIUS (Ours)}  
%The ERM-IRL framework minimizes both the negative log-likelihood (NLL) and Bellman error (BE) terms, making it computationally more complex than BC. 

%\noindent \textbf{IQ-Learn} \citep{garg2021iq}  
%A popular \cite{rafailov2024r} IRL method that minimizes an occupancy-matching objective, i.e., it does not enforce Bellman consistency. For details, refer to Section \ref{sec:ImitationID}.

%\noindent \textbf{Behavioral Cloning (BC)}  
%BC minimizes only the NLL term, making it computationally simple and sample-efficient.

Table \ref{fig:OpenAI_gym} presents the \% optimality results for Lunar Lander v2, Cartpole v1, and Acrobot v1. As predicted in our theoretical analysis, BC consistently outperforms ERM-IRL in terms of \% optimality, validating our theoretical claims.

\begin{table*}[ht]
    \centering
    \scalebox{0.85}{
    \begin{tabular}{l
            >{\centering\arraybackslash}p{1.5cm}
            >{\centering\arraybackslash}p{1.5cm}
            >{\centering\arraybackslash}p{1.5cm}
            >{\centering\arraybackslash}p{1.5cm}
            >{\centering\arraybackslash}p{1.5cm}
            >{\centering\arraybackslash}p{1.5cm}
            >{\centering\arraybackslash}p{1.5cm}
            >{\centering\arraybackslash}p{1.5cm}
            >{\centering\arraybackslash}p{1.5cm}}
\toprule
\multirow{4}{*}{\parbox{0.5cm}{Trajs}} 
& \multicolumn{3}{c}{\makecell{\ul{Lunar Lander v2 (\%)} \\ (Larger \% the better)}} 
& \multicolumn{3}{c}{\makecell{\ul{Cartpole v1 (\%)} \\ (Larger \% the better)}} 
& \multicolumn{3}{c}{\makecell{\ul{Acrobot v1 (\%)} \\ (Smaller \% the better)}} \\
\cmidrule(lr){2-4} \cmidrule(lr){5-7} \cmidrule(lr){8-10}
& {\ul{Gladius}} & IQ-learn & BC 
& {\ul{Gladius}} & IQ-learn & BC
& {\ul{Gladius}} & IQ-learn & BC \\
\midrule
1  & \textbf{107.30 }& 83.78 & 103.38   & 100.00  & 100.00  & 100.00     & 103.67  & 103.47  & \textbf{100.56}  \\
    & (10.44)  & (22.25)  & (13.78)   & (0.00)  & (0.00)  & (0.00)   & (32.78)  & (55.44)  & (26.71)  \\
\midrule
3   & \textbf{106.64}  &  102.44 & 104.46 & 100.00  & 100.00  & 100.00 & 102.19  & 101.28  & \textbf{101.25}  \\
    & (11.11)  & (20.66)   & (11.57)  & (0.00)  & (0.00)  & (0.00)  & (22.69)  & (37.51)  & (36.42)  \\
\midrule
7   &  101.10 & 104.91  & \textbf{ 105.99}  & 100.00  & 100.00  & 100.00 & 100.67 & 100.58  &\textbf{98.08} \\
    & (16.28)  & (13.98) &  (10.20) & (0.00)         & (0.00)  & (0.00)  & (22.30)      & (30.09)  &  (24.27)  \\
\midrule
10  & 104.46  & 105.13   & \textbf{ 107.01}   & 100.00  & 100.00  & 100.00  & 99.07  & 101.10 &  \textbf{97.75}\\
    & (13.65)  & (13.83)   & (10.75)  & (0.00)           & (0.00)  &  (0.00) &  (20.58)    & (30.40)  & (16.67)  \\
\midrule
15  & 106.11  &  106.51 & \textbf{107.42}  & 100.00  & 100.00  &100.00 & 96.50  & 95.34  & \textbf{95.33}  \\
    & (10.65)  & (14.10)  & (10.45)  & (0.00)         & (0.00)  &  (0.00)  & (18.53)        & (26.92)  & (15.42)  \\
\bottomrule
\multicolumn{10}{l}{\footnotesize 
Based on 100 episodes for each baseline. Each baseline was trained for 5000 epochs.}
\end{tabular}
    }
\caption{Mean and standard deviation of \% optimality of 100 episodes}
\label{fig:OpenAI_gym}   
\end{table*}

\section{Conclusion}
In this paper, we propose a provably globally convergent empirical risk minimization framework that combines non-parametric estimation methods (e.g., machine learning methods) with IRL/DDC models. This method's convergence to global optima stems from our new theoretical finding that the Bellman error (i.e., Bellman residual) satisfies the Polyak-Łojasiewicz (PL) condition, which is a weaker but almost equally useful condition as strong convexity for providing theoretical assurances. 

The three key advantages of our method are: (1) it is easily applicable to high-dimensional state spaces, (2) it can operate without the knowledge of (or requiring the estimation of) state-transition probabilities, and (3) it is applicable to infinite state spaces. These three properties make our algorithm practically applicable and useful in high-dimensional, infinite-size state and action spaces that are common in business and economics applications. We demonstrate our approach's empirical performance through extensive simulation experiments (covering both low and high-dimensional settings). We find that, on average, our method performs quite well in recovering rewards in both low and high-dimensional settings. Further, it has better/on-par performance compared to other benchmark algorithms in this area (including algorithms that assume the parametric form of the reward function and knowledge of state transition probabilities) and is able to recover rewards even in settings where other algorithms are not viable.

% ===== End input: conclusion.tex =====

%\input

\bibliographystyle{plainnat}
\bibliography{bib}

\newpage
\begin{appendices}
\setcounter{table}{0}
\setcounter{figure}{0}
\setcounter{equation}{0}
\setcounter{page}{0}
\renewcommand{\thetable}{A\arabic{table}}
\renewcommand{\thefigure}{A\arabic{figure}}
\renewcommand{\theequation}{A\arabic{equation}}
\renewcommand{\thepage}{\roman{page}}
\pagenumbering{roman}

% ===== Begin input: app_experiments.tex =====
\section{Extended experiment discussions}
\label{sec:ExtendedExp}
\subsection{More discussions on Bus engine replacement experiments}\label{sec:AppendixBus}

Figure \ref{fig:error50}, \ref{fig:error1000} and
Table \ref{tab:r_traj_50} - \ref{tab:q_traj_1000} shown below present the estimated results for reward and $Q^\ast$ using 50 trajectories (5,000 transitions) and 1,000 trajectories (100,000 transitions). As you can see in Figure \ref{fig:error50} and \ref{fig:error1000}, Rust and ML-IRL , which know the exact transition probabilities and employ correct parameterization (i.e., linear), demonstrate strong extrapolation capabilities for [Mileage, action] pairs that are rarely observed or entirely missing from the dataset (mileage 6-10). In contrast, GLADIUS, a neural network-based method, struggles with these underrepresented pairs.
\medskip
\noindent However, as we saw in the main text's Table \ref{fig:mse_r_estimation}, GLADIUS achieves par or lower Mean Absolute Percentage Error (MAPE), which is defined as 
$
\frac{1}{N} \sum_{i=1}^N\left|\frac{\hat{r}_i-r_i}{r_i}\right| \times 100$
where $N$ is the total number of samples from expert policy $\pi^*$ and $\hat{r}_i$ is each algorithm's estimator for the true reward $r_i$. This is because it overall outperforms predicting $r$ values for the [Mileage, action] pairs that appear most frequently and therefore contribute most significantly to the error calculation, as indicated by the visibility of the yellow shading in the tables below. (Higher visibility implies larger frequency.)
\begin{center}
   \textbf{ Results for 50 trajectories (absolute error plot, $r$ prediction, $Q^\ast$ prediction)}
\end{center}
\begin{figure}[h!]
    \centering
    \includegraphics[width=1\linewidth]{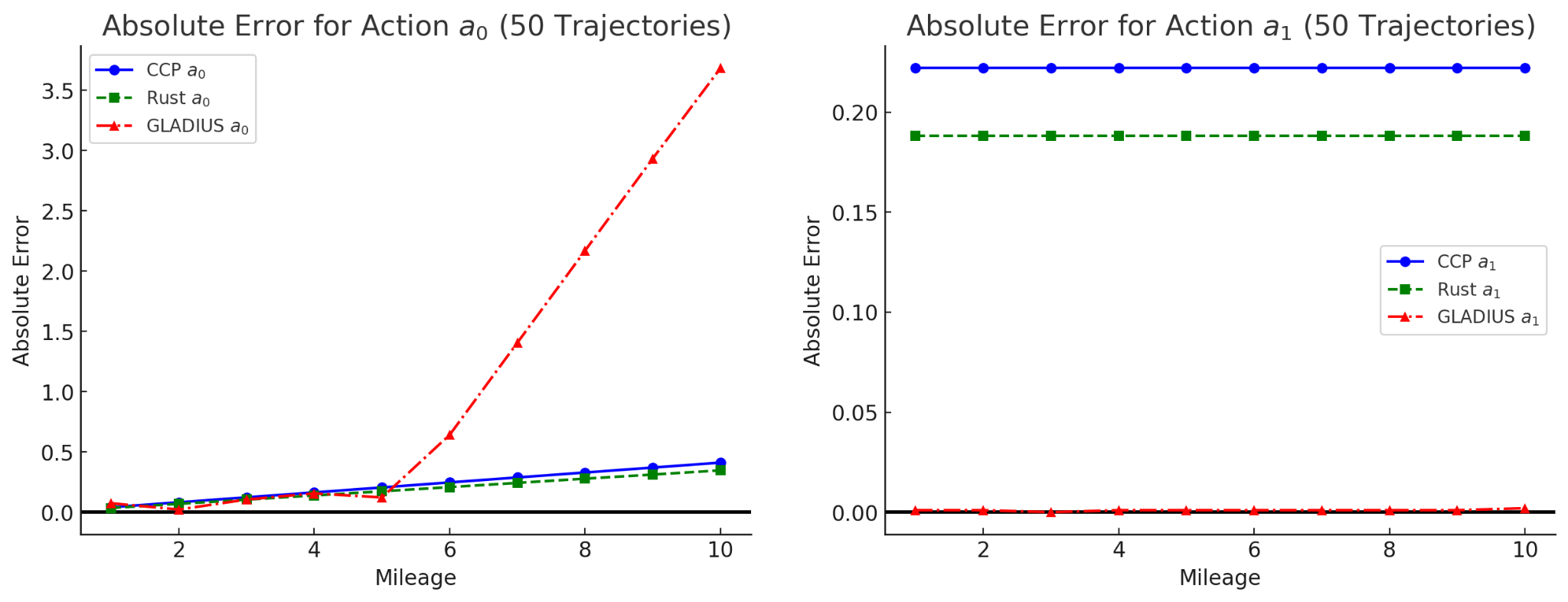}
    \caption{Reward estimation error comparison for 50 trajectories. Closer to 0 (black line) is better.}
    \label{fig:error50}
\end{figure}

\begin{table*}[h]
    \centering
    \scalebox{0.85}{
    \begin{tabular}{lcc|cccccccc}
        \toprule
        \multirow{2}{*}{Mileage} & \multicolumn{2}{c|}{Frequency} & \multicolumn{2}{c}{Ground Truth $r$} & \multicolumn{2}{c}{ML-IRL} & \multicolumn{2}{c}{Rust} & \multicolumn{2}{c}{GLADIUS} \\
        \cmidrule(lr){2-3} \cmidrule(lr){4-5} \cmidrule(lr){6-7} \cmidrule(lr){8-9} \cmidrule(lr){10-11}
        & $a_0$ & $a_1$ & $a_0$ & $a_1$ & $a_0$ & $a_1$ & $a_0$ & $a_1$ & $a_0$ & $a_1$ \\
        \midrule
                1  & \cellcolor{yellow!100} 412  & \cellcolor{yellow!5} 37  & \cellcolor{yellow!100} -1.000  & \cellcolor{yellow!5} -5.000  & \cellcolor{yellow!100} -0.959  & \cellcolor{yellow!5} -4.777  & \cellcolor{yellow!100} -0.965  & \cellcolor{yellow!5} -4.812  & \cellcolor{yellow!100} -1.074  & \cellcolor{yellow!5} -4.999 \\
        2  & \cellcolor{yellow!16} 65  & \cellcolor{yellow!2} 18  & \cellcolor{yellow!16} -2.000  & \cellcolor{yellow!2} -5.000  & \cellcolor{yellow!16} -1.918  & \cellcolor{yellow!2} -4.777  & \cellcolor{yellow!16} -1.931  & \cellcolor{yellow!2} -4.812  & \cellcolor{yellow!16} -1.978  & \cellcolor{yellow!2} -5.001 \\
        3  & \cellcolor{yellow!10} 43  & \cellcolor{yellow!19} 80  & \cellcolor{yellow!10} -3.000  & \cellcolor{yellow!19} -5.000  & \cellcolor{yellow!10} -2.877  & \cellcolor{yellow!19} -4.777  & \cellcolor{yellow!10} -2.896  & \cellcolor{yellow!19} -4.812  & \cellcolor{yellow!10} -3.105  & \cellcolor{yellow!19} -5.000 \\
        4  & \cellcolor{yellow!6} 24  & \cellcolor{yellow!25} 101  & \cellcolor{yellow!6} -4.000  & \cellcolor{yellow!25} -5.000  & \cellcolor{yellow!6} -3.836  & \cellcolor{yellow!25} -4.777  & \cellcolor{yellow!6} -3.861  & \cellcolor{yellow!25} -4.812  & \cellcolor{yellow!6} -3.844  & \cellcolor{yellow!25} -5.001 \\
        5  & \cellcolor{yellow!3} 8  & \cellcolor{yellow!33} 134  & \cellcolor{yellow!3} -5.000  & \cellcolor{yellow!33} -5.000  & \cellcolor{yellow!3} -4.795  & \cellcolor{yellow!33} -4.777  & \cellcolor{yellow!3} -4.827  & \cellcolor{yellow!33} -4.812  & \cellcolor{yellow!3} -4.878  & \cellcolor{yellow!33} -5.001 \\
        6  & \cellcolor{yellow!2} 4  & \cellcolor{yellow!5} 37  & \cellcolor{yellow!2} -6.000  & \cellcolor{yellow!5} -5.000  & \cellcolor{yellow!2} -5.753  & \cellcolor{yellow!5} -4.777  & \cellcolor{yellow!2} -5.792  & \cellcolor{yellow!5} -4.812  & \cellcolor{yellow!2} -6.642  & \cellcolor{yellow!5} -5.001 \\
        7  & \cellcolor{yellow!1} 1  & \cellcolor{yellow!4} 26  & \cellcolor{yellow!1} -7.000  & \cellcolor{yellow!4} -5.000  & \cellcolor{yellow!1} -6.712  & \cellcolor{yellow!4} -4.777  & \cellcolor{yellow!1} -6.757  & \cellcolor{yellow!4} -4.812  & \cellcolor{yellow!1} -8.406  & \cellcolor{yellow!4} -5.001 \\
        8  & \cellcolor{yellow!0} 0  & \cellcolor{yellow!1} 7  & \cellcolor{yellow!0} -8.000  & \cellcolor{yellow!1} -5.000  & \cellcolor{yellow!0} -7.671  & \cellcolor{yellow!1} -4.777  & \cellcolor{yellow!0} -7.722  & \cellcolor{yellow!1} -4.812  & \cellcolor{yellow!0} -10.170  & \cellcolor{yellow!1} -5.001 \\
        9  & \cellcolor{yellow!0} 0  & \cellcolor{yellow!0} 2  & \cellcolor{yellow!0} -9.000  & \cellcolor{yellow!0} -5.000  & \cellcolor{yellow!0} -8.630  & \cellcolor{yellow!0} -4.777  & \cellcolor{yellow!0} -8.688  & \cellcolor{yellow!0} -4.812  & \cellcolor{yellow!0} -11.934  & \cellcolor{yellow!0} -5.001 \\
        10 & \cellcolor{yellow!0} 0  & \cellcolor{yellow!0} 1  & \cellcolor{yellow!0} -10.000  & \cellcolor{yellow!0} -5.000  & \cellcolor{yellow!0} -9.589  & \cellcolor{yellow!0} -4.777  & \cellcolor{yellow!0} -9.653  & \cellcolor{yellow!0} -4.812  & \cellcolor{yellow!0} -13.684  & \cellcolor{yellow!0} -5.002 \\

        \bottomrule
    \end{tabular}
    }
    \caption{Reward estimation for 50 trajectories. Color indicates appearance frequencies.}
    \label{tab:r_traj_50}
\end{table*}

\begin{table*}[h!]
    \centering
    \scalebox{0.85}{
    \begin{tabular}{lcc|cccccccc}
        \toprule
        \multirow{2}{*}{Mileage} & \multicolumn{2}{c|}{Frequency} & \multicolumn{2}{c}{Ground Truth Q} & \multicolumn{2}{c}{ML-IRL Q} & \multicolumn{2}{c}{Rust Q} & \multicolumn{2}{c}{GLADIUS Q} \\
        \cmidrule(lr){2-3} \cmidrule(lr){4-5} \cmidrule(lr){6-7} \cmidrule(lr){8-9} \cmidrule(lr){10-11}
        & $a_0$ & $a_1$ & $a_0$ & $a_1$ & $a_0$ & $a_1$ & $a_0$ & $a_1$ & $a_0$ & $a_1$ \\
        \midrule
        1  & \cellcolor{yellow!100} 412  & \cellcolor{yellow!28} 37  & \cellcolor{yellow!100} -52.534  & \cellcolor{yellow!28} -54.815  & \cellcolor{yellow!100} -49.916  & \cellcolor{yellow!28} -52.096  & \cellcolor{yellow!100} -50.327  & \cellcolor{yellow!28} -52.523  & \cellcolor{yellow!100} -53.059  & \cellcolor{yellow!28} -55.311 \\
        2  & \cellcolor{yellow!16} 65  & \cellcolor{yellow!13} 18  & \cellcolor{yellow!16} -53.834  & \cellcolor{yellow!13} -54.815  & \cellcolor{yellow!16} -51.165  & \cellcolor{yellow!13} -52.096  & \cellcolor{yellow!16} -51.584  & \cellcolor{yellow!13} -52.523  & \cellcolor{yellow!16} -54.270  & \cellcolor{yellow!13} -55.312 \\
        3  & \cellcolor{yellow!10} 43  & \cellcolor{yellow!60} 80  & \cellcolor{yellow!10} -54.977  & \cellcolor{yellow!60} -54.815  & \cellcolor{yellow!10} -52.266  & \cellcolor{yellow!60} -52.096  & \cellcolor{yellow!10} -52.691  & \cellcolor{yellow!60} -52.523  & \cellcolor{yellow!10} -55.548  & \cellcolor{yellow!60} -55.312 \\
        4  & \cellcolor{yellow!6} 24  & \cellcolor{yellow!75} 101  & \cellcolor{yellow!6} -56.037  & \cellcolor{yellow!75} -54.815  & \cellcolor{yellow!6} -53.286  & \cellcolor{yellow!75} -52.096  & \cellcolor{yellow!6} -53.718  & \cellcolor{yellow!75} -52.523  & \cellcolor{yellow!6} -56.356  & \cellcolor{yellow!75} -55.312 \\
        5  & \cellcolor{yellow!2} 8  & \cellcolor{yellow!100} 134  & \cellcolor{yellow!2} -57.060  & \cellcolor{yellow!100} -54.815  & \cellcolor{yellow!2} -54.270  & \cellcolor{yellow!100} -52.096  & \cellcolor{yellow!2} -54.708  & \cellcolor{yellow!100} -52.523  & \cellcolor{yellow!2} -57.419  & \cellcolor{yellow!100} -55.312 \\
        6  & \cellcolor{yellow!1} 4  & \cellcolor{yellow!28} 37  & \cellcolor{yellow!1} -58.069  & \cellcolor{yellow!28} -54.815  & \cellcolor{yellow!1} -55.239  & \cellcolor{yellow!28} -52.096  & \cellcolor{yellow!1} -55.683  & \cellcolor{yellow!28} -52.523  & \cellcolor{yellow!1} -59.184  & \cellcolor{yellow!28} -55.312 \\
        7  & \cellcolor{yellow!0} 1  & \cellcolor{yellow!19} 26  & \cellcolor{yellow!0} -59.072  & \cellcolor{yellow!19} -54.815  & \cellcolor{yellow!0} -56.202  & \cellcolor{yellow!19} -52.096  & \cellcolor{yellow!0} -56.652  & \cellcolor{yellow!19} -52.523  & \cellcolor{yellow!0} -60.950  & \cellcolor{yellow!19} -55.312 \\
        8  & \cellcolor{yellow!0} 0  & \cellcolor{yellow!5} 7  & \cellcolor{yellow!0} -60.074  & \cellcolor{yellow!5} -54.815  & \cellcolor{yellow!0} -57.162  & \cellcolor{yellow!5} -52.096  & \cellcolor{yellow!0} -57.619  & \cellcolor{yellow!5} -52.523  & \cellcolor{yellow!0} -62.715  & \cellcolor{yellow!5} -55.312 \\
        9  & \cellcolor{yellow!0} 0  & \cellcolor{yellow!1} 2  & \cellcolor{yellow!0} -61.074  & \cellcolor{yellow!1} -54.815  & \cellcolor{yellow!0} -58.122  & \cellcolor{yellow!1} -52.096  & \cellcolor{yellow!0} -58.585  & \cellcolor{yellow!1} -52.523  & \cellcolor{yellow!0} -64.481  & \cellcolor{yellow!1} -55.312 \\
        10 & \cellcolor{yellow!0} 0  & \cellcolor{yellow!1} 1  & \cellcolor{yellow!0} -62.074  & \cellcolor{yellow!1} -54.815  & \cellcolor{yellow!0} -59.081  & \cellcolor{yellow!1} -52.096  & \cellcolor{yellow!0} -59.550  & \cellcolor{yellow!1} -52.523  & \cellcolor{yellow!0} -66.228  & \cellcolor{yellow!1} -55.308 \\
        \bottomrule
    \end{tabular}
    }
    \caption{$Q^\ast$ estimation for 50 trajectories. Color indicates appearance frequencies.}
    \label{tab:q_traj_50}
\end{table*}

\begin{center}
   \textbf{ Results for 1000 trajectories (absolute error plot, $r$ prediction, $Q^\ast$ prediction)}
\end{center}

\begin{figure}[h!]
    \centering
    \includegraphics[width=1\linewidth]{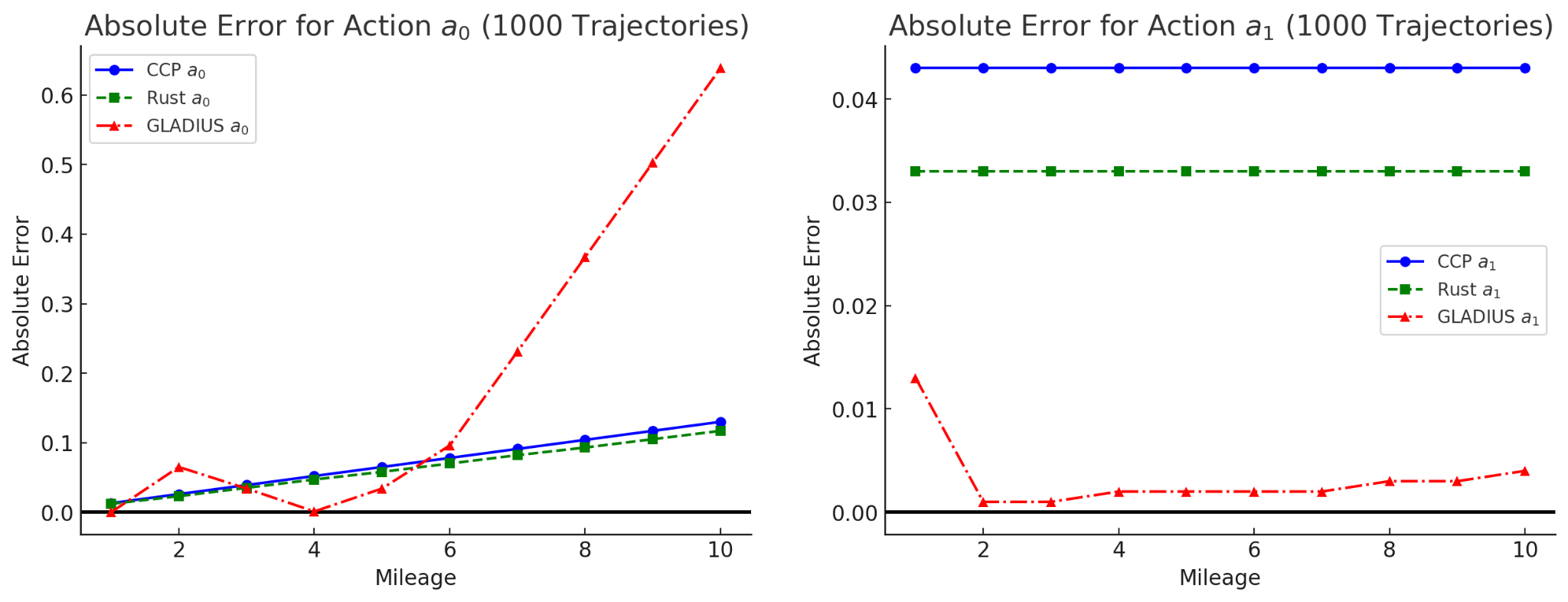}
    \caption{Reward estimation error comparison for 1,000 trajectories. Closer to 0 (black line) is better.}
    \label{fig:error1000}
\end{figure}

\begin{table*}[h!]
    \centering
    \scalebox{0.85}{
    \begin{tabular}{lcc|cccccccc}
        \toprule
        \multirow{2}{*}{Mileage} & \multicolumn{2}{c|}{Frequency} & \multicolumn{2}{c}{Ground Truth $r$} & \multicolumn{2}{c}{ML-IRL } & \multicolumn{2}{c}{Rust} & \multicolumn{2}{c}{GLADIUS} \\
        \cmidrule(lr){2-3} \cmidrule(lr){4-5} \cmidrule(lr){6-7} \cmidrule(lr){8-9} \cmidrule(lr){10-11}
        & $a_0$ & $a_1$ & $a_0$ & $a_1$ & $a_0$ & $a_1$ & $a_0$ & $a_1$ & $a_0$ & $a_1$ \\
        \midrule
           1  & \cellcolor{yellow!100} 7994  & \cellcolor{yellow!10} 804  & \cellcolor{yellow!100} -1.000  & \cellcolor{yellow!10} -5.000  & \cellcolor{yellow!100} -1.013  & \cellcolor{yellow!10} -5.043  & \cellcolor{yellow!100} -1.012  & \cellcolor{yellow!10} -5.033  & \cellcolor{yellow!100} -1.000  & \cellcolor{yellow!10} -5.013 \\
        2  & \cellcolor{yellow!18} 1409  & \cellcolor{yellow!7} 541  & \cellcolor{yellow!18} -2.000  & \cellcolor{yellow!7} -5.000  & \cellcolor{yellow!18} -2.026  & \cellcolor{yellow!7} -5.043  & \cellcolor{yellow!18} -2.023  & \cellcolor{yellow!7} -5.033  & \cellcolor{yellow!18} -1.935  & \cellcolor{yellow!7} -5.001 \\
        3  & \cellcolor{yellow!13} 1060  & \cellcolor{yellow!16} 1296  & \cellcolor{yellow!13} -3.000  & \cellcolor{yellow!16} -5.000  & \cellcolor{yellow!13} -3.039  & \cellcolor{yellow!16} -5.043  & \cellcolor{yellow!13} -3.035  & \cellcolor{yellow!16} -5.033  & \cellcolor{yellow!13} -2.966  & \cellcolor{yellow!16} -5.000 \\
        4  & \cellcolor{yellow!7} 543  & \cellcolor{yellow!25} 1991  & \cellcolor{yellow!7} -4.000  & \cellcolor{yellow!25} -5.000  & \cellcolor{yellow!7} -4.052  & \cellcolor{yellow!25} -5.043  & \cellcolor{yellow!7} -4.047  & \cellcolor{yellow!25} -5.033  & \cellcolor{yellow!7} -3.998  & \cellcolor{yellow!25} -5.002 \\
        5  & \cellcolor{yellow!3} 274  & \cellcolor{yellow!30} 2435  & \cellcolor{yellow!3} -5.000  & \cellcolor{yellow!30} -5.000  & \cellcolor{yellow!3} -5.065  & \cellcolor{yellow!30} -5.043  & \cellcolor{yellow!3} -5.058  & \cellcolor{yellow!30} -5.033  & \cellcolor{yellow!3} -4.966  & \cellcolor{yellow!30} -5.002 \\
        6  & \cellcolor{yellow!0} 35  & \cellcolor{yellow!10} 829  & \cellcolor{yellow!0} -6.000  & \cellcolor{yellow!10} -5.000  & \cellcolor{yellow!0} -6.078  & \cellcolor{yellow!10} -5.043  & \cellcolor{yellow!0} -6.070  & \cellcolor{yellow!10} -5.033  & \cellcolor{yellow!0} -5.904  & \cellcolor{yellow!10} -5.002 \\
        7  & \cellcolor{yellow!0} 8  & \cellcolor{yellow!6} 476  & \cellcolor{yellow!0} -7.000  & \cellcolor{yellow!6} -5.000  & \cellcolor{yellow!0} -7.091  & \cellcolor{yellow!6} -5.043  & \cellcolor{yellow!0} -7.082  & \cellcolor{yellow!6} -5.033  & \cellcolor{yellow!0} -6.769  & \cellcolor{yellow!6} -5.002 \\
        8  & \cellcolor{yellow!0} 1  & \cellcolor{yellow!3} 218  & \cellcolor{yellow!0} -8.000  & \cellcolor{yellow!3} -5.000  & \cellcolor{yellow!0} -8.104  & \cellcolor{yellow!3} -5.043  & \cellcolor{yellow!0} -8.093  & \cellcolor{yellow!3} -5.033  & \cellcolor{yellow!0} -7.633  & \cellcolor{yellow!3} -5.003 \\
        9  & \cellcolor{yellow!0} 0  & \cellcolor{yellow!1} 73  & \cellcolor{yellow!0} -9.000  & \cellcolor{yellow!1} -5.000  & \cellcolor{yellow!0} -9.117  & \cellcolor{yellow!1} -5.043  & \cellcolor{yellow!0} -9.105  & \cellcolor{yellow!1} -5.033  & \cellcolor{yellow!0} -8.497  & \cellcolor{yellow!1} -5.003 \\
        10 & \cellcolor{yellow!0} 0  & \cellcolor{yellow!0} 10  & \cellcolor{yellow!0} -10.000  & \cellcolor{yellow!0} -5.000  & \cellcolor{yellow!0} -10.130  & \cellcolor{yellow!0} -5.043  & \cellcolor{yellow!0} -10.117  & \cellcolor{yellow!0} -5.033  & \cellcolor{yellow!0} -9.361  & \cellcolor{yellow!0} -5.004 \\
        \bottomrule
    \end{tabular}
    }
    \caption{Reward estimation for 1,000 trajectories. Color indicates appearance frequencies.}
    \label{tab:r_traj_1000}
\end{table*}

\begin{table*}[h!]
    \centering
    \scalebox{0.85}{
    \begin{tabular}{lcc|cccccccc}
        \toprule
        \multirow{2}{*}{Mileage} & \multicolumn{2}{c|}{Frequency} & \multicolumn{2}{c}{Ground Truth Q} & \multicolumn{2}{c}{ML-IRL  Q} & \multicolumn{2}{c}{Rust Q} & \multicolumn{2}{c}{GLADIUS Q} \\
        \cmidrule(lr){2-3} \cmidrule(lr){4-5} \cmidrule(lr){6-7} \cmidrule(lr){8-9} \cmidrule(lr){10-11}
        & $a_0$ & $a_1$ & $a_0$ & $a_1$ & $a_0$ & $a_1$ & $a_0$ & $a_1$ & $a_0$ & $a_1$ \\
        \midrule
        1  & \cellcolor{yellow!100} 7994  & \cellcolor{yellow!33} 804  & \cellcolor{yellow!100} -52.534  & \cellcolor{yellow!33} -54.815  & \cellcolor{yellow!100} -53.110  & \cellcolor{yellow!33} -55.405  & \cellcolor{yellow!100} -53.019  & \cellcolor{yellow!33} -55.309  & \cellcolor{yellow!100} -52.431  & \cellcolor{yellow!33} -54.733 \\
        2  & \cellcolor{yellow!18} 1409  & \cellcolor{yellow!22} 541  & \cellcolor{yellow!18} -53.834  & \cellcolor{yellow!22} -54.815  & \cellcolor{yellow!18} -54.423  & \cellcolor{yellow!22} -55.405  & \cellcolor{yellow!18} -54.330  & \cellcolor{yellow!22} -55.309  & \cellcolor{yellow!18} -53.680  & \cellcolor{yellow!22} -54.720 \\
        3  & \cellcolor{yellow!13} 1060  & \cellcolor{yellow!53} 1296  & \cellcolor{yellow!13} -54.977  & \cellcolor{yellow!53} -54.815  & \cellcolor{yellow!13} -55.578  & \cellcolor{yellow!53} -55.405  & \cellcolor{yellow!13} -55.483  & \cellcolor{yellow!53} -55.309  & \cellcolor{yellow!13} -54.852  & \cellcolor{yellow!53} -54.721 \\
        4  & \cellcolor{yellow!7} 543  & \cellcolor{yellow!82} 1991  & \cellcolor{yellow!7} -56.037  & \cellcolor{yellow!82} -54.815  & \cellcolor{yellow!7} -56.649  & \cellcolor{yellow!82} -55.405  & \cellcolor{yellow!7} -56.554  & \cellcolor{yellow!82} -55.309  & \cellcolor{yellow!7} -55.942  & \cellcolor{yellow!82} -54.721 \\
        5  & \cellcolor{yellow!3} 274  & \cellcolor{yellow!100} 2435  & \cellcolor{yellow!3} -57.060  & \cellcolor{yellow!100} -54.815  & \cellcolor{yellow!3} -57.684  & \cellcolor{yellow!100} -55.405  & \cellcolor{yellow!3} -57.588  & \cellcolor{yellow!100} -55.309  & \cellcolor{yellow!3} -56.932  & \cellcolor{yellow!100} -54.721 \\
        6  & \cellcolor{yellow!0} 35  & \cellcolor{yellow!34} 829  & \cellcolor{yellow!0} -58.069  & \cellcolor{yellow!34} -54.815  & \cellcolor{yellow!0} -58.705  & \cellcolor{yellow!34} -55.405  & \cellcolor{yellow!0} -58.608  & \cellcolor{yellow!34} -55.309  & \cellcolor{yellow!0} -57.886  & \cellcolor{yellow!34} -54.721 \\
        7  & \cellcolor{yellow!0} 8  & \cellcolor{yellow!20} 476  & \cellcolor{yellow!0} -59.072  & \cellcolor{yellow!20} -54.815  & \cellcolor{yellow!0} -59.721  & \cellcolor{yellow!20} -55.405  & \cellcolor{yellow!0} -59.623  & \cellcolor{yellow!20} -55.309  & \cellcolor{yellow!0} -58.745  & \cellcolor{yellow!20} -54.721 \\
        8  & \cellcolor{yellow!0} 1  & \cellcolor{yellow!9} 218  & \cellcolor{yellow!0} -60.074  & \cellcolor{yellow!9} -54.815  & \cellcolor{yellow!0} -60.735  & \cellcolor{yellow!9} -55.405  & \cellcolor{yellow!0} -60.636  & \cellcolor{yellow!9} -55.309  & \cellcolor{yellow!0} -59.604  & \cellcolor{yellow!9} -54.722 \\
        9  & \cellcolor{yellow!0} 0  & \cellcolor{yellow!3} 73  & \cellcolor{yellow!0} -61.074  & \cellcolor{yellow!3} -54.815  & \cellcolor{yellow!0} -61.748  & \cellcolor{yellow!3} -55.405  & \cellcolor{yellow!0} -61.648  & \cellcolor{yellow!3} -55.309  & \cellcolor{yellow!0} -60.463  & \cellcolor{yellow!3} -54.722 \\
        10 & \cellcolor{yellow!0} 0  & \cellcolor{yellow!0} 10  & \cellcolor{yellow!0} -62.074  & \cellcolor{yellow!0} -54.815  & \cellcolor{yellow!0} -62.760  & \cellcolor{yellow!0} -55.405  & \cellcolor{yellow!0} -62.660  & \cellcolor{yellow!0} -55.309  & \cellcolor{yellow!0} -61.322  & \cellcolor{yellow!0} -54.722 \\
        \bottomrule
    \end{tabular}
    }
    \caption{$Q^\ast$ estimation for 1,000 trajectories. Color indicates appearance frequencies.}
    \label{tab:q_traj_1000}
\end{table*}

\newpage

% ===== Begin input: app_proofs.tex =====

\section{Equivalence between Dynamic Discrete choice and Entropy regularized Inverse Reinforcement learning}\label{sec:DDCIRL}

\subsection{Properties of Type 1 Extreme Value (T1EV) distribution}
Type 1 Extreme Value (T1EV), or Gumbel distribution, has a location parameter and a scale parameter. The T1EV distribution with location parameter $\nu$ and scale parameter 1 is denoted as Gumbel $(\nu, 1)$ and has its CDF, PDF, and mean as follows:
$$
\begin{gathered}
\text { CDF: } F(x ; \nu)=e^{-e^{-(x-\nu)}}
\\
\text { PDF: } f(x ; \nu)=e^{-\left((x-\nu)+e^{-(x-\nu)}\right)}
\\
\text { Mean } = \nu + \gamma
\end{gathered}
$$

Suppose that we are given a set of $N$ independent Gumbel random variables $G_i$, each with their own parameter $\nu_i$, i.e. $G_i \sim \operatorname{Gumbel}\left(\nu_i, 1\right)$.

\begin{lem}\label{lem:GumbelMax}
    Let $Z=\max G_i$. Then $Z \sim \operatorname{Gumbel}\left(\nu_Z=\log \sum_{i} e^{\nu_i}, 1\right)$.
\end{lem}
\begin{proof}
    $F_Z(x)=\prod_{i} F_{G_i}(x)=\prod_{i} e^{-e^{-\left(x-\nu_i\right)}}=e^{-\sum_{i} e^{-\left(x-\nu_i\right)}}=e^{-e^{-x} \sum_{i} e^{\nu_i}}=e^{-e^{-\left(x-\nu_Z\right)}}$
\end{proof}

\begin{cor}\label{cor:GumbelOptProb}
    $P\left(G_k>\max _{i \neq k} G_i\right)=\frac{e^{\nu_k}}{\sum_{i} e^{\nu_i}}$.
\end{cor}
\begin{proof}
\begin{align}
    P\left(G_k>\max _{i \neq k} G_i\right)&=\int_{-\infty}^{\infty} P\left(\max _{i \neq k} G_i<x\right) f_{G_{k}}(x) d x\notag
    \\&=\int_{-\infty}^{\infty} e^{-\sum_{i \neq k} e^{-\left(x-\nu_i\right)}} e^{-\left(x-\nu_k\right)} e^{-e^{-\left(x-\nu_k\right)}} d x \notag
    \\&=e^{\nu_k} \int_{-\infty}^{\infty} e^{-e^{-x} \sum_{i} e^{\nu_i}} e^{-x} d x \notag
    \\&=e^{\nu_k}\int_{\infty}^0 e^{-u S} u\left(-\frac{d u}{u}\right)  \; \; \left(\text{Let } \sum_{i} e^{\nu_i}=S, u=e^{-x}\right)\notag
    \\&=e^{\nu_k}\int_0^{\infty} e^{-u S} d u=e^{\nu_k}\left[-\frac{1}{S} e^{-u S}\right]_0^{\infty} = \frac{e^{\nu_k}}{S}\notag 
    \\&=\frac{e^{\nu_k}}{\sum_{i} e^{\nu_i}} \notag
\end{align}
\end{proof}

\begin{lem}\label{lem:ExpofLargerGumbel}
    Let $G_1\sim \text{Gumbel }(\nu_1, 1)$ and $G_2\sim \text{Gumbel }(\nu_2, 1)$. Then $\mathbb{E}\left[G_1 \mid G_1\geq G_2\right]=\gamma + \log \left( 1+ e^{\left(-(\nu_1-\nu_2)\right)} \right)$ holds. 
\end{lem}

\begin{proof}
    Let $\nu_1-\nu_2 = c$. Then $\mathbb{E}\left[G_1 \mid G_1\geq G_2\right]$ is equivalent to $\nu_1 + \frac{\int_{-\infty}^{+\infty} x F(x+c) f(x) \mathrm{d} x}{\int_{-\infty}^{+\infty} F(x+c) f(x) \mathrm{d} x}$, where the pdf $f$ and cdf $F$ are associated with $\text{Gumbel }(0, 1)$, because

    \begin{align}
        P\left(G_1 \leq x, G_1 \geq G_2\right)&=\int_{-\infty}^x F_{G_2}(t) f_{G_1}(t) d t=\int_{-\infty}^x F\left(t-\nu_2\right) f\left(t-\nu_1\right) d t\notag
        \\
        \mathbb{E}\left[G_1 \mid G_1 \geq G_2\right]&=\frac{\int_{-\infty}^{\infty} x F\left(x+c-\nu_1\right) f\left(x-\nu_1\right) d x}{\int_{-\infty}^{\infty} F\left(x+c-\nu_1\right) f\left(x-\nu_1\right) d x}\notag
        \\
        &=\frac{\int_{-\infty}^{\infty}\left(y+\nu_1\right) F(y+c) f(y) d y}{\int_{-\infty}^{\infty} F(y+c) f(y) d y} \notag
        \\
        &=\nu_1+\frac{\int_{-\infty}^{\infty} y F(y+c) f(y) d y}{\int_{-\infty}^{\infty} F(y+c) f(y) d y} \notag
    \end{align}

    Now note that 

    $\begin{aligned} \int_{-\infty}^{+\infty} F(x+c) f(x) \mathrm{d} x & =\int_{-\infty}^{+\infty} \exp \{-\exp [-x-c]\} \exp \{-x\} \exp \{-\exp [-x]\} \mathrm{d} x \\ & \stackrel{a=e^{-c}}{=} \int_{-\infty}^{+\infty} \exp \{-(1+a) \exp [-x]\} \exp \{-x\} \mathrm{d} x \\ & =\frac{1}{1+a}\left[\exp \left\{-(1+a) e^{-x}\right\}\right]_{-\infty}^{+\infty} \\ & =\frac{1}{1+a}\end{aligned}$
    \\
    and
    \\
    $\begin{aligned}  \int_{-\infty}^{+\infty} x F(x+c) f(x) \mathrm{d} x&=\int_{-\infty}^{+\infty} x \exp \{-(1+a) \exp [-x]\} \exp \{-x\} \mathrm{d} x \\ & \stackrel{z=e^{-x}}{=} \int_0^{+\infty} \log (z) \exp \{-(1+a) z\} \mathrm{d} z \\ & =\frac{-1}{1+a}\left[\operatorname{Ei}(-(1+a) z)-\log (z) e^{-(1+a) z}\right]_0^{\infty} \\ & =\frac{\gamma+\log (1+a)}{1+a} \\ & \end{aligned}$
    \\
    Therefore, $\mathbb{E}\left[G_1 \mid G_1\geq G_2\right]=\gamma + \nu_k+ \log \left( 1+ e^{\left(-(\nu_1-\nu_2)\right)} \right)$ holds.
\end{proof}

\begin{cor}\label{cor:GumbelMaxasProb} $ \mathbb{E}\left[G_k \mid G_k = \max G_i\right] = \gamma + \nu_k - \log \left(\frac{e^{\nu_k}}{\sum_{i} e^{\nu_i}}\right)$. 
\end{cor}
\begin{proof}

\begin{align}
    \mathbb{E}\left[G_k \mid G_k = \max G_i\right] &= \mathbb{E}\left[G_k \mid G_k \geq \max_{i\neq k} G_i\right]\notag
    \\
    &=\gamma + \nu_k + \log \left( 1+ e^{\left(-(\nu_k-\log\sum_{i\neq k} e^{\nu_i})\right)}\right)\tag{Lemma \ref{lem:ExpofLargerGumbel}}
    \\
    &=\gamma + \nu_k + \log \left( 1+ \frac{\sum_{i\neq k} e^{\nu_i}}{e^{-\nu_k}}\right) \notag
    \\
    &=\gamma + \nu_k + \log \left(\sum_{i} e^{\nu_i}/e^{\nu_k} \right)\notag
    \\
    &=\gamma + \nu_k -\log \left(e^{\nu_k} / \sum_{i} e^{\nu_i} \right)\notag
\end{align}

\end{proof}

\subsection{Properties of entropy regularization}
Suppose we have a choice out of discrete choice set $\mathcal{A} = \{x_i\}_{i=1}^{|\mathcal{A}|}$. A choice policy can be a deterministic policy such as $\operatorname{argmax}_{i \in 1, \ldots, |\mathcal{A}|} x_i$, or stochastic policy that is characterized by $\mathbf{q}\in \triangle_{\mathcal{A}}$. When we want to enforce smoothness in choice, we can regularize choice by newly defining the choice rule 
$$\arg\max _{\mathbf{q} \in \Delta_{\mathcal{A}}}\left(\langle\mathbf{q}, \mathbf{x}\rangle-\Omega(\mathbf{q})\right)$$
where $\Omega$ is a regularizing function. 

\begin{lem}\label{lem:logsumexp_Shannon}
When the regularizing function is constant $-\tau$ multiple of Shannon entropy $H(\mathbf{q})=-\sum_{i=1}^{|\mathcal{A}|} q_i \log \left(q_i\right)$, $$\max _{\mathbf{q} \in \Delta_{\mathcal{A}}}\left(\langle\mathbf{q}, \mathbf{x}\rangle-\Omega(\mathbf{q})\right)=\tau \log \left(\sum_i \exp \left(x_i / \tau\right)\right)$$ and

$$\arg\max _{\mathbf{q} \in \Delta_{\mathcal{A}}}\left(\langle\mathbf{q}, \mathbf{x}\rangle-\Omega(\mathbf{q})\right)=\frac{\exp \left(\frac{x_i}{\tau}\right)}{\sum_{j=1}^n \exp \left(\frac{x_j}{\tau}\right)}$$
\end{lem}
\begin{proof}
    In the following, I will assume $\tau>0$. Let
\begin{align}
G(\mathbf{q})&=\langle\mathbf{q}, \mathbf{x}\rangle-\Omega(\mathbf{q})\notag
\\
&=\sum_{i=1}^n q_i x_i-\tau \sum_{i=1}^n q_i \log \left(q_i\right)
\notag
\\
&=\sum_{i=1}^n q_i \left(x_i-\tau \log \left(q_i\right)\right) \notag
\end{align}

We are going to find the max by computing the gradient and setting it to 0 . We have
$$
\frac{\partial G}{\partial q_i}=x_i-\tau\left(\log \left(q_i\right)+1\right)
$$
and
$$
\frac{\partial G}{\partial q_i \partial q_j}=\left\{\begin{array}{l}
-\frac{\tau}{q_i}, \quad \text { if } i=j \\
0, \quad \text { otherwise. }
\end{array}\right.
$$
This last equation states that the Hessian matrix is negative definite (since it is diagonal and $-\frac{\tau}{q_1}<0$ ), and thus ensures that the stationary point we compute is actually the maximum. Setting the gradient to $\mathbf{0}$ yields $q_i^*=\exp \left(\frac{x_i}{\tau}-1\right)$, however the resulting $\mathbf{q}^*$ might not be a probability distribution. To ensure $\sum_{i=1}^n q_i^*=1$, we add a normalization:
$$
q_i^*=\frac{\exp \left(\frac{x_i}{\tau}-1\right)}{\sum_{j=1}^n \exp \left(\frac{x_j}{\tau}-1\right)}=\frac{\exp \left(\frac{x_i}{\tau}\right)}{\sum_{j=1}^n \exp \left(\frac{x_j}{\tau}\right)} .
$$
This new $\mathbf{q}^*$ is still a stationary point and belongs to the probability simplex, so it must be the maximum. Hence, you get
$$
\begin{aligned}
\max _{\tau H}(\mathbf{x})& =G\left(\mathbf{q}^*\right)=\sum_{i=1}^n \frac{\exp \left(\frac{x_i}{\tau}\right)}{\sum_{j=1}^n \exp \left(\frac{x_j}{\tau}\right)} x_i-\tau \sum_{i=1}^n \frac{\exp \left(\frac{x_1}{\tau}\right)}{\sum_{j=1}^n \exp \left(\frac{x_j}{\tau}\right)}\left(\frac{x_i}{\tau}-\log \sum_{i=1}^n \exp \left(\frac{x_j}{\tau}\right)\right) \\
&= \tau \log \sum_{i=1}^n \exp \left(\frac{x_j}{\tau}\right)
\end{aligned}
$$
as desired.
\end{proof}

\subsection{IRL with entropy regularization}
\label{sec:IRLentropy}

\subsubsection*{Markov decision processes} 
Consider an MDP defined by the tuple $\left(\mathcal{S}, \mathcal{A}, P, \nu_0, r, \beta\right)$:
\begin{itemize}
    \item $\mathcal{S}$ and $\mathcal{A}$ denote finite state and action spaces
    \item $P \in \Delta_{\mathcal{S}}^{\mathcal{S} \times \mathcal{A}}$ is a Markovian transition kernel, and $\nu_0 \in \Delta_{\mathcal{S}}$ is the initial state distribution. 
    \item $r \in \mathbb{R}^{\mathcal{S} \times \mathcal{A}}$ is a reward function.
    \item $\beta \in(0,1)$ a discount factor
\end{itemize}
\subsubsection{Agent behaviors}

Denote the distribution of the agent's initial state $s_0\in \mathcal{S}$ as $\nu_0$. Given a stationary Markov policy $\pi \in \Delta_{\mathcal{A}}^{\mathcal{S}}$, an agent starts from initial state $s_0$ and make an action $a_h\in \mathcal{A}$ at state $s_h\in \mathcal{S}$ according to $a_h\sim\pi\left(\cdot \mid s_h\right)$ at each period $h$. We use $\mathbb{P}_{\nu_0}^\pi$ to denote the distribution over the sample space $(\mathcal{S} \times \mathcal{A})^{\infty}=\left\{\left(s_0, a_0, s_1, a_1, \ldots\right): s_h \in \mathcal{S}, a_h \in \mathcal{A}, h \in \mathbb{N}\right\}$ induced by the policy $\pi$ and the initial distribution $\nu_0$. We also use $\mathbb{E}_\pi$ to denote the expectation with respect to $\mathbb{P}_{\nu_0}^\pi$. Maximum entropy inverse reinforcement learning (MaxEnt-IRL) makes the following assumption: 

\begin{asmp}[Assumption \ref{ass:IRLoptimaldecision}] Agent follows the policy 
$$
\pi^*=\operatorname{argmax}_{\pi \in \Delta_{\mathcal{A}}^{\mathcal{S}}}\mathbb{E}_\pi\left[\sum_{h=0}^{\infty} \beta^h \left(r\left(s_h, a_h\right)+\lambda\mathcal{H}\left(\pi\left(\cdot \mid s_h\right)\right)\right)\right]
$$
where $\mathcal{H}$ denotes the Shannon entropy and $\lambda$ is the regularization parameter.
\end{asmp}
For the rest of the section, we use $\lambda=1$. We then define the function $V$ as $V(s_{h^\prime})=\mathbb{E}_{\pi^*}\left[\sum_{h=h^\prime}^{\infty} \beta^{h} \left(r\left(s_h, a_h\right)+\mathcal{H}\left(\pi^*\left(\cdot \mid s_h\right)\right)\right)\right]$ and call it the \textit{value function}. According to Assumption \ref{ass:IRLoptimaldecision}, the value function $V$ must satisfy the Bellman equation, i.e., 
\begin{align}
V\left(s\right)&=\max _{\mathbf{q} \in \triangle_\mathcal{A}}\left\{\mathbb{E}_{a\sim\mathbf{q} }\left[r\left(s, a\right)+\beta \cdot \mathbb{E}\left[V\left(s^\prime\right)\mid s, a\right]\right]+\mathcal{H}(\mathbf{q})\right\}\notag
\\
&=\max _{\mathbf{q} \in \triangle_\mathcal{A}}\left\{\sum_{a\in \mathcal{A}} q_a\left(r\left(s, a\right)+\beta \cdot \mathbb{E}\left[V\left(s^\prime\right)\mid s, a\right]\right)+\mathcal{H}(\mathbf{q})\right\}=\max _{\mathbf{q} \in \triangle_\mathcal{A}}\left\{\sum_{a\in \mathcal{A}} q_a Q(s,a)+\mathcal{H}(\mathbf{q})\right\}\label{eq:VandmaxQ}
\\
&=\ln \left[\sum_{a\in \mathcal{A}}\exp\left(r\left(s, a\right)+\beta \cdot \mathbb{E}\left[{V}\left(s^\prime\right)\mid s, a\right]\right)\right]\label{eq:logsumexp_reg}
\\
&=\ln \left[\sum_{a\in \mathcal{A}}\exp\left(Q(s, a)\right)\right]\label{eq:IRLlogsumQ}
\end{align}
and $\mathbf{q}^*:=\arg\max_{\mathbf{q} \in \triangle_\mathcal{A}} \left\{\mathbb{E}_{a\sim\mathbf{q} }\left[r\left(s, a\right)+\beta \cdot \mathbb{E}\left[V\left(s^\prime\right)\mid s, a\right]\right]+\mathcal{H}(\mathbf{q})\right\}$ is characterized by
\\
\begin{align}
\mathbf{q}^* = [q_1^* \ldots q^*_{|\mathcal{A}|}], \text{ where }
    q^*_a= \frac{\exp \left({Q(s, a)}\right)}{\sum_{a^\prime\in \mathcal{A}} \exp \left({Q(s, a^\prime)}\right)} \text{ for } a\in \mathcal{A}  \label{eq:IRLopt}
\end{align}
where:
\begin{itemize}
    \item $Q(s, a):=r\left(s, a\right)+\beta \cdot \mathbb{E}\left[{V}\left(s^\prime\right)\mid s, a\right]$
    \item Equality in Equation \eqref{eq:logsumexp_reg} and equality in Equation \eqref{eq:IRLopt} is from Lemma \ref{lem:logsumexp_Shannon}
\end{itemize}
This implies that $
    \pi^*(a\mid s) = q^*_a= \frac{\exp \left({Q(s, a)}\right)}{\sum_{a^\prime\in \mathcal{A}} \exp \left({Q(s, a^\prime)}\right)} \text{ for } a\in \mathcal{A}.$ In addition to the Bellman equation in terms of value function $V$, 
Bellman equation in terms of choice-specific value function $Q(s,a)$ can be derived by combining $Q(s, a):=r\left(s, a\right)+\beta \cdot \mathbb{E}\left[{V}\left(s^\prime\right)\mid s, a\right]$ and Equation \eqref{eq:IRLlogsumQ}:
\begin{align}
    Q(s, a)=r(s, a)+\beta \mathbb{E}_{s^\prime \sim P(s, a)}\left[\ln \left(\sum_{a^{\prime} \in \mathcal{A}} \exp \left(Q\left(s^{\prime}, a^{\prime}\right)\right)\right)\mid s, a\right] \notag
\end{align}
\;
\\
We can also derive an alternative form of choice-specific value function $Q(s,a)$ by combining $Q(s, a):=r\left(s, a\right)+\beta \cdot \mathbb{E}_{s^\prime \sim P(s, a)}\left[{V}\left(s^\prime\right)\mid s, a\right]$ and Equation \eqref{eq:VandmaxQ}:

\begin{align}
    Q(s, a)&=r\left(s, a\right)+\beta \cdot \mathbb{E}_{s^\prime \sim P(s, a)}\left[\max _{\mathbf{q} \in \triangle_\mathcal{A}}\left\{\sum_{a\in \mathcal{A}} q_a Q(s^\prime,a)+\mathcal{H}(\mathbf{q})\right\}\mid s, a\right]\notag
    \\
    &=r\left(s, a\right)+\beta \cdot \mathbb{E}_{s^\prime \sim P(s, a)}\left[\max _{\mathbf{q} \in \triangle_\mathcal{A}}\left\{\sum_{a\in \mathcal{A}} q_a \left(Q(s^\prime,a) - \log q_a\right)\right\}\mid s, a\right]\notag
    \\
    &=r\left(s, a\right)+\beta \cdot \mathbb{E}_{s^\prime \sim P(s, a), a^\prime \sim \pi^*(a\mid \cdot)}\left[ \left(Q(s^\prime,a^\prime) - \log \pi^*(a^\prime\mid s^\prime)\right)\mid s, a\right] \label{eqn:IRLQBellman_new}
    \\
    &=r\left(s, a\right)+\beta \cdot \mathbb{E}_{s^\prime \sim P(s, a)}\left[ \left(Q(s^\prime,a^\prime) - \log \pi^*(a^\prime\mid s^\prime)\right)\mid s, a\right]\text{ for all } a^\prime\in \mathcal{A} \notag
\end{align}
The last line comes from the fact that $Q(s^\prime,a^\prime) - \log \pi^*(a^\prime\mid s^\prime)$ is equivalent to $\log \left(\sum_{a^{\prime} \in \mathcal{A}} \exp \left(Q\left(s^{\prime}, a^{\prime}\right)\right)\right)$, which is a quantity that does not depend on the realization of specific action $a^\prime$.

\subsection{Single agent Dynamic Discrete Choice (DDC) model}
\label{sec:SingleDDC}

\subsubsection*{Markov decision processes} 
Consider an MDP $\tau:=\left(\mathcal{S}, \mathcal{A}, P, \nu_0, r, G(\delta,1), \beta \right)$:
\begin{itemize}
    \item $\mathcal{S}$ and $\mathcal{A}$ denote finite state and action spaces
    \item $P \in \Delta_{\mathcal{S}}^{\mathcal{S} \times \mathcal{A}}$ is a Markovian transition kernel, and $\nu_0 \in \Delta_{\mathcal{S}}$ is the initial state distribution. 
    \item $r(s_h,a_h)+\epsilon_{ah}$ is the immediate reward (called the flow utility in the Discrete Choice Model literature) from taking action $a_h$ at state $s_h$ at time-step $h$, where:
    \begin{itemize}
        \item $r \in \mathbb{R}^{\mathcal{S} \times \mathcal{A}}$ is a deterministic reward function
        \item  
    $\epsilon_{ah}\overset{i.i.d.}{\sim}  G(\delta, 1)$ is the random part of the reward, where $G$ is Type 1 Extreme Value (T1EV) distribution (a.k.a. Gumbel distribution). The mean of $G(\delta, 1)$ is $\delta + \gamma$, where $\gamma$ is the Euler constant. 
    \item In the econometrics literature, this reward setting is a result of the combination of two assumptions: conditional independence (CI) and additive separability (AS) \cite{magnac2002identifying}. 
\begin{figure}[H]
    \centering
    \includegraphics[width=0.3\linewidth]{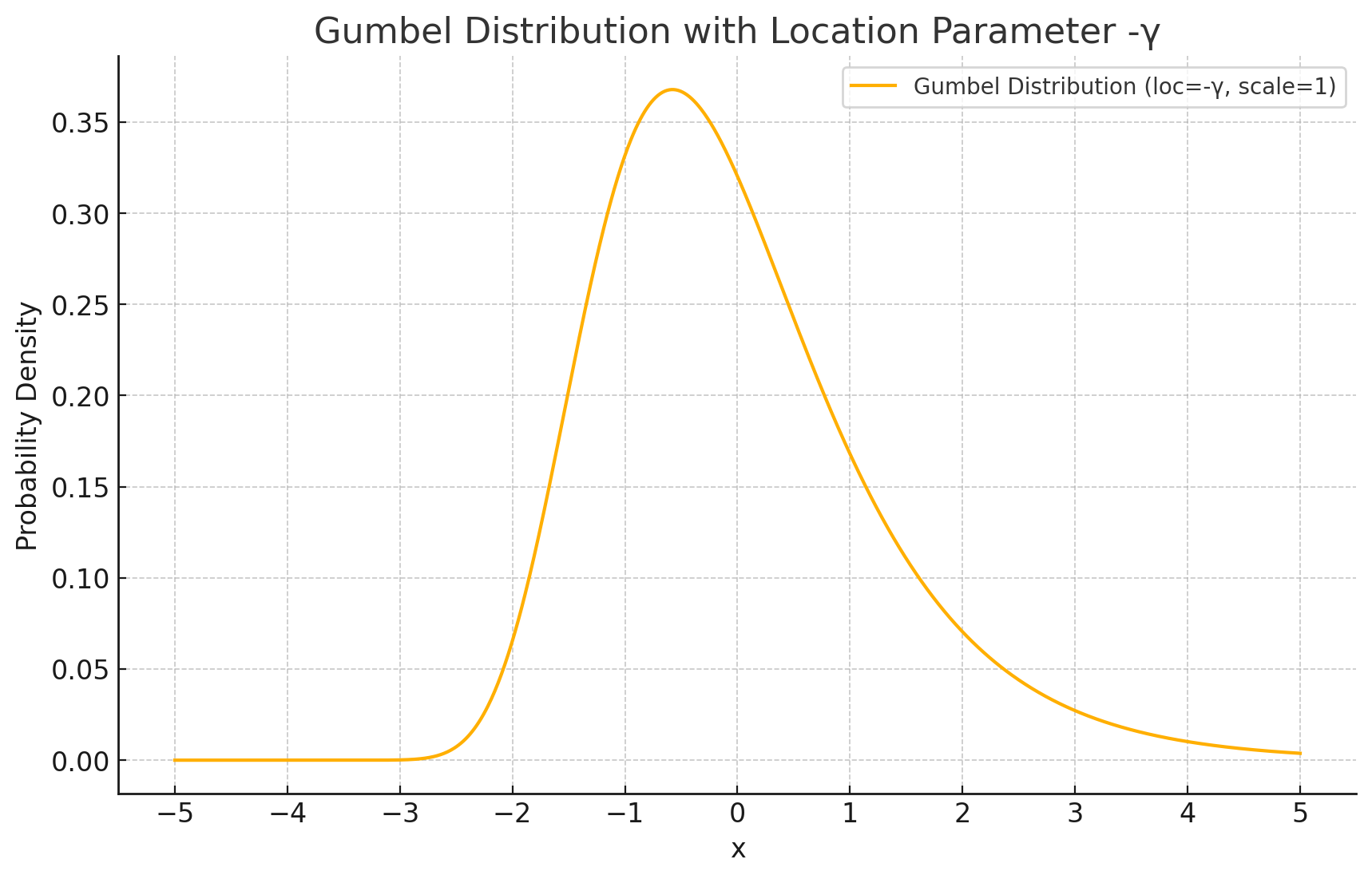}
    \caption{Gumbel distribution $G(-\gamma, 1)$}
\end{figure}
    \end{itemize}
    \item $\beta \in(0,1)$ a discount factor  
\end{itemize}

\subsubsection{Agent behaviors} Denote the distribution of agent's initial state $s_0\in \mathcal{S}$ as $\nu_0$. Given a stationary Markov policy $\pi \in \Delta_{\mathcal{A}}^{\mathcal{S}}$, an agent starts from initial state $s_0$ and make an action $a_h\in \mathcal{A}$ at state $s_h\in \mathcal{S}$ according to $a_h\sim\pi\left(\cdot \mid s_h\right)$ at each period $h$. We use $\mathbb{P}_{\nu_0}^\pi$ to denote the distribution over the sample space $(\mathcal{S} \times \mathcal{A})^{\infty}=\left\{\left(s_0, a_0, s_1, a_1, \ldots\right): s_h \in \mathcal{S}, a_h \in \mathcal{A}, h \in \mathbb{N}\right\}$ induced by the policy $\pi$ and the initial distribution $\nu_0$. We also use $\mathbb{E}_\pi$ to denote the expectation with respect to $\mathbb{P}_{\nu_0}^\pi$. As in Inverse Reinforcement learning (IRL), a Dynamic Discrete Choice (DDC) model makes the following assumption: 

\begin{asmp}\label{ass:optimaldecision} Agent makes decision according to the policy $\operatorname{argmax}_{\pi \in \Delta_{\mathcal{A}}^{\mathcal{S}}}$
$\mathbb{E}_\pi\left[\sum_{h=0}^{\infty} \beta^h( r\left(s_h, a_h\right)+\epsilon_{ah})\right]$.
\end{asmp}

As Assumption \ref{ass:optimaldecision} specifies the agent's policy, we omit $\pi$ in the notations from now on. Define $\boldsymbol{\epsilon}_h = [\epsilon_{1h}\ldots \epsilon_{|\mathcal{A}|h}]$, where $\epsilon_{ih}\overset{i.i.d}{\sim} G(\delta, 1)$ for $i=1\ldots |\mathcal{A}|$. We define a function $V$ as
\begin{align}
    V\left(s_{h^\prime}, \boldsymbol{\epsilon_{h^\prime}}\right) = \mathbb{E}\left[\sum_{h=h^\prime}^{\infty} \beta^h( r\left(s_h, a_h\right)+\epsilon_{ah})\mid s_{h^\prime}\right] \notag
\end{align}
and call it the value function. According to Assumption \ref{ass:optimaldecision}, the value function $V$ must satisfy the Bellman equation, i.e., 
\begin{align}
V\left(s, \boldsymbol{\epsilon}\right)=\max _{a \in \mathcal{A}}\left\{r\left(s, a\right)+\epsilon_{a}+\beta \cdot \mathbb{E}_{s^\prime \sim P(s, a), \boldsymbol{\epsilon^\prime }\sim \boldsymbol{\epsilon}}\left[V\left(s^\prime, \boldsymbol{\epsilon}^\prime\right)\mid s, a\right]\right\} \label{eq:VBellman}.
\end{align}
\;
\\
Define 
\begin{align}
    \bar{V}\left(s\right) &\triangleq E_{\boldsymbol{\epsilon}}\left[V\left(s, \boldsymbol{\epsilon}\right)\right] \notag
    \\
    Q(s, a) &\triangleq r\left(s, a\right)+\beta \cdot \mathbb{E}_{s^\prime \sim P(s, a)}\left[\bar{V}\left(s^\prime\right)\mid s, a\right]\label{eq:QandexpV}
\end{align}
We call $\bar{V}$ the expected value function, and $Q(s, a)$ as the choice-specific value function. Then the Bellman equation can be written as

\;
\begin{align}
\bar{V}\left(s\right) &=\mathbb{E}_{
\boldsymbol{\epsilon}}\left[\max _{a \in \mathcal{A}}\left\{r\left(s, a\right)+\epsilon_{a}+\beta \cdot \mathbb{E}\left[\bar{V}\left(s^\prime\right)\mid s, a\right]\right\}\right]\label{eq:DP_DDC_pre} 
\\
&=\ln \left[\sum_{a\in \mathcal{A}}\exp\left(r\left(s, a\right)+\beta \cdot \mathbb{E}\left[\bar{V}\left(s^\prime\right)\mid s, a\right]\right)\right] + \delta + \gamma \tag{$\because$ Lemma \ref{lem:GumbelMax}}\label{eq:DP_DDC}
\\
&=\ln \left[\sum_{a\in \mathcal{A}}\exp\left(Q(s,a)\right)\right]  + \delta + \gamma \label{eq:logsumQ}
\end{align}
\\
\;
\\
Furthermore, Corollary \ref{cor:GumbelOptProb} characterizes that the agent's optimal policy is characterized by 
\begin{align}
    \pi^*(a \mid s) =\frac{\exp \left({Q(s, a)}\right)}{\sum_{a^\prime\in \mathcal{A}} \exp \left({Q(s, a^\prime)}\right)} \text{ for } a\in \mathcal{A} \label{eq:DDCopt}
\end{align}
\;
\\
In addition to Bellman equation in terms of value function $V$ in Equation \eqref{eq:VBellman}, 
Bellman equation in terms of choice-specific value function $Q$ comes from combining Equation \eqref{eq:QandexpV} and Equation \eqref{eq:logsumQ}:
\begin{align}
    Q(s, a)=r(s, a)+\beta \mathbb{E}_{s^\prime \sim P(s, a)}\left[\ln \left(\sum_{a^{\prime} \in \mathcal{A}} \exp \left(Q\left(s^{\prime}, a^{\prime}\right)\right)\right)\mid s, a\right] + \delta + \gamma
\end{align}
\;
\\
When $\delta = -\gamma$ (i.e., the Gumbel noise is mean 0), we have 
\begin{align}
    Q(s, a)=r(s, a)+\beta \mathbb{E}_{s^\prime \sim P(s, a)}\left[\ln \left(\sum_{a^{\prime} \in \mathcal{A}} \exp \left(Q\left(s^{\prime}, a^{\prime}\right)\right)\right)\mid s, a\right] \tag{ \ref{eq:QBellmanDDC}}
\end{align}
\;
\\
This Bellman equation can be also written in another form.

\begin{align}
  Q(s, a) &\triangleq r\left(s, a\right)+\beta \cdot \mathbb{E}_{s^\prime \sim P(s, a)}\left[\bar{V}\left(s^\prime\right)\mid s, a\right]\tag{Equation \ref{eq:QandexpV}}
  \\
  &= r\left(s, a\right)+\beta \cdot \mathbb{E}_{s^\prime \sim P(s, a), \boldsymbol{\epsilon}^\prime\sim \boldsymbol{\epsilon} }\left[{V}\left(s^\prime, \boldsymbol{\epsilon}^\prime\right)\mid s, a\right]\notag
  \\
  &=r\left(s, a\right)+\beta \cdot \mathbb{E}_{s^\prime \sim P(s, a),  \boldsymbol{\epsilon}^\prime\sim \boldsymbol{\epsilon}}\left[\max_{a^\prime\in \mathcal{A}} \left(Q\left(s^\prime, a^\prime\right)+\epsilon^\prime_a\right)\mid s, a\right] \label{eq:AnotherQBellman}
  \\
  &=r\left(s, a\right)+\beta \cdot \mathbb{E}_{s^\prime \sim P(s, a), a^\prime \sim \pi^*(\cdot\mid s^\prime)}\left[Q(s^\prime, a^\prime) + \delta + \gamma - \log  \pi^*(a^\prime \mid s^\prime) \mid s, a\right] \makebox[2em][l]{\quad(Corollary \ref{cor:GumbelMaxasProb})} \notag
  \\\label{DDCBellman_new}
\end{align}
where $\pi^*(s, a) = \left(\frac{Q(s, a)}{\sum_{a^{\prime}\in \mathcal{A}}Q(s, a^{\prime})}\right)$.

\subsection{Equivalence between DDC and Entropy regularized IRL}%\label{sec:DDCIRLequiv}

Equation \ref{eq:logsumexp_reg}, Equation \ref{eq:IRLopt} and Equation \ref{eqn:IRLQBellman_new} characterizes the choice-specific value function's Bellman equation and optimal policy in entropy regularized IRL setting when regularizing coefficient is 1: 
$$
    Q(s, a)=r(s, a)+\beta \mathbb{E}_{s^\prime \sim P(s, a)}\left[\ln \left(\sum_{a^{\prime} \in \mathcal{A}} \exp \left(Q\left(s^{\prime},a^{\prime}\right)\right)\right)\mid s, a\right]
$$

$$
\pi^*(a \mid s) =\frac{\exp \left({Q(s, a)}\right)}{\sum_{a^\prime\in \mathcal{A}} \exp \left({Q(s, a^\prime)}\right)} \text{ for } a\in \mathcal{A} 
$$

$$
Q(s,a)=r\left(s, a\right)+\beta \cdot \mathbb{E}_{s^\prime \sim P(s, a), a^\prime \sim \pi^*(\cdot \mid s^\prime)}\left[Q(s^\prime, a^\prime) - \log  \pi^*(a^\prime\mid s^\prime) \mid s, a\right]
$$
\\
Equation \ref{eq:DDCopt}, Equation \ref{eq:QBellmanDDC}, and Equation \ref{DDCBellman_new} (when $\delta = -\gamma$) characterizes the choice-specific value function's Bellman equation and optimal policy of Dynamic Discrete Choice setting:
$$
    Q(s, a)=r(s, a)+\beta \mathbb{E}_{s^\prime \sim P(s, a)}\left[\ln \left(\sum_{a^{\prime} \in \mathcal{A}} \exp \left(Q\left(s^{\prime}, a^{\prime}\right)\right)\right)\mid s, a\right] 
$$

$$
\pi^*(a \mid s) =\frac{\exp \left({Q(s, a)}\right)}{\sum_{a^\prime\in \mathcal{A}} \exp \left({Q(s, a^\prime)}\right)} \text{ for } a\in \mathcal{A} 
$$

$$
Q(s,a) = r\left(s, a\right)+\beta \cdot \mathbb{E}_{s^\prime \sim P(s, a), a^\prime \sim \pi^*(\cdot\mid s^\prime)}\left[Q(s^\prime, a^\prime) - \log  \pi^*(a^\prime \mid s^\prime) \mid s, a\right]
$$
\\
$Q$ that satisfies \ref{eq:DDCopt} is unique \cite{rust1994structural}, and $Q-r$ forms a one-to-one relationship. Therefore, the exact equivalence between these two setups implies that the same reward function $r$ and discount factor $\beta$ will lead to the same choice-specific value function $Q$ and the same optimal policy for the two problems.

\subsection{IRL with occupancy matching}\label{sec:occupancy}

\cite{ho2016generative} defines another inverse reinforcement learning problem that is based on the notion of occupancy matching. Let $\nu_0$ be the initial state distribution and $d^\pi$ be the discounted state-action occupancy of $\pi$ which is defined as $ d^\pi=(1-$ $\beta) \sum_{t=0}^{\infty} \beta^t d_t^\pi$, with $d_t^\pi(s, a)=\mathbb{P}_{\pi, \nu_0}\left[s_t=s, a_t=a\right]$. Note that $Q^\pi(s, a):=\mathbb{E}_\pi\left[\sum_{t=0}^{\infty} \beta^t r(s_t,a_t) \mid s_0=s, a_0=a\right] = \sum_{t=0}^{\infty} \beta^t \mathbb{E}_{(\tilde{s},\tilde{a})\sim d_t^\pi}[r(\tilde{s},\tilde{a})\mid s_0=a,a_0=a].$ Defining the discounted state-action occupancy of the expert policy $\pi^\ast$ as $d^\ast$, \cite{ho2016generative} defines the inverse reinforcement learning problem as the following max-min problem:
\begin{align}
    \underset{r \in \mathcal{C}}{\operatorname{max}}\underset{{\pi \in \Pi}}{\min}\left(\mathbb{E}_{d^\ast}[r(s, a)]-\mathbb{E}_{d^\pi}[r(s, a)]-\mathcal{H}(\pi)-\psi(r)\right) \label{eq:occupancyObj}
\end{align}
where $\mathcal{H}$ is the Shannon entropy we used in MaxEnt-IRL formulation and $\psi$ is the regularizer imposed on the reward model $r$.  

Would occupancy matching find $Q$ that satisfies the Bellman equation? Denote the policy as $\pi^\ast$ and its corresponding discounted state-action occupancy measure as $d^\ast=(1-$ $\beta) \sum_{t=0}^{\infty} \beta^t d_t^\ast$, with $d_t^\ast(s, a)=\mathbb{P}_{\pi^\ast, \nu_0}\left[s_t=s, a_t=a\right]$. We define the expert's action-value function as $Q^\ast(s, a):=\mathbb{E}_{\pi^\ast}\left[\sum_{t=0}^{\infty} \beta^t r(s_t,a_t) \mid s_0=s, a_0=a\right]$ and the Bellman operator of $\pi^\ast$ as $\mathcal{T}^\ast$. Then we have the following Lemma \ref{cor:occp=naiveBE} showing that occupancy matching (even without regularization) may not minimize Bellman error for every state and action.
\begin{lem}[Occupancy matching is equivalent to naive weighted Bellman error sum]\label{cor:occp=naiveBE} The perfect occupancy matching given the same $(s_0, a_0)$ satisfies
    $$\mathbb{E}_{(s, a) \sim d^\ast}[r(s, a)\mid s_0, a_0]-\mathbb{E}_{(s, a) \sim d^\pi}[r(s, a)\mid s_0, a_0] = \mathbb{E}_{(s,a)\sim d^{\ast}}[(\mathcal{T}^\ast Q^\pi-Q^\pi)(s,a)\mid s_0, a_0]$$
\end{lem}
\begin{proof} 
   Note that $\mathbb{E}_{(s, a) \sim d^{\ast}}[r(s, a)\mid s_0, a_0]=\sum_{t=0}^{\infty} \beta^t \mathbb{E}_{(s,a)\sim d_t^\ast}[r(s,a)\mid s_0,a_0]=Q^\ast(s, a)$ and $\mathbb{E}_{(s, a) \sim d^{\pi}}[r(s_, a)\mid s_0, a_0]=\sum_{t=0}^{\infty} \beta^t \mathbb{E}_{(s,a)\sim d_t^\pi}[r(s,a)\mid s_0,a_0] = Q^\pi(s, a)$. Therefore
   \begin{align}
       &\mathbb{E}_{(s, a) \sim d^\ast}[r(s, a)\mid s_0,a_0]-\mathbb{E}_{(s, a) \sim d^\pi}[r(s, a)\mid s_0,a_0] =(1-\beta)Q^\ast(s_0, a_0)-(1-\beta)Q^\pi(s_0, a_0)\notag
       \\
       &=(1-\beta)\frac{1}{1-\beta} \mathbb{E}_{(s,a)\sim d^{\ast}}[(\mathcal{T}^\ast Q^\pi-Q^\pi)(s,a)\mid s_0, a_0]\tag{Lemma \ref{lem:telescoping}}
       \\
       &= \mathbb{E}_{(s,a)\sim d^{\ast}}[(\mathcal{T}^\ast Q^\pi-Q^\pi)(s,a)\mid s_0, a_0] \notag
   \end{align}
\end{proof}
Lemma \ref{cor:occp=naiveBE} implies that occupancy measure matching, even without reward regularization, does not necessarily imply Bellman errors being 0 for every state and action. In fact, what they minimize is the \textit{average Bellman error} \cite{jiang2017contextual,uehara2020minimax}. This implies that $r$ cannot be inferred from $Q$ using the Bellman equation after deriving $Q$ using occupancy matching. 

\begin{lem}[Bellman Error Telescoping]\label{lem:telescoping} 
Let the Bellman operator $\mathcal{T}^\pi$ is defined to map $f\in \mathbb{R}^{S\times A}$ to $\mathcal{T}^\pi f := r(s,a) + \mathbb{E}_{s^\prime\sim P(s,a), a^\prime\sim \pi(\cdot\mid s^\prime)}[f(s^\prime,a^\prime)\mid s,a]$. 
For any $\pi$, and any $f\in \mathbb{R}^{S\times A}$,
$$
Q^\pi(s_0, a_0)-f(s_0, a_0) = \frac{1}{1-\beta} \mathbb{E}_{(s,a)\sim d^\pi}[(\mathcal{T}^\pi f-f)(s,a)\mid s_0, a_0].
$$

\end{lem}

\begin{proof}
Note that the right-hand side of the statement can be expanded as
\begin{align}
    &r(s_0, a_0)+\beta \cancel{\mathbb{E}_{s^\prime\sim P(s,a), a^\prime\sim \pi(\cdot\mid s^\prime)}[f(s^\prime,a^\prime)\mid s,a]}-f(s_0,a_0)\notag
    \\
    &+\beta\mathbb{E}_{(s,a)\sim d_1^\pi}\left[r(s, a)+\beta \cancel{\mathbb{E}_{s^\prime\sim P(s,a), a^\prime\sim \pi(\cdot\mid s^\prime)}[f(s^\prime,a^\prime)\mid s,a]}-\cancel{f(s,a)}\mid s_0, a_0\right]\notag
    \\
    &+\beta^2\mathbb{E}_{(s,a)\sim d_2^\pi}\left[r(s, a)+\beta \cancel{\mathbb{E}_{s^\prime\sim P(s,a), a^\prime\sim \pi(\cdot\mid s^\prime)}[f(s^\prime,a^\prime)\mid s,a]}-\cancel{f(s,a)}\mid s_0, a_0\right]\notag
    \\&\quad\quad\quad\quad\quad\quad\quad\quad\quad\quad\quad\quad\ldots\notag
    \\
    &=Q^\pi(s_0, a_0)-f(s_0, a_0)\notag
\end{align}
which is the left-hand side of the statement.

\end{proof}

\iffalse

\begin{lem}[Equivalence between Occupancy matching and Behavioral Cloning]
The solution set to the Occupancy matching objective (Equation \eqref{eq:occupancyObj}) without regularization terms is equivalent to the solution set to the behavioral cloning objective (Equation \eqref{eq:BC}).
\end{lem}

\begin{proof}
    In proving Lemma \ref{cor:occp=naiveBE}, we saw that $\mathbb{E}_{(s, a) \sim d^\ast}[r(s, a)\mid s_0,a_0]-\mathbb{E}_{(s, a) \sim d^\pi}[r(s, a)\mid s_0,a_0] =(1-\beta)Q^\ast(s_0, a_0)-(1-\beta)Q^\pi(s_0, a_0)$. 

    Now from Lemma \ref{lem:minMLE}, we have 
      \begin{align}
           \underset{Q\in \mathcal{Q}}{\arg\max } &\; \;\mathbb{E}_{(s, a)\sim \pi^*, \nu_0}  \left[\log\left(\hat{p}_{Q}(\;\cdot
    \mid s)\right)\right] \notag
    \\
     &=\left\{Q \in \mathcal{Q} \mid\hat{p}_{Q}(\;\cdot
    \mid s) = \pi^*(\;\cdot
    \mid s)\quad  \forall s\in\bar{\mathcal{S}}\quad\text{a.e.}\right\}\notag
    \\
     &=\left\{Q \in \mathcal{Q} \mid Q(s,a_1)-Q(s,a_2)= Q^*(s,a_1)-Q^*(s,a_2) \quad \forall a_1, a_2\in\mathcal{A}, s\in\bar{\mathcal{S}}\right\} \notag
    \end{align}
This concludes that the solution set to the behavioral cloning objective is equivalent to the occupancy matching objective without the regularization term.
\end{proof}

\fi
\section{Technical Proofs}

\subsection{Theory of TD correction using biconjugate trick}\label{sec:BiconjProofs}

\begin{proof}[Proof of Lemma \ref{lem:OurBiconj}]
      \begin{align}
     &\mathcal{L}_{BE}(s,a)(Q):=\mathbb{E}_{ s^\prime \sim P(s, a)}\left[\delta_{Q}\left(s,a, s^\prime\right)\mid s, a\right]^2\notag
     \\
     &=\max_{h\in \mathbb{R}}2\cdot\mathbb{E}_{ s^\prime \sim P(s, a)}\left[\delta_{Q}\left(s,a, s^\prime\right)\mid s, a\right]\cdot h-h^2\tag{Biconjugate}
     \\
     &=\max_{h\in \mathbb{R}}2\cdot\mathbb{E}_{ s^\prime \sim P(s, a)}\left[\hat{\mathcal{T}}Q - Q\mid s, a\right]\cdot \underbrace{h}_{=\rho-Q(s,a)}-h^2\notag
    \\
     &=\max_{\rho(s,a)\in \mathbb{R}}\mathbb{E}_{ s^\prime \sim P(s, a)}\left[2\left(\hat{\mathcal{T}}Q - Q\right)\left(\rho-Q\right)
    -\left(\rho-Q\right)^2\mid s, a\right]\notag
    \\
    &= \max_{\rho(s,a)\in \mathbb{R}}\mathbb{E}_{ s^\prime \sim P(s, a)}\left[   \left(\hat{\mathcal{T}}{Q}-Q\right)^2-\left(\hat{\mathcal{T}}{Q}- \rho\right)^2\mid s, a\right]\label{eq:SBEED} 
    \end{align}
where the unique maximum is with 
\begin{align}
    \rho^{*}(s,a) &= h^{*}(s,a)+Q(s,a) = \mathcal{T}Q(s,a)- Q(s,a)+Q(s,a)\notag
    \\    &=\mathcal{T}Q(s,a)\notag
\end{align}
and where the equality of \ref{eq:SBEED} is from
\begin{align}
&\!\!\!\!\!\!\!\!\!2\left(\hat{\mathcal{T}}Q - Q\right)\left(\rho-Q\right)
    -\left(\rho-Q\right)^2\notag
    \\
    &=2(\hat{\mathcal{T}}{Q}\rho - \hat{\mathcal{T}}{Q}Q - \cancel{Q\rho} + Q^2)  - (\rho^2 - \cancel{2Q\rho} + Q^2)\notag
    \\
    &=2\hat{\mathcal{T}}{Q}\rho - 2\hat{\mathcal{T}}{Q}Q +\cancel{2} Q^2- \rho^2 - \cancel{Q^2}\notag
    \\
    &=\hat{\mathcal{T}}{Q}^2- 2\hat{\mathcal{T}}{Q}Q+Q^2-\hat{\mathcal{T}}{Q}^2+2\hat{\mathcal{T}}{Q}\rho- \rho^2\notag
    \\
    &=\left(\hat{\mathcal{T}}{Q}-Q\right)^2-\left(\hat{\mathcal{T}}{Q}- \rho\right)^2\notag
\end{align}

Now note that    
     \begin{align}
     &\mathcal{L}_{BE}(s,a)(Q)= \max_{\rho(s,a)\in \mathbb{R}}\mathbb{E}_{ s^\prime \sim P(s, a)}\left[   \left(\hat{\mathcal{T}}{Q}-Q\right)^2-\left(\hat{\mathcal{T}}{Q}- \rho\right)^2\mid s, a\right]] \tag{equation \ref{eq:SBEED}}
    \\
    &= \mathbb{E}_{ s^\prime \sim P(s, a)}\left[ \left(\hat{\mathcal{T}}{Q}-Q\right)^2\mid s,a \right]-\min_{\rho(s,a)\in \mathbb{R}}\mathbb{E}_{ s^\prime \sim P(s, a)}\left[\left(\hat{\mathcal{T}}{Q}- \underbrace{\rho}_{= r+\beta \zeta}\right)^2\mid s, a\right]  \notag
    \\
    &= \mathbb{E}_{ s^\prime \sim P(s, a)}\left[\mathcal{L}_{TD}(Q)(s,a,s^\prime)\right]-\beta^2 \min_{\zeta\in \mathbb{R}}\mathbb{E}_{ s^\prime \sim P(s, a)}\left[\left(\hat{V}(s^\prime)- \zeta\right)^2\mid s, a\right]\label{eq:yesmin}
    \\
    &= \mathbb{E}_{ s^\prime \sim P(s, a)}\left[\mathcal{L}_{TD}(Q)(s,a, s^\prime)\right]-\beta^2 \mathbb{E}_{ s^\prime \sim P(s, a)}\left[\left(\hat{V}(s^\prime)- \mathbb{E}_{ s^\prime \sim P(s, a)}[\hat{V}(s^\prime)\mid s,a]\right)^2\mid s, a\right]\label{eq:nomin}
    \end{align}
where the equality of Equation \eqref{eq:nomin} comes from the fact that the $\zeta$ that maximize Equation \eqref{eq:yesmin} is $\zeta^* := \mathbb{E}_{ s^\prime \sim P(s, a)}[\hat{V}(s')\mid s,a]$, because
\begin{align}
r(s,a)+\beta \cdot \zeta^{*} (s,a) &:=  \rho^{*}(s,a)\notag
     \\
     &= \mathcal{T}Q(s,a) \notag
     \\
     &:=r(s,a)+\beta \cdot \mathbb{E}_{ s^\prime \sim P(s, a)}\left[\hat{V}(s^\prime)\mid s,a\right]\notag
\end{align}
For $Q^\ast$, $\mathcal{T}Q^\ast=Q^\ast$ holds. Therefore, we get
\begin{align}
r(s,a)+\beta \cdot \zeta^{*} (s,a) &:=  \rho^{*}(s,a)\notag
     \\
     &= \mathcal{T}Q^\ast(s,a) = Q^\ast(s,a) \notag
\end{align}
\end{proof}

\subsection{Proof of Lemma \ref{thm:MagnacThesmar}}
\label{sec:PfMagnac}
\begin{proof}
    Suppose that the system of equations (Equation \eqref{eq:HotzMillereqs})
\begin{equation}
\left\{
\begin{array}{l}
    \dfrac{\exp({Q}\left(s,a\right))}{\sum_{a^\prime\in \mathcal{A}} \exp({Q}\left(s,a^\prime\right))} = \pi^*(\;a
    \mid s) \; \; \; \forall s\in \mathcal{S}, a\in\mathcal{A}
    \\[1em]
    r(s, a_s)+\beta \cdot \mathbb{E}_{s^{\prime} \sim P(s, a_s)}\left[\log(\sum_{a^\prime\in\mathcal{A}}\exp Q(s^\prime, a^\prime)) \mid s, a_s\right]-Q(s, a_s)=0 \;\; \; \forall s\in \mathcal{S} 
\end{array}
\right.
\notag
\end{equation} 
is satisfied for $Q\in \mathcal{Q}$, where $\mathcal{Q}$ denote the space of all $Q$ functions. Then we have the following equivalent recharacterization of the second condition $\forall s\in \mathcal{S}$,
\begin{align}
     Q(s, a_s)&=r(s, a_s)+\beta \cdot \mathbb{E}_{s^{\prime} \sim P(s, a_s)}\left[\log(\sum_{a^\prime\in\mathcal{A}}\exp Q(s^\prime, a^\prime)) \mid s, a_s\right]\;\; \; \notag
     \\
     &=  r(s, a_s)+\beta \cdot \mathbb{E}_{s^\prime \sim P(s, a_s)}\left[Q(s^\prime, a^\prime) - \log \pi^*(a^\prime \mid s^\prime) \mid s, a_s\right] \;\; \forall a^\prime\in\mathcal{A} \notag
     \\
     &=  r(s, a_s)+\beta \cdot \mathbb{E}_{s^\prime \sim P(s, a_s)}\left[Q(s^\prime, a_{s^\prime}) - \log \pi^*(a_{s^\prime} \mid s^\prime) \mid s, a_s\right]
\end{align}

We will now show the existence and uniqueness of a solution using a standard fixed point argument on a Bellman operator. Let $\mathcal{F}$ be the space of functions $f: \mathcal{S} \rightarrow \mathbb{R}$ induced by elements of $\mathcal{Q}$, where each $Q \in \mathcal{Q}$ defines an element of $\mathcal{F}$ via

$$
f_Q(s):=Q\left(s, a_s\right)
$$

and define an operator $\mathcal{T}_f: \mathcal{F} \rightarrow$ $\mathcal{F}$ that acts on functions $f_Q$ :

$$
\left(\mathcal{T}_f f_Q\right)(s):=r\left(s, a_s\right)+\beta \sum_{s^{\prime}} P\left(s^{\prime} \mid s, a_s\right)\left[f_Q\left(s^{\prime}\right)-\log \pi^*\left(a_{s^{\prime}} \mid s^{\prime}\right)\right]
$$

Then for $Q_1, Q_2 \in\mathcal{Q}$, We have
\begin{align} & \left(\mathcal{T}_f f_{Q_1}\right)(s):=r\left(s, a_s\right)+\beta \sum_{s^{\prime}} P\left(s^{\prime} \mid s, a_s\right)\left[f_{Q_1}\left(s^{\prime}\right)-\log \pi^*\left(a_{s^{\prime}} \mid s^{\prime}\right)\right] \notag
\\ & \left(\mathcal{T}_f f_{Q_2}\right)(s):=r\left(s, a_s\right)+\beta \sum_{s^{\prime}} P\left(s^{\prime} \mid s, a_s\right)\left[f_{Q_2}\left(s^{\prime}\right)-\log \pi^*\left(a_{s^{\prime}} \mid s^{\prime}\right)\right]\notag
\end{align}

Subtracting the two, we get
\begin{align}
\left|\left(\mathcal{T}_f f_{Q_1}\right)(s)-\left(\mathcal{T}_f f_{Q_2}\right)(s)\right| 
&\leq \beta \sum_{s^{\prime}} P\left(s^{\prime} \mid s, a_s\right)\left|f_{Q_1}\left(s^{\prime}\right)-f_{Q_2}\left(s^{\prime}\right)\right| \notag
\\
&\leq \beta\left\|f_{Q_1}-f_{Q_2}\right\|_{\infty} \notag
\end{align}

Taking supremum norm over $s\in\mathcal{S}$, we get
$$\left\|\mathcal{T}_f f_{Q_1}-\mathcal{T}_f f_{Q_2}\right\|_{\infty} \leq \beta\left\|f_{Q_1}-f_{Q_2}\right\|_\infty$$

This implies that $\mathcal{T}_f$ is a contraction mapping under supremum norm, with $\beta\in (0,1)$. Since $\mathcal{Q}$ is a Banach space under sup norm (Lemma \ref{lem:completeMetric}), we can apply Banach fixed point theorem to show that there exists a unique $f_Q$ that satisfies $\mathcal{T}_f(f_Q) = f_Q$, and by definition of $f_Q$ there exists a unique $Q$ that satisfies $\mathcal{T}_f(f_Q) = f_Q$, i.e., 
$$r\left(s, a_s\right)+\beta \cdot \mathbb{E}_{s^{\prime} \sim P\left(s, a_s\right)}\left[\log \left(\sum_{a^{\prime} \in \mathcal{A}} \exp Q\left(s^{\prime}, a^{\prime}\right)\right) \mid s, a_s\right]-Q\left(s, a_s\right)=0 \quad \forall s \in \mathcal{S}$$

Since $Q^\ast$ satisfies the system of equations \eqref{eq:HotzMillereqs}, $Q^\ast$ is the only solution to the system of equations.

Also, since $Q^\ast = \mathcal{T}Q^\ast = r(s,a)+\beta \cdot \mathbb{E}_{s^{\prime} \sim P(s, a)}\bigl[\log(\sum_{a^\prime\in\mathcal{A}}\exp Q^\ast(s^\prime, a^\prime)) \mid s, a\bigr]$ holds, we can identify $r$ as
\begin{align}
    r(s,a) &= Q^\ast(s, a) - \beta \cdot \mathbb{E}_{s^{\prime} \sim P(s, a)}\bigl[\log(\sum_{a^\prime\in\mathcal{A}}\exp Q^\ast(s^\prime, a^\prime)) \mid s, a\bigr] \notag
\end{align}

\end{proof}
\begin{lem}\label{lem:completeMetric} Suppose that $\mathcal{Q}$ consists of bounded functions on $\mathcal{S} \times \mathcal{A}$. Then $\mathcal{Q}$ is a Banach space with the supremum norm as the induced norm.
\end{lem}
\begin{proof}
Suppose a sequence of functions $\left\{Q_n\right\}$ in $\mathcal{Q}$ is Cauchy in the supremum norm. We must show that $ Q_n\rightarrow Q^\ast$ as $n\rightarrow \infty$ for some $Q^\ast$ and $Q^\ast$ is also bounded. %because $\left|Q^*(s, a)\right|=\lim _{n \rightarrow \infty}\left|Q_n(s, a)\right| \leq M$. 
Note that $Q_n$ being Cauchy in sup norm implies that for every $(s, a)$, the sequence $\left\{Q_n(s, a)\right\}$ is Cauchy in $\mathbb{R}$. Since $\mathbb{R}$ is a complete space, every Cauchy sequence of real numbers has a limit; this allows us to define function $Q^\ast:\mathcal{S}\times \mathcal{A} \mapsto \mathbb{R}$ such that $Q^*(s, a)=\lim _{n \rightarrow \infty} Q_n(s, a)$. Then we can say that $Q_n(s, a) \rightarrow Q^*(s, a) $ for every $(s, a) \in \mathcal{S} \times \mathcal{A}$. Since each $Q_n$ is bounded, we take the limit and obtain:
$$
\sup _{s, a}\left|Q^*(s, a)\right|=\lim _{n \rightarrow \infty} \sup _{s, a}\left|Q_n(s, a)\right| \leq M
$$
which implies $Q^* \in \mathcal{Q}$.

\noindent Now what's left is to show that the supremum norm
$$
\|Q\|_{\infty}=\sup _{(s, a) \in \mathcal{S} \times \mathcal{A}}|Q(s, a)|
$$
induces the metric, i.e., 
$$
d\left(Q_1, Q_2\right):=\left\|Q_1-Q_2\right\|_{\infty}=\sup _{(s, a) \in \mathcal{S} \times \mathcal{A}}\left|Q_1(s, a)-Q_2(s, a)\right|
$$
The function $d$ satisfies the properties of a metric:

- Non-negativity: $d\left(Q_1, Q_2\right) \geq 0$ and $d\left(Q_1, Q_2\right)=0$ if and only if $Q_1=$ $Q_2$.

- Symmetry: $d\left(Q_1, Q_2\right)=d\left(Q_2, Q_1\right)$ by the absolute difference.

- Triangle inequality:
$$
d\left(Q_1, Q_3\right)=\sup _{s, a}\left|Q_1(s, a)-Q_3(s, a)\right| \leq \sup _{s, a}\left|Q_1(s, a)-Q_2(s, a)\right|+\sup _{s, a}\left|Q_2(s, a)-Q_3(s, a)\right|
$$
which shows $d\left(Q_1, Q_3\right) \leq d\left(Q_1, Q_2\right)+d\left(Q_2, Q_3\right)$.
\end{proof}

\subsection{Proof of Theorem 
\ref{thm:mainopt}}\label{sec:pfOfmainOpt}

Define $\hat{Q}$ as 
\begin{align}
    \hat{Q} &\in \underset{Q\in \mathcal{Q}}{\arg\min } \; \;\mathbb{E}_{(s, a)\sim \pi^*, \nu_0}  \left[-\log\left(\hat{p}_{Q}(a
\mid s)\right)\right] + \mathbb{E}_{(s, a)\sim \pi^*, \nu_0}\left[ \mathbf{1}_{a = a_s} \mathcal{L}_{BE}(Q)(s,a)\right] \tag{Equation \ref{eq:mainopt}}
\end{align}
From Lemma \ref{thm:MagnacThesmar}, it is sufficient to show that $\hat{Q}$
satisfies the equations \eqref{eq:HotzMillereqs} of Lemma \ref{thm:MagnacThesmar}, i.e., 
\begin{equation}
\left\{
\begin{array}{l}
    \dfrac{\exp({\hat{Q}}\left(s,a\right))}{\sum_{a^\prime\in \mathcal{A}} \exp({\hat{Q}}\left(s,a^\prime\right))} = \pi^*(a
    \mid s) \; \; \; \forall s\in \bar{\mathcal{S}}, a \in \mathcal{A}
    \\[1em]
    r(s, a_s)+\beta \cdot \mathbb{E}_{s^{\prime} \sim P(s, a_s)}\left[\log(\sum_{a^\prime\in\mathcal{A}}\exp \hat{Q}(s^\prime, a^\prime)) \mid s, a_s\right]-\hat{Q}(s, a_s)=0 \;\;\; \forall s\in \bar{\mathcal{S}}
    
\end{array}\tag{Equation \ref{eq:HotzMillereqs}}
\right. 
\end{equation}
where $\bar{\mathcal{S}}$ (the reachable states from $\nu_0$, $\pi^\ast$) was defined as:
$$
\bar{\mathcal{S}}=\left\{s \in \mathcal{S} \mid \operatorname{Pr}\left(s_t=s \mid s_0 \sim \nu_0, \pi^*\right)>0 \text { for some } t \geq 0\right\} 
$$
Now note that:
  \begin{align}
     &\left\{Q \in \mathcal{Q} \mid\hat{p}_{Q}(\;\cdot
    \mid s) = \pi^*(\;\cdot
    \mid s)\quad  \forall s\in\bar{\mathcal{S}}\quad\text{a.e.}\right\} \notag
    \\
    &=\underset{Q\in \mathcal{Q}}{\arg\max } \; \;\mathbb{E}_{(s, a)\sim \pi^*, \nu_0}  \left[\log\left(\hat{p}_{Q}(\;\cdot
    \mid s)\right)\right] \tag{$\because$ Lemma \ref{lem:minMLE}}
    \\
    &=\underset{Q\in \mathcal{Q}}{\arg\min } \; \;\mathbb{E}_{(s, a)\sim \pi^*, \nu_0}  \left[-\log\left(\hat{p}_{Q}(\;\cdot
    \mid s)\right)\right] \notag
    \end{align}
and 
    \begin{align}
     &\left\{Q \in \mathcal{Q} \mid\mathcal{L}_{BE}(Q)(s,a_s) = 0\quad  \forall s\in\bar{\mathcal{S}}\right\} \notag
    \\
    &=\underset{Q\in \mathcal{Q}}{\arg\min } \; \;\mathbb{E}_{(s, a)\sim \pi^*, \nu_0}  \left[\mathbf{1}_{a = a_s} \mathcal{L}_{BE}(Q)(s,a)\right] \notag
    \end{align}
Therefore what we want to prove, Equations \eqref{eq:HotzMillereqs}, becomes the following Equation \eqref{eq:modifiedHotz}:

\begin{equation}
\left\{
\begin{array}{l}
    \hat{Q} \in \underset{Q\in \mathcal{Q}}{\arg\min } \; \;\mathbb{E}_{(s, a)\sim \pi^*, \nu_0}  \left[-\log\left(\hat{p}_{Q}(\;\cdot
    \mid s)\right)\right] 
    \\[1em]
     \hat{Q} \in \underset{Q\in \mathcal{Q}}{\arg\min } \; \;\mathbb{E}_{(s, a)\sim \pi^*, \nu_0}  \left[\mathbf{1}_{a = a_s} \mathcal{L}_{BE}(Q)(s,a)\right]
    
\end{array} \tag{equation \ref{eq:modifiedHotz}}
\right. 
\end{equation}
where its solution set is nonempty by Lemma \ref{thm:MagnacThesmar}, i.e., 
$$ \underset{Q\in \mathcal{Q}}{\arg\min } \; \;\mathbb{E}_{(s, a)\sim \pi^*, \nu_0} \left[-\log\left(\hat{p}_{Q}(a
\mid s)\right)\right] \;\; \cap \;\; \underset{Q\in \mathcal{Q}}{\arg\min } \; \; \mathbb{E}_{(s, a)\sim \pi^*, \nu_0} \left[\mathbf{1}_{a = a_s}\mathcal{L}_{BE}(\hat{Q})(s,a)\right] \;\; \neq \;\; \emptyset$$ 

Under this non-emptiness, according to Lemma \ref{lem:sharingsol}, $\hat{Q}$ satisfies Equation \eqref{eq:modifiedHotz}. This implies that $\hat{Q}(s,a) = Q^\ast(s,a)$ for $s\in\bar{\mathcal{S}}$ and $a\in\mathcal{A}$, as the solution to set of Equations \eqref{eq:HotzMillereqs} is $Q^*$. This implies that 
\begin{align}
    r(s, a)=\hat{Q}(s, a)-\beta \cdot \mathbb{E}_{s^{\prime} \sim P(s, a)}\left[\log \left(\sum_{a^{\prime} \in \mathcal{A}} \exp \hat{Q}\left(s^{\prime}, a^{\prime}\right)\right) \mid s, a\right] \notag
\end{align}
for $s\in\bar{\mathcal{S}}$ and $a\in\mathcal{A}$.
\QED

\begin{lem}\label{lem:minMLE}
    \begin{align}
           \underset{Q\in \mathcal{Q}}{\arg\min } &\; \;\mathbb{E}_{(s, a)\sim \pi^*, \nu_0}  \left[-\log\left(\hat{p}_{Q}(\;\cdot
    \mid s)\right)\right] \notag
    \\
     &=\left\{Q \in \mathcal{Q} \mid\hat{p}_{Q}(\;\cdot
    \mid s) = \pi^*(\;\cdot
    \mid s)\quad  \forall s\in\bar{\mathcal{S}}\quad\text{a.e.}\right\}\notag
    \\
     &=\left\{Q \in \mathcal{Q} \mid Q(s,a_1)-Q(s,a_2)= Q^*(s,a_1)-Q^*(s,a_2) \quad \forall a_1, a_2\in\mathcal{A}, s\in\bar{\mathcal{S}}\right\} \notag
    \end{align}
\end{lem}

\begin{proof}[Proof of Lemma \ref{lem:minMLE}]
    \begin{align}
    \mathbb{E}_{(s, a)\sim \pi^*, \nu_0}  \left[\log\left(\hat{p}_{Q}(
    a
    \mid s)\right)\right] & = 
\mathbb{E}_{(s,a)\sim\pi^*, \nu_0} [\log \hat{p}_{Q}(a|s) - \ln \pi^*(a|s) + \ln \pi^*(a|s)]\notag \\
    &=-\mathbb{E}_{(s,a)\sim\pi^*, \nu_0} \left[\ln \frac{\pi^*(a|s)}{\hat{p}_{Q}(a|s)} \right] + \mathbb{E}_{(s,a)\sim\pi^*, \nu_0} [\ln \pi^*(a|s)]\notag \\
    &= -\mathbb{E}_{s\sim\pi^*, \nu_0} \left[D_{KL}(\pi^*(\cdot\mid s) \| \hat{p}_{Q}(\cdot\mid s))\right] + \mathbb{E}_{(s,a)\sim\pi^*, \nu_0} [\ln \pi^*(a|s)]\notag \notag
\end{align}
Therefore,
\begin{align}
    \underset{Q\in\mathcal{Q}}{\arg\max}&\; \mathbb{E}_{(s, a)\sim \pi^*, \nu_0}  \left[\log\left(\hat{p}_{Q}(\;\cdot
    \mid s)\right)\right] =\underset{Q\in\mathcal{Q}}{\arg\min}\;\mathbb{E}_{s\sim\pi^*, \nu_0} \left[D_{KL}(\pi^*(\cdot\mid s) \| \hat{p}_{Q}(\cdot\mid s))\right]\notag
    \\
    &=\{Q\in\mathcal{Q}\mid D_{KL}(\pi^*(\cdot\mid s) \| \hat{p}_{Q}(\cdot\mid s))=0 \text{ for all }s\in\bar{\mathcal{S}}\} \tag{$\because \; Q^*\in \mathcal{Q}$ and $D_{KL}(\pi^* \| \pi^*)=0$}\notag
    \\
    &=\{Q\in\mathcal{Q}\mid \hat{p}_Q(\cdot \mid s)=\pi^*(\cdot \mid s)\; \; \text{a.e.} \text{ for all }s\in\bar{\mathcal{S}}\}\notag
    \\
    &= \{Q\in\mathcal{Q}\mid \frac{\hat{p}_Q\left(a_1 \mid s\right)}{\hat{p}_Q\left(a_2 \mid s\right)}=\frac{\pi^*\left(a_1 \mid s\right)}{\pi^*\left(a_2 \mid s\right)} \quad \forall a_1, a_2 \in \mathcal{A}, s\in\bar{\mathcal{S}}\}\notag   
    \\
    &=\left\{Q \in \mathcal{Q} \mid \exp (Q(s,a_1)-Q(s,a_2))=\exp \left(Q^*(s,a_1)-Q^*(s,a_2)\right) \quad \forall a_1, a_2\in\mathcal{A}, s\in\bar{\mathcal{S}} \right\}\notag
    \\
    &=\left\{Q \in \mathcal{Q} \mid Q(s,a_1)-Q(s,a_2)= Q^*(s,a_1)-Q^*(s,a_2) \quad \forall a_1, a_2\in\mathcal{A}, s\in\bar{\mathcal{S}}\right\} \notag
\end{align}
\end{proof}

\begin{lem}\label{lem:sharingsol}
    Let $f_1: \mathcal{X} \rightarrow \mathbb{R}$ and $f_2: \mathcal{X} \rightarrow \mathbb{R}$ be two functions defined on a common domain $\mathcal{X}$. Suppose the sets of minimizers of $f_1$ and $f_2$ intersect, i.e.,
$$
\arg \min f_1 \cap \arg \min f_2 \neq \emptyset
$$
Then, any minimizer of the sum $f_1+f_2$ is also a minimizer of both $f_1$ and $f_2$ individually. That is, if
$$
x^* \in \arg \min \left(f_1+f_2\right)
$$
then
$$
x^* \in \arg \min f_1 \; \cap \; \arg \min f_2
$$
\end{lem}
\begin{proof}
    Since \( \arg\min f_1 \cap \arg\min f_2 \neq \emptyset \), let \( x^\dagger \) be a common minimizer such that $$
x^\dagger \in \arg\min f_1 \cap \arg\min f_2$$
This implies that  
\begin{align}
    f_1(x^\dagger) &= \min_{x \in \mathcal{X}} f_1(x) =: m_1, \notag
    \\
    f_2(x^\dagger) &= \min_{x \in \mathcal{X}} f_2(x) =: m_2. \notag
\end{align}

Now, let \( x^* \) be any minimizer of \( f_1 + f_2 \), so  
\begin{align}
    x^* \in \arg\min (f_1 + f_2) &\iff f_1(x^*) + f_2(x^*) \leq f_1(x) + f_2(x), \quad \forall x \in \mathcal{X}. \notag
\end{align}
Evaluating this at \( x^\dagger \), we obtain  
\begin{align}
    f_1(x^*) + f_2(x^*) &\leq f_1(x^\dagger) + f_2(x^\dagger) \notag
    \\
    &= m_1 + m_2. \notag
\end{align}

Now, suppose for contradiction that $x^* \notin \arg\min f_1$, 
meaning  
\begin{align}
    f_1(x^*) &> m_1 \notag
\end{align}
But then
\begin{align}
    f_2(x^*) & \le m_1 + m_2 - f_1(x^*) \notag
    \\
    &< m_1 + m_2 - m_1 = m_2 \notag
\end{align}

This contradicts the fact that \( m_2 = \min f_2 \), so \( x^* \) must satisfy  
\begin{align}
    f_1(x^*) &= m_1 \notag
\end{align}

By symmetry, assuming $x^* \notin \arg\min f_2$ leads to the same contradiction, forcing  
\begin{align}
    f_2(x^*) &= m_2 \notag
\end{align}

Thus, we conclude  
\begin{align}
    x^* \in \arg\min f_1 \cap \arg\min f_2 \notag
\end{align}
\end{proof}

\subsection{Grouped empirical notation for the PL proof}\label{app:empirical_grouping}
For the PL proof, we group the empirical averages by the distinct sampled states and state-action pairs that appear in the observed finite dataset \(\mathcal D\). The underlying state space may be continuous; the matrices below are finite-dimensional only because an empirical risk is a finite sample average. This grouping is not an additional statistical assumption; it is the standard rewriting of a finite empirical average by distinct sampled states and their empirical weights.

For notational convenience, write the finite observations entering the NLL, Bellman-error, and auxiliary \(\zeta\)-regression terms as
\[
\mathcal D_{\rm NLL}=\{(s_i^{\rm NLL},a_i^{\rm NLL})\}_{i=1}^{N_{\rm NLL}},
\quad
\mathcal D_{\rm BE}=\{(s_j^{\rm BE},a_j^{\rm BE},s_j^{\prime\rm BE})\}_{j=1}^{N_{\rm BE}},
\quad
\mathcal D_{\zeta}=\{(\bar s_k,\bar a_k,\bar s_k')\}_{k=1}^{N_\zeta}.
\]
These collections are not assumed to be independent or disjoint; they are objective-specific projections or re-indexings of the same observed dataset \(\mathcal D\). For the Bellman-error term, \(a_j^{\rm BE}\) denotes the anchor action \(a_{s_j^{\rm BE}}\).

Let $\mathcal X_{\rm NLL}$ be the set of distinct states in the NLL sample. For $x\in\mathcal X_{\rm NLL}$, define
\[
    n_x:=\#\{i:s_i^{\rm NLL}=x\},
    \qquad
    \omega_x^{\rm NLL}:={n_x\over N_{\rm NLL}},
    \qquad
    \widehat\pi_x(a):={\#\{i:s_i^{\rm NLL}=x,\,a_i^{\rm NLL}=a\}\over n_x}.
\]
Then the empirical NLL can be written exactly as
\[
    \widehat L_{\rm NLL}(\theta)
    :=
    \sum_{x\in\mathcal X_{\rm NLL}}
    \omega_x^{\rm NLL}
    \sum_{a\in\mathcal A}
    \widehat\pi_x(a)
    \left[-\log \hat p_{Q_\theta}(a\mid x)\right],
\]
where
\[
    \hat p_{Q_\theta}(a\mid s)
    =
    {\exp(Q_\theta(s,a))\over \sum_{b\in\mathcal A}\exp(Q_\theta(s,b))}.
\]
This is identical to the sample average
\[
    {1\over N_{\rm NLL}}
    \sum_{i=1}^{N_{\rm NLL}}
    \left[-Q_\theta(s_i^{\rm NLL},a_i^{\rm NLL})+V_{Q_\theta}(s_i^{\rm NLL})\right].
\]

Let $\mathcal X_{\rm BE}=\{x_1,\ldots,x_B\}$ be the set of distinct current anchor states in the Bellman-error sample. For each $x_j\in\mathcal X_{\rm BE}$, define
\[
    m_j:=\#\{\ell:s_\ell^{\rm BE}=x_j\},
    \qquad
    \omega_j^{\rm BE}:={m_j\over N_{\rm BE}},
\]
\[
    \widehat P_j(y):={\#\{\ell:s_\ell^{\rm BE}=x_j,\,s_\ell^{\prime\rm BE}=y\}\over m_j},
    \qquad
    \widehat r_j:={1\over m_j}\sum_{\ell:s_\ell^{\rm BE}=x_j}r(s_\ell^{\rm BE},a_\ell^{\rm BE}).
\]
If the reward is deterministic in $(s,a)$, then $\widehat r_j=r(x_j,a_{x_j})$. Define the grouped empirical Bellman residual
\[
    \Psi_j(\theta)
    :=
    \widehat r_j
    +\beta\sum_y\widehat P_j(y)V_{Q_\theta}(y)
    -Q_\theta(x_j,a_{x_j}),
    \qquad
    j=1,\ldots,B,
\]
and the empirical Bellman-error loss
\[
    \widehat L_{\rm BE}(\theta)
    :=
    \sum_{j=1}^B\omega_j^{\rm BE}\Psi_j(\theta)^2.
\]
When the current anchor states are distinct in the raw sample, this reduces to the ungrouped sample average. In the general case it is the squared Bellman residual of the empirical transition kernel induced by the sample.

For the auxiliary regression term, let $Z_\zeta=\{\bar z_1,\ldots,\bar z_G\}$ be the distinct state-action pairs in $\{(\bar s_k,\bar a_k)\}_{k=1}^{N_\zeta}$. For $g\in[G]$, define
\[
    n_g^\zeta:=\#\{k:(\bar s_k,\bar a_k)=\bar z_g\},
    \qquad
    \omega_g^\zeta:={n_g^\zeta\over N_\zeta},
\]
\[
    \bar y_g(\theta_2)
    :=
    {1\over n_g^\zeta}
    \sum_{k:(\bar s_k,\bar a_k)=\bar z_g}
    V_{Q_{\theta_2}}(\bar s_k').
\]
Up to an additive constant independent of $\phi$, the auxiliary regression loss is
\[
    \widehat F_\zeta(\phi;\theta_2)
    :=
    \beta^2
    \sum_{g=1}^G
    \omega_g^\zeta
    \left(\zeta_\phi(\bar z_g)-\bar y_g(\theta_2)\right)^2.
\]
The inner maximization in Equation \eqref{eq:paraEmp} is the maximization of $-\widehat F_\zeta$.

Let $\mathcal X_Q$ be the finite set of all states at which $Q_\theta$ is evaluated:
\[
    \mathcal X_Q
    :=
    \mathcal X_{\rm NLL}
    \cup
    \mathcal X_{\rm BE}
    \cup
    \bigcup_{j=1}^B\operatorname{supp}(\widehat P_j).
\]
The empirical $Q$-evaluation set is the full action-block evaluation set
\[
    Z_Q:=\mathcal X_Q\times\mathcal A
    =\{z_1,\ldots,z_M\}.
\]
Define
\[
    \mathbf Q_\theta(Z_Q):=(Q_\theta(z_1),\ldots,Q_\theta(z_M))^\top\in\mathbb R^M,
    \qquad
    J_Q(\theta;Z_Q):=D_\theta\mathbf Q_\theta(Z_Q)\in\mathbb R^{M\times d_\theta}.
\]
Similarly,
\[
    \boldsymbol\zeta_\phi(Z_\zeta):=(\zeta_\phi(\bar z_1),\ldots,\zeta_\phi(\bar z_G))^\top,
    \qquad
    J_\zeta(\phi;Z_\zeta):=D_\phi\boldsymbol\zeta_\phi(Z_\zeta).
\]
Writing $q_\theta:=\mathbf Q_\theta(Z_Q)$ and
$\boldsymbol\Psi(q_\theta):=(\Psi_1(\theta),\ldots,\Psi_B(\theta))^\top$, define the Bellman residual output Jacobian
\[
    A_\theta:=D_q\boldsymbol\Psi(q)\big|_{q=q_\theta}
    \in\mathbb R^{B\times M}.
\]
For later use, define
\[
    \omega_{\rm NLL,min}:=\min_{x\in\mathcal X_{\rm NLL}}\omega_x^{\rm NLL},
    \quad
    \omega_{\rm BE,min}:=\min_{j\in[B]}\omega_j^{\rm BE},
    \quad
    \omega_{\rm BE,max}:=\max_{j\in[B]}\omega_j^{\rm BE},
    \quad
    \omega_{\zeta,min}:=\min_{g\in[G]}\omega_g^\zeta.
\]

\subsection{Parametrization details for Assumption~\ref{ass:nonSingularJac}}\label{app:empirical_param_proofs}
\subsubsection{Proof of Lemma \ref{lem:linPolyNonsingular}}
\begin{proof}
Define
\[
    \Phi_Q(Z_Q)
    :=
    \begin{bmatrix}
        \varphi(z_1)^\top\\
        \vdots\\
        \varphi(z_M)^\top
    \end{bmatrix}
    \in\mathbb R^{M\times d_\theta},
    \qquad
    X_\zeta(Z_\zeta)
    :=
    \begin{bmatrix}
        \chi(\bar z_1)^\top\\
        \vdots\\
        \chi(\bar z_G)^\top
    \end{bmatrix}.
\]
For the linear $Q$ class,
\[
    \mathbf Q_\theta(Z_Q)=\Phi_Q(Z_Q)\theta,
    \qquad
    J_Q(\theta;Z_Q)=D_\theta\mathbf Q_\theta(Z_Q)=\Phi_Q(Z_Q).
\]
Hence
\[
    J_Q(\theta;Z_Q)J_Q(\theta;Z_Q)^\top
    =
    \Phi_Q(Z_Q)\Phi_Q(Z_Q)^\top.
\]
If this empirical Gram matrix is positive definite with minimum eigenvalue at least $\mu_Q>0$, then Assumption~\ref{ass:nonSingularJac}(ii) holds.

The smoothness and derivative bounds in Assumption~\ref{ass:nonSingularJac}(i) are immediate. Indeed,
\[
    \nabla_\theta Q_\theta(z)=\varphi(z),
    \qquad
    \nabla_\theta^2Q_\theta(z)=0.
\]
Since $Z_Q$ is finite and the feature evaluations are finite,
\[
    \sup_{z\in Z_Q}\|\nabla_\theta Q_\theta(z)\|_2
    =
    \max_{z\in Z_Q}\|\varphi(z)\|_2<\infty,
    \qquad
    \sup_{z\in Z_Q}\|\nabla_\theta^2Q_\theta(z)\|_{\rm op}=0.
\]

The auxiliary calculation is identical. Since
\[
    \boldsymbol\zeta_\phi(Z_\zeta)=X_\zeta(Z_\zeta)\phi,
    \qquad
    J_\zeta(\phi;Z_\zeta)=X_\zeta(Z_\zeta),
\]
the lower bound
\[
    X_\zeta(Z_\zeta)X_\zeta(Z_\zeta)^\top\succeq \mu_\zeta I_G
\]
implies Assumption~\ref{ass:nonSingularJac}(iii). Also,
\[
    \nabla_\phi\zeta_\phi(\bar z)=\chi(\bar z),
    \qquad
    \nabla_\phi^2\zeta_\phi(\bar z)=0,
\]
so the auxiliary derivative bounds in Assumption~\ref{ass:nonSingularJac}(i) follow because $Z_\zeta$ is finite.
\end{proof}

\begin{rem}[A probabilistic sufficient condition for linear features]
Conditional on the realized empirical evaluation set $Z_Q(\mathcal D)$, if the feature rows $\varphi(z_1)^\top,\ldots,\varphi(z_M)^\top$ are independent centered isotropic sub-Gaussian vectors in $\mathbb R^{d_\theta}$ and $d_\theta\ge C M$, then standard smallest-singular-value bounds for random rectangular matrices imply
\[
    \Phi_Q(Z_Q)\Phi_Q(Z_Q)^\top\succeq c d_\theta I_M
\]
with probability at least $1-e^{-c'd_\theta}$ for constants $c,c'>0$ depending only on the sub-Gaussian norm; see, for example, \citet{rudelson2009smallest}. Thus the empirical-output conditioning required in Assumption~\ref{ass:nonSingularJac} is the standard full-row-rank condition for over-parameterized linear systems.
\end{rem}

\subsubsection{Proof of Lemma \ref{lem:NNenjoysPL}}
\begin{proof}
We first verify Assumption~\ref{ass:nonSingularJac}(i). Since the hidden activations are $C^2$ and the layers are affine, each empirical output map $\theta\mapsto Q_\theta(z)$ is $C^2$ on $B_Q$, and each map $\phi\mapsto\zeta_\phi(\bar z)$ is $C^2$ on $B_\zeta$. The balls $B_Q$ and $B_\zeta$ are compact, and the empirical evaluation sets $Z_Q$ and $Z_\zeta$ are finite. Therefore the continuous functions
\[
    (\theta,z)\mapsto\|\nabla_\theta Q_\theta(z)\|_2,
    \qquad
    (\theta,z)\mapsto\|\nabla_\theta^2Q_\theta(z)\|_{\rm op}
\]
attain finite maxima on $B_Q\times Z_Q$, and the analogous maxima for $\zeta_\phi$ are finite on $B_\zeta\times Z_\zeta$. This proves the smoothness and bounded-derivative part of the assumption.

It remains to verify the empirical-output Jacobian conditions. The ERM objective depends on the finite vector $\mathbf Q_\theta(Z_Q(\mathcal D))$ induced by the observed continuous-state data. Thus the relevant tangent kernel is the finite matrix
\[
    K_Q(\theta;Z_Q(\mathcal D))
    =
    J_Q(\theta;Z_Q(\mathcal D))J_Q(\theta;Z_Q(\mathcal D))^\top,
\]
which is exactly the finite-output tangent kernel $D\mathcal F(\theta)D\mathcal F(\theta)^\top$ used in the PL$^*$ theory of \citet{liu2022loss}. The initial positive-definiteness step is supplied by standard finite-data NTK results: \citet{jacot2018neural} establish positive definiteness of limiting smooth NTKs under nondegenerate input conditions, while \citet{du2019gradient,allenzhu2019convergence} use analogous finite-data input conditions for over-parameterized ReLU networks. In a continuous-state setting, the hypotheses are satisfied on the corresponding full-probability event whenever the sampling distribution is not concentrated on the degeneracy sets excluded by those results, and accidental repeated evaluations are removed by the grouping in Appendix~\ref{app:empirical_grouping}. Hence the limiting empirical kernel satisfies
\[
    \lambda_\star:=\lambda_{\min}(K_\infty(Z_Q(\mathcal D)))>0.
\]

We first consider a network $\widetilde Q_\theta$ with a linear output layer. Define
\[
    \widetilde{\mathbf Q}_\theta(Z)
    :=
    (\widetilde Q_\theta(z_1),\ldots,\widetilde Q_\theta(z_M))^\top,
    \qquad
    \widetilde J(\theta;Z):=D_\theta\widetilde{\mathbf Q}_\theta(Z),
\]
and
\[
    \widetilde K(\theta;Z):=\widetilde J(\theta;Z)\widetilde J(\theta;Z)^\top.
\]
Write $Z=Z_Q(\mathcal D)$ for brevity. Conditional on the realized empirical evaluation set $Z$, finite-width NTK concentration at initialization gives, after increasing the width $m$ if necessary,
\[
    \|\widetilde K(\theta_0;Z)-K_\infty(Z)\|_2\le {\lambda_\star\over4}
\]
with probability at least $1-\delta/3$. This is the finite-width concentration step that connects the random finite network to the limiting NTK matrix; combined with the positive-definiteness results for nondegenerate finite inputs cited above, it gives a strictly positive initial empirical NTK eigenvalue. Weyl's inequality then yields
\[
    \lambda_{\min}(\widetilde K(\theta_0;Z))
    \ge
    \lambda_{\min}(K_\infty(Z))-
    \|\widetilde K(\theta_0;Z)-K_\infty(Z)\|_2
    \ge
    {3\lambda_\star\over4}.
\]
The same concentration bound gives a finite initialization upper bound
\[
    \|\widetilde J(\theta_0;Z)\|_2^2
    =
    \lambda_{\max}(\widetilde K(\theta_0;Z))
    \le
    \lambda_{\max}(K_\infty(Z))+{\lambda_\star\over4}
    =:C_0^2
\]
with probability at least $1-\delta/3$, after increasing $m$ if necessary.

It remains to control the kernel on the ball $B(\theta_0,R)$. The transition-to-linearity/Hessian-control theorem of \citet{liu2022loss} gives, for smooth wide networks with a linear output layer and every fixed finite radius $R$, a high-probability bound
\[
    \sup_{\theta\in B(\theta_0,R)}
    \max_{i\in[M]}
    \|\nabla_\theta^2\widetilde Q_\theta(z_i)\|_{\rm op}
    \le
    C_R(m),
    \qquad
    C_R(m)\to0
\]
as $m\to\infty$; in the notation of \citet{liu2022loss}, $C_R(m)=\widetilde O(R^{3L}/\sqrt m)$. For any $\theta\in B(\theta_0,R)$ and any $i\in[M]$, the fundamental theorem of calculus gives
\[
\begin{aligned}
    \nabla_\theta\widetilde Q_\theta(z_i)-\nabla_\theta\widetilde Q_{\theta_0}(z_i)
    &=
    \int_0^1
    \nabla_\theta^2\widetilde Q_{\theta_0+t(\theta-\theta_0)}(z_i)(\theta-\theta_0)\,dt,
\end{aligned}
\]
and therefore
\[
    \|\nabla_\theta\widetilde Q_\theta(z_i)-\nabla_\theta\widetilde Q_{\theta_0}(z_i)\|_2
    \le C_R(m)R.
\]
Consequently,
\[
    \|\widetilde J(\theta;Z)-\widetilde J(\theta_0;Z)\|_F
    \le
    \sqrt M C_R(m)R.
\]
Thus
\[
    \sup_{\theta\in B(\theta_0,R)}\|\widetilde J(\theta;Z)\|_2
    \le
    C_0+\sqrt M C_R(m)R
    =:C_J(m).
\]
Choose $m$ large enough that $C_J(m)\le 2C_0$ and
\[
    2C_J(m)\sqrt M C_R(m)R\le {\lambda_\star\over4}.
\]
Then, for every $\theta\in B(\theta_0,R)$,
\[
\begin{aligned}
    \|\widetilde K(\theta;Z)-\widetilde K(\theta_0;Z)\|_2
    &\le
    \|\widetilde J(\theta)(\widetilde J(\theta)-\widetilde J(\theta_0))^\top\|_2
    +
    \|(\widetilde J(\theta)-\widetilde J(\theta_0))\widetilde J(\theta_0)^\top\|_2\\
    &\le
    (\|\widetilde J(\theta;Z)\|_2+\|\widetilde J(\theta_0;Z)\|_2)
    \|\widetilde J(\theta;Z)-\widetilde J(\theta_0;Z)\|_2\\
    &\le
    2C_J(m)\sqrt M C_R(m)R
    \le
    {\lambda_\star\over4}.
\end{aligned}
\]
A second application of Weyl's inequality gives, for every $\theta\in B(\theta_0,R)$,
\[
    \lambda_{\min}(\widetilde K(\theta;Z))
    \ge
    {3\lambda_\star\over4}-{\lambda_\star\over4}
    =
    {\lambda_\star\over2}.
\]
Thus the empirical-output Jacobian condition holds with $\mu_Q=\lambda_\star/2$.

If the network is a coordinate-wise smooth output transformation of a linear-output network,
\[
    Q_\theta(z)=\sigma_{\rm out}(\widetilde Q_\theta(z)),
\]
and if
\[
    \rho:=\inf_{\theta\in B(\theta_0,R)}\min_{i\in[M]}
    |\sigma_{\rm out}'(\widetilde Q_\theta(z_i))|>0,
\]
then
\[
    J_Q(\theta;Z)=D_\sigma(\theta)\widetilde J(\theta;Z),
\]
where $D_\sigma(\theta)$ is diagonal with entries $\sigma_{\rm out}'(\widetilde Q_\theta(z_i))$. Hence
\[
    K_Q(\theta;Z)=D_\sigma(\theta)\widetilde K(\theta;Z)D_\sigma(\theta),
\]
and for every $v\in\mathbb R^M$,
\[
    v^\top K_Q(\theta;Z)v
    =
    (D_\sigma(\theta)v)^\top\widetilde K(\theta;Z)(D_\sigma(\theta)v)
    \ge
    {\lambda_\star\over2}\rho^2\|v\|_2^2.
\]
Thus the lower bound is preserved with $\mu_Q=\rho^2\lambda_\star/2$.

The proof for the auxiliary network $\zeta_\phi$ on the empirical evaluation set $Z_\zeta(\mathcal D)$ is identical.
\end{proof}

\subsection{Technical output-space lemmas}\label{app:technical_output_lemmas}

\begin{lem}[Empirical-output PL transfer]
\label{lem:finite_output_pl_transfer}
Let $q_\theta=\mathbf Q_\theta(Z_Q)$ and $J_Q(\theta;Z_Q)=D_\theta q_\theta$. Suppose
\[
    J_Q(\theta;Z_Q)J_Q(\theta;Z_Q)^\top\succeq \mu_Q I_M,
    \qquad
    \forall\theta\in B_Q.
\]
Let $\varphi:\mathbb R^M\to\mathbb R$ be differentiable and suppose that $\varphi(q_\theta)$ satisfies a PL inequality on $\{q_\theta:\theta\in B_Q\}$ with constant $\alpha>0$. Then $L(\theta):=\varphi(q_\theta)$ satisfies a PL inequality on $B_Q$ with constant $\alpha\mu_Q$.
\end{lem}

\subsection{Proof of Lemma \ref{lem:finite_output_pl_transfer}}
\begin{proof}
By the chain rule,
\[
    \nabla_\theta L(\theta)
    =
    J_Q(\theta;Z_Q)^\top\nabla_q\varphi(q_\theta).
\]
Therefore
\[
\begin{aligned}
    \|\nabla_\theta L(\theta)\|_2^2
    &=
    \|J_Q(\theta;Z_Q)^\top\nabla_q\varphi(q_\theta)\|_2^2\\
    &=
    \nabla_q\varphi(q_\theta)^\top
    J_Q(\theta;Z_Q)J_Q(\theta;Z_Q)^\top
    \nabla_q\varphi(q_\theta)\\
    &\ge
    \mu_Q\|\nabla_q\varphi(q_\theta)\|_2^2\\
    &\ge
    2\alpha\mu_Q\bigl(\varphi(q_\theta)-\varphi_B^*\bigr).
\end{aligned}
\]
Dividing by two gives the result, because $L(\theta)=\varphi(q_\theta)$ and $L_B^*=\varphi_B^*$ by definition.
\end{proof}

\begin{lem}[Bellman constant-shift coercivity]
\label{lem:bellman_shift_coercivity}
Let $\mathcal U\subset\mathbb R^M$ be the product of within-state zero-sum action-logit subspaces over the action blocks in $Z_Q=\mathcal X_Q\times\mathcal A$, and let $P_\perp$ denote the orthogonal projection onto $\mathcal U^\perp$. Then the empirical Bellman residual output Jacobian satisfies
\[
    \|P_\perp A_\theta^\top u\|_2
    \ge
    \gamma_\perp\|u\|_2,
    \qquad
    \forall u\in\mathbb R^B,
    \quad \forall\theta\in B_Q,
\]
where one may take
\[
    \gamma_\perp:={1-\beta\over\sqrt{|\mathcal A|\,|\mathcal X_Q|}}>0.
\]
Consequently,
\[
    A_\theta A_\theta^\top\succeq \gamma_\perp^2 I_B,
    \qquad
    \forall\theta\in B_Q.
\]
\end{lem}

\subsection{Proof of the Bellman constant-shift coercivity lemma}
\begin{proof}
For a state $x\in\mathcal X_Q$, let $\mathbf 1_x\in\mathbb R^{|\mathcal A|}$ denote the all-ones vector on the action block at $x$. The orthogonal complement $\mathcal U^\perp$ consists exactly of vectors that are constant over each action block. The orthogonal projection of an action-coordinate vector $e_{(x,a)}$ onto $\mathcal U^\perp$ is
\[
    P_\perp e_{(x,a)}={1\over |\mathcal A|}\mathbf 1_x,
    \qquad
    \left\|{1\over |\mathcal A|}\mathbf 1_x\right\|_2={1\over\sqrt{|\mathcal A|}}.
\]
Moreover, for every $x$ and every probability vector $p(\cdot\mid x)$,
\[
    P_\perp\sum_{a\in\mathcal A}p(a\mid x)e_{(x,a)}
    ={1\over |\mathcal A|}\mathbf 1_x.
\]
Let $u\in\mathbb R^B$. Since
\[
    \nabla_q\Psi_j(q_\theta)
    =
    \beta\sum_y\widehat P_j(y)
    \sum_{a\in\mathcal A}\hat p_{Q_\theta}(a\mid y)e_{(y,a)}
    -e_{(x_j,a_{x_j})},
\]
we have
\[
    P_\perp A_\theta^\top u
    ={1\over |\mathcal A|}
    \sum_{x\in\mathcal X_Q}
    \left(\beta(Tu)_x-(Eu)_x\right)\mathbf 1_x,
\]
where the linear maps $E,T:\mathbb R^B\to\mathbb R^{|\mathcal X_Q|}$ are defined by
\[
    (Eu)_x:=\sum_{j:x_j=x}u_j,
    \qquad
    (Tu)_x:=\sum_{j=1}^B\widehat P_j(x)u_j.
\]
Because the states $x_1,\ldots,x_B$ are distinct after grouping, each column of $E$ has exactly one nonzero entry and no two columns share that entry. Hence
\[
    \|Eu\|_1=\|u\|_1.
\]
Because each $\widehat P_j$ is a probability distribution, $T$ is column-stochastic with nonnegative entries, and therefore
\[
    \|Tu\|_1\le \|u\|_1.
\]
Thus
\[
    \|Eu-\beta Tu\|_1
    \ge
    \|Eu\|_1-\beta\|Tu\|_1
    \ge
    (1-\beta)\|u\|_1
    \ge
    (1-\beta)\|u\|_2.
\]
Since $Eu-\beta Tu\in\mathbb R^{|\mathcal X_Q|}$,
\[
    \|Eu-\beta Tu\|_2
    \ge
    {1\over\sqrt{|\mathcal X_Q|}}\|Eu-\beta Tu\|_1.
\]
Finally,
\[
\begin{aligned}
    \|P_\perp A_\theta^\top u\|_2
    &=
    {1\over\sqrt{|\mathcal A|}}
    \|\beta Tu-Eu\|_2\\
    &\ge
    {1-\beta\over\sqrt{|\mathcal A|\,|\mathcal X_Q|}}
    \|u\|_2.
\end{aligned}
\]
This proves the asserted lower bound on $\sigma_{\min}(P_\perp A_\theta^\top)$. Since projection cannot increase norm,
\[
    \|A_\theta^\top u\|_2\ge \|P_\perp A_\theta^\top u\|_2\ge\gamma_\perp\|u\|_2,
\]
which is equivalent to $A_\theta A_\theta^\top\succeq\gamma_\perp^2 I_B$.
\end{proof}

\subsection{Proof of Lemma \ref{lem:BE_PL}}
\begin{proof}
Let
\[
    \boldsymbol\Psi(q_\theta):=(\Psi_1(\theta),\ldots,\Psi_B(\theta))^\top,
    \qquad
    W_{\rm BE}:=\operatorname{diag}(\omega_1^{\rm BE},\ldots,\omega_B^{\rm BE}).
\]
Then
\[
    \widehat L_{\rm BE}(\theta)
    =
    \boldsymbol\Psi(q_\theta)^\top W_{\rm BE}\boldsymbol\Psi(q_\theta).
\]
The output-space gradient is
\[
    \nabla_q\widehat L_{\rm BE}(q_\theta)
    =
    2A_\theta^\top W_{\rm BE}\boldsymbol\Psi(q_\theta).
\]
Using the Bellman constant-shift coercivity lemma,
\[
\begin{aligned}
    {1\over2}\|\nabla_q\widehat L_{\rm BE}(q_\theta)\|_2^2
    &=
    2\|A_\theta^\top W_{\rm BE}\boldsymbol\Psi(q_\theta)\|_2^2\\
    &\ge
    2\gamma_\perp^2\|W_{\rm BE}\boldsymbol\Psi(q_\theta)\|_2^2\\
    &\ge
    2\gamma_\perp^2(\omega_{\rm BE,min})^2
    \|\boldsymbol\Psi(q_\theta)\|_2^2\\
    &\ge
    2\gamma_\perp^2{(\omega_{\rm BE,min})^2\over\omega_{\rm BE,max}}
    \widehat L_{\rm BE}(\theta).
\end{aligned}
\]
Thus the Bellman-error objective is output-space PL with constant
\[
    2\gamma_\perp^2{(\omega_{\rm BE,min})^2\over\omega_{\rm BE,max}}.
\]
Applying Lemma~\ref{lem:finite_output_pl_transfer} gives the stated parameter-space PL inequality.
\end{proof}

\subsection{Proof of Lemma \ref{lem:NLL_PL}}
\begin{proof}
For $x\in\mathcal X_{\rm NLL}$, write
\[
    p_{\theta,x}:=\hat p_{Q_\theta}(\cdot\mid x)\in\Delta(\mathcal A),
    \qquad
    \widehat\pi_x\in\Delta(\mathcal A).
\]
We first record the empirical softmax lower bound used below. By Assumption~\ref{ass:nonSingularJac}(i), $\theta\mapsto Q_\theta(z)$ is continuous for every $z\in Z_Q$. Since $B_Q$ is compact and $Z_Q$ is finite, there exist finite constants $Q_-<Q_+$ such that
\[
    Q_-\le Q_\theta(z)\le Q_+,
    \qquad
    \forall z\in Z_Q,
    \quad
    \forall \theta\in B_Q.
\]
Therefore, for every $x\in\mathcal X_Q$, $a\in\mathcal A$, and $\theta\in B_Q$,
\[
    \hat p_{Q_\theta}(a\mid x)
    =
    {\exp(Q_\theta(x,a))\over \sum_{b\in\mathcal A}\exp(Q_\theta(x,b))}
    \ge
    {\exp(Q_-)\over |\mathcal A|\exp(Q_+)}
    =:p_->0.
\]
The NLL loss can be decomposed as
\[
\begin{aligned}
    \widehat L_{\rm NLL}(\theta)
    &=
    \sum_{x\in\mathcal X_{\rm NLL}}
    \omega_x^{\rm NLL}
    \left[-\sum_{a\in\mathcal A}\widehat\pi_x(a)\log p_{\theta,x}(a)\right]\\
    &=
    \sum_{x\in\mathcal X_{\rm NLL}}
    \omega_x^{\rm NLL}
    H(\widehat\pi_x)
    +
    \sum_{x\in\mathcal X_{\rm NLL}}
    \omega_x^{\rm NLL}
    \mathrm{KL}(\widehat\pi_x\|p_{\theta,x}),
\end{aligned}
\]
where $H(\widehat\pi_x)$ is the entropy of the empirical action distribution at state $x$ and does not depend on $\theta$.

The output-space gradient on the action block for state $x$ is
\[
    \nabla_{q_x}\widehat L_{\rm NLL}(q_\theta)
    =
    \omega_x^{\rm NLL}(p_{\theta,x}-\widehat\pi_x).
\]
Therefore
\[
    {1\over2}\|\nabla_q\widehat L_{\rm NLL}(q_\theta)\|_2^2
    =
    {1\over2}
    \sum_{x\in\mathcal X_{\rm NLL}}
    (\omega_x^{\rm NLL})^2
    \|p_{\theta,x}-\widehat\pi_x\|_2^2.
\]
We use the elementary inequality
\[
    \mathrm{KL}(\pi\|p)
    \le
    {1\over p_-}\|\pi-p\|_2^2
\]
which holds for every pair of probability vectors $\pi,p\in\Delta(\mathcal A)$ with $p(a)\ge p_-$ for all $a$. To verify it, write $r_a=\pi(a)/p(a)$ and use
\[
    r\log r-r+1\le (r-1)^2,
    \qquad r\ge0.
\]
Since $\sum_a p(a)(r_a-1)=0$,
\[
\begin{aligned}
    \mathrm{KL}(\pi\|p)
    &=
    \sum_a p(a)r_a\log r_a
    =
    \sum_a p(a)(r_a\log r_a-r_a+1)\\
    &\le
    \sum_a p(a)(r_a-1)^2
    =
    \sum_a {(\pi(a)-p(a))^2\over p(a)}
    \le
    {1\over p_-}\|\pi-p\|_2^2.
\end{aligned}
\]
Hence
\[
    \|p_{\theta,x}-\widehat\pi_x\|_2^2
    \ge
    p_-\,\mathrm{KL}(\widehat\pi_x\|p_{\theta,x}).
\]
It follows that
\[
\begin{aligned}
    {1\over2}\|\nabla_q\widehat L_{\rm NLL}(q_\theta)\|_2^2
    &\ge
    {\omega_{\rm NLL,min}p_-\over2}
    \sum_{x\in\mathcal X_{\rm NLL}}
    \omega_x^{\rm NLL}
    \mathrm{KL}(\widehat\pi_x\|p_{\theta,x})\\
    &\ge
    {\omega_{\rm NLL,min}p_-\over2}
    \left(\widehat L_{\rm NLL}(\theta)-\widehat L_{\rm NLL}^*\right),
\end{aligned}
\]
because $\widehat L_{\rm NLL}^*$ is at least the entropy term
$\sum_x\omega_x^{\rm NLL}H(\widehat\pi_x)$.
Applying Lemma~\ref{lem:finite_output_pl_transfer} gives
\[
    {1\over2}\|\nabla_\theta\widehat L_{\rm NLL}(\theta)\|_2^2
    \ge
    {\mu_Qp_-\omega_{\rm NLL,min}\over2}
    \left(\widehat L_{\rm NLL}(\theta)-\widehat L_{\rm NLL}^*\right).
\]
\end{proof}

\subsection{Proof of Lemma \ref{lem:zeta_empirical_pl}}
\begin{proof}
Let
\[
    z_\phi:=\boldsymbol\zeta_\phi(Z_\zeta),
    \qquad
    \bar y(\theta_2):=(\bar y_1(\theta_2),\ldots,\bar y_G(\theta_2))^\top,
    \qquad
    W_\zeta:=\operatorname{diag}(\omega_1^\zeta,\ldots,\omega_G^\zeta).
\]
Up to an additive constant independent of $\phi$,
\[
    \widehat F_\zeta(\phi;\theta_2)
    =
    \beta^2(z_\phi-\bar y(\theta_2))^\top W_\zeta(z_\phi-\bar y(\theta_2)).
\]
Consequently,
\[
    \widehat F_\zeta(\phi;\theta_2)-\widehat F_\zeta^*(\theta_2)
    \le
    \beta^2(z_\phi-\bar y(\theta_2))^\top W_\zeta(z_\phi-\bar y(\theta_2)).
\]
The output-space gradient is
\[
    \nabla_z\widehat F_\zeta(z_\phi;\theta_2)
    =
    2\beta^2 W_\zeta(z_\phi-\bar y(\theta_2)).
\]
Therefore
\[
\begin{aligned}
    {1\over2}\|\nabla_z\widehat F_\zeta(z_\phi;\theta_2)\|_2^2
    &=
    2\beta^4\|W_\zeta(z_\phi-\bar y(\theta_2))\|_2^2\\
    &\ge
    2\beta^4\omega_{\zeta,min}
    (z_\phi-\bar y(\theta_2))^\top W_\zeta(z_\phi-\bar y(\theta_2))\\
    &\ge
    2\beta^2\omega_{\zeta,min}
    \left(\widehat F_\zeta(\phi;\theta_2)-\widehat F_\zeta^*(\theta_2)\right).
\end{aligned}
\]
By the chain rule,
\[
    \nabla_\phi\widehat F_\zeta(\phi;\theta_2)
    =
    J_\zeta(\phi;Z_\zeta)^\top
    \nabla_z\widehat F_\zeta(z_\phi;\theta_2).
\]
Using Assumption~\ref{ass:nonSingularJac}(iii),
\[
\begin{aligned}
    {1\over2}\|\nabla_\phi\widehat F_\zeta(\phi;\theta_2)\|_2^2
    &\ge
    {\mu_\zeta\over2}
    \|\nabla_z\widehat F_\zeta(z_\phi;\theta_2)\|_2^2\\
    &\ge
    2\beta^2\mu_\zeta\omega_{\zeta,min}
    \left(\widehat F_\zeta(\phi;\theta_2)-\widehat F_\zeta^*(\theta_2)\right).
\end{aligned}
\]
This proves the minimization PL inequality. The maximization formulation follows by replacing $\widehat F_\zeta$ by $-\widehat F_\zeta$.
\end{proof}

\begin{lem}[Automatic control of NLL--Bellman gradient interactions]
\label{lem:auto_cancellation_control}
Let $P_\parallel$ and $P_\perp$ be the orthogonal projections onto $\mathcal U$ and $\mathcal U^\perp$, respectively, and define
\[
    C_\parallel:=\sup_{\theta\in B_Q}\|P_\parallel A_\theta^\top\|_2,
    \qquad
    \kappa_\perp:={C_\parallel\over \gamma_\perp}.
\]
Then there exist constants $a_{\rm NLL},a_{\rm BE}>0$ such that, for all $\theta\in B_Q$,
\[
    \|\nabla_q\widehat L_{\rm NLL}(q_\theta)
      +\nabla_q\widehat L_{\rm BE}(q_\theta)\|_2^2
    \ge
    a_{\rm NLL}\|\nabla_q\widehat L_{\rm NLL}(q_\theta)\|_2^2
    +a_{\rm BE}\widehat L_{\rm BE}(\theta).
\]
In particular, the combined $Q$-objective does not require a separate nonnegative-gradient-cross-term assumption.
\end{lem}

\subsection{Proof of Lemma \ref{lem:auto_cancellation_control}}
\begin{proof}
For every $x\in\mathcal X_{\rm NLL}$,
\[
    \sum_{a\in\mathcal A}
    \left(p_{\theta,x}(a)-\widehat\pi_x(a)\right)=0.
\]
The NLL gradient is zero on state blocks not belonging to $\mathcal X_{\rm NLL}$. Therefore the full output-space NLL gradient belongs to $\mathcal U$, and
\[
    P_\perp\nabla_q\widehat L_{\rm NLL}(q_\theta)=0.
\]

The constant $C_\parallel$ is finite because $\theta\mapsto A_\theta$ is continuous on the compact ball $B_Q$. Let
\[
    g_{\rm NLL}:=\nabla_q\widehat L_{\rm NLL}(q_\theta),
    \qquad
    g_{\rm BE}:=\nabla_q\widehat L_{\rm BE}(q_\theta)
    =2A_\theta^\top W_{\rm BE}\boldsymbol\Psi(q_\theta).
\]
Decompose
\[
    g_{\rm BE}=g_\parallel+g_\perp,
    \qquad
    g_\parallel:=P_\parallel g_{\rm BE},
    \qquad
    g_\perp:=P_\perp g_{\rm BE}.
\]
Since $g_{\rm NLL}\in\mathcal U$ and $g_\perp\in\mathcal U^\perp$,
\[
    \|g_{\rm NLL}+g_{\rm BE}\|_2^2
    =
    \|g_{\rm NLL}+g_\parallel\|_2^2+
    \|g_\perp\|_2^2.
\]
By the definitions of $C_\parallel$ and $\gamma_\perp$,
\[
    \|g_\parallel\|_2
    \le
    2C_\parallel\|W_{\rm BE}\boldsymbol\Psi(q_\theta)\|_2,
    \qquad
    \|g_\perp\|_2
    \ge
    2\gamma_\perp\|W_{\rm BE}\boldsymbol\Psi(q_\theta)\|_2.
\]
Hence
\[
    \|g_\parallel\|_2\le \kappa_\perp\|g_\perp\|_2.
\]
For arbitrary vectors $a,b,c$ with $a,b$ in the same Euclidean space, $c$ orthogonal to that space, and $\|b\|\le\kappa\|c\|$, the following deterministic inequality holds:
\[
    \|a+b\|^2+\|c\|^2
    \ge
    {1\over2(1+\kappa^2)}\|a\|^2
    +{1\over2}\|c\|^2.
\]
If $\kappa=0$, then $b=0$ and the displayed inequality is immediate. Suppose $\kappa>0$. Since $\|a+b\|\ge \max\{\|a\|-\|b\|,0\}$ and $\|b\|\le\kappa\|c\|$, the left-hand side is bounded below by
\[
    \max\{\|a\|-t,0\}^2+{t^2\over\kappa^2}
\]
for some $t\ge0$. Minimizing this expression over $t\ge0$ gives $\|a\|^2/(1+\kappa^2)$. The left-hand side is also at least $\|c\|^2$. Averaging these two lower bounds gives the displayed inequality.

Applying the inequality with $a=g_{\rm NLL}$, $b=g_\parallel$, $c=g_\perp$, and $\kappa=\kappa_\perp$ gives
\[
    \|g_{\rm NLL}+g_{\rm BE}\|_2^2
    \ge
    {1\over2(1+\kappa_\perp^2)}
    \|g_{\rm NLL}\|_2^2
    +{1\over2}\|g_\perp\|_2^2.
\]
Finally,
\[
\begin{aligned}
    \|g_\perp\|_2^2
    &\ge
    4\gamma_\perp^2\|W_{\rm BE}\boldsymbol\Psi(q_\theta)\|_2^2\\
    &\ge
    4\gamma_\perp^2(\omega_{\rm BE,min})^2
    \|\boldsymbol\Psi(q_\theta)\|_2^2\\
    &\ge
    4\gamma_\perp^2{(\omega_{\rm BE,min})^2\over\omega_{\rm BE,max}}
    \widehat L_{\rm BE}(\theta).
\end{aligned}
\]
Substitution proves the lemma.
\end{proof}

\subsection{Proof of Theorem \ref{thm:bothPL}}\label{sec:PfofbothPL}
\begin{proof}
By Lemma~\ref{lem:NLL_PL} before applying the empirical-output transfer step, the empirical NLL has output-space PL constant
\[
    \alpha_{\rm NLL}={p_-\omega_{\rm NLL,min}\over2},
\]
meaning
\[
    {1\over2}\|\nabla_q\widehat L_{\rm NLL}(q_\theta)\|_2^2
    \ge
    \alpha_{\rm NLL}
    \left(\widehat L_{\rm NLL}(\theta)-\widehat L_{\rm NLL}^*\right).
\]
Lemma~\ref{lem:auto_cancellation_control} gives
\[
\begin{aligned}
    {1\over2}
    \|\nabla_q(\widehat L_{\rm NLL}+\widehat L_{\rm BE})(q_\theta)\|_2^2
    &\ge
    {1\over4(1+\kappa_\perp^2)}
    \|\nabla_q\widehat L_{\rm NLL}(q_\theta)\|_2^2\\
    &\quad+
    \gamma_\perp^2{(\omega_{\rm BE,min})^2\over\omega_{\rm BE,max}}
    \widehat L_{\rm BE}(\theta)\\
    &\ge
    {\alpha_{\rm NLL}\over2(1+\kappa_\perp^2)}
    \left(\widehat L_{\rm NLL}(\theta)-\widehat L_{\rm NLL}^*\right)\\
    &\quad+
    \gamma_\perp^2{(\omega_{\rm BE,min})^2\over\omega_{\rm BE,max}}
    \widehat L_{\rm BE}(\theta).
\end{aligned}
\]
Since $\widehat L_{\rm BE}(\theta)\ge \widehat L_{\rm BE}(\theta)-\widehat L_{\rm BE}^*$ and
\[
    \widehat R_Q^*
    =
    \inf_\theta\{\widehat L_{\rm NLL}(\theta)+\widehat L_{\rm BE}(\theta)\}
    \ge
    \widehat L_{\rm NLL}^*+\widehat L_{\rm BE}^*,
\]
we have
\[
    \left(\widehat L_{\rm NLL}(\theta)-\widehat L_{\rm NLL}^*\right)
    +
    \left(\widehat L_{\rm BE}(\theta)-\widehat L_{\rm BE}^*\right)
    \ge
    \widehat R_Q(\theta)-\widehat R_Q^*.
\]
Therefore
\[
    {1\over2}
    \|\nabla_q\widehat R_Q(q_\theta)\|_2^2
    \ge
    \alpha_{\widehat R}
    \left(\widehat R_Q(q_\theta)-\widehat R_Q^*\right),
\]
where
\[
    \alpha_{\widehat R}
    =
    \min\left\{
        {p_-\omega_{\rm NLL,min}\over4(1+\kappa_\perp^2)},
        \gamma_\perp^2{(\omega_{\rm BE,min})^2\over\omega_{\rm BE,max}}
    \right\}.
\]
Applying Lemma~\ref{lem:finite_output_pl_transfer} to $\widehat R_Q(q_\theta)$ gives
\[
    {1\over2}\|\nabla_\theta\widehat R_Q(\theta)\|_2^2
    \ge
    \mu_Q\alpha_{\widehat R}
    \left(\widehat R_Q(\theta)-\widehat R_Q^*\right).
\]
Thus one admissible PL constant in Theorem~\ref{thm:bothPL} is
\[
    c_{\widehat R}
    =
    \mu_Q
    \min\left\{
        {p_-\omega_{\rm NLL,min}\over4(1+\kappa_\perp^2)},
        \gamma_\perp^2{(\omega_{\rm BE,min})^2\over\omega_{\rm BE,max}}
    \right\}.
\]
The auxiliary statement is exactly Lemma~\ref{lem:zeta_empirical_pl}.
\end{proof}

\subsection{Proof of Theorem \ref{prop:linConvergence} (empirical convergence and population excess-risk control)}
\label{sec:ProofLinConv}

\subsubsection*{1. Empirical optimization error}
By Theorem~\ref{thm:bothPL}, the empirical $Q$-objective satisfies a PL inequality in $\theta$ with some constant $c_{\widehat R}>0$. By Lemma~\ref{lem:zeta_empirical_pl}, for every fixed $\theta_2$, the inner auxiliary loss satisfies PL in $\phi$, and therefore the maximization objective $-\widehat F_\zeta(\phi;\theta_2)$ satisfies the corresponding PL inequality for ascent. The inner objective is a negative squared loss in the finite vector $\boldsymbol\zeta_\phi(Z_\zeta)$, hence it is concave in that empirical output vector. Under Assumption~\ref{ass:nonSingularJac}(iii), the empirical-output parametrization transfers this output-space curvature to the parameterized inner problem in the PL sense. Consequently, the empirical minimax objective has the two-sided PL structure required by Theorem~3.3 of \citet{yang2020global}.

Let $f(\theta,\phi)$ denote the empirical minimax objective, and define
\[
    g(\theta):=\max_\phi f(\theta,\phi),
    \qquad
    g^*:=\min_\theta g(\theta).
\]
Let
\[
    a(\theta):=g(\theta)-g^*,
    \qquad
    b(\theta,
    \phi):=g(\theta)-f(\theta,\phi).
\]
Both $a(\theta)$ and $b(\theta,\phi)$ are nonnegative and vanish at a minimax point. Define the potential
\[
    P_T:=a(\widehat\theta_T)+\alpha b(\widehat\theta_T,\widehat\phi_T),
\]
where $\alpha=1/10$ as in \citet{yang2020global}. Under the stepsize choices and bounded-variance assumptions stated in Theorem~3.3 of \citet{yang2020global}, there exist constants $\nu,\gamma>0$ such that
\[
    P_T\le {\nu\over \gamma+T}.
\]
Since $a(\widehat\theta_T)\le P_T$, we obtain
\[
    \widehat R_Q(\widehat\theta_T)-\widehat R_Q^*
    =a(\widehat\theta_T)
    \le {\nu\over \gamma+T}.
\]
This proves the empirical excess-risk bound.

Because $\widehat R_Q$ is twice continuously differentiable as a function of the empirical evaluation vector and satisfies the PL inequality on $B_Q$, the local PL--quadratic-growth equivalence for $C^2$ functions gives a constant $c_{\rm qg}>0$ such that
\[
    \widehat R_Q(\theta)-\widehat R_Q^*
    \ge
    c_{\rm qg}\operatorname{dist}(\theta,S_{\mathcal D_N})^2,
    \qquad
    \forall\theta\in B_Q,
\]
where
\[
    S_{\mathcal D_N}:=\arg\min_{\theta\in B_Q}\widehat R_Q(\theta).
\]
Applying this inequality at $\widehat\theta_T$ gives
\[
    \operatorname{dist}(\widehat\theta_T,S_{\mathcal D_N})^2
    \le
    {1\over c_{\rm qg}}
    (\widehat R_Q(\widehat\theta_T)-\widehat R_Q^*)
    \le
    {\nu\over c_{\rm qg}(\gamma+T)}.
\]

\subsubsection*{2. Population excess-risk control}
The stability result of \citet{kang2025stabilitygeneralizationbellmanresiduals} gives, for the same SGDA iterates, the full population ERM-DDC/IRL excess-risk bound
\[
    \mathbb E\left[
      \mathcal R_{\exp}(Q_{\widehat\theta_T})-\mathcal R_{\exp}(Q^*)
    \right]
    \le
    (1+L/\rho)G\varepsilon_T+{\nu\over\gamma+T},
\]
where
\[
    \varepsilon_T
    =
    O\!\left((c_2+T)^{-\min\{1/2,3cc_1/8\}}\right)+{C\over N}.
\]
Substituting this expression proves the population excess-risk statement in Theorem~\ref{prop:linConvergence}. This step uses only the empirical optimization bound above and the cited stability/generalization theorem; it does not require a population NTK lower bound.
\QED

\subsection{Quantitative identification for the full ERM-DDC/IRL risk}
\label{app:quantitative_identification}

Theorem~\ref{thm:mainopt} gives exact identification: the population ERM-DDC/IRL risk is minimized by $Q^*$ on the expert support.  The next result is the quantitative version needed for rates.  It is not a residual-propagation statement for a pure Bellman-residual objective.  It uses the two components of the DDC/IRL risk together: the NLL term fixes within-state action-value differences, and the anchor-action Bellman term fixes the state-wise additive normalization.

The only support issue that enters this argument is whether the anchor-action Bellman equation depends on next states that are invisible under the expert state distribution.  In the entropy-regularized/logit DDC model, this overlap is not an additional assumption.  It follows from full-support choice probabilities.

Let $d_S^*$ denote the discounted expert state occupancy measure used by the population risk, and define
\[
    \|g\|_{2,S}^2:=\mathbb E_{s\sim d_S^*}[g(s)^2].
\]
For the anchor action $a_s$, define the anchor-next-state measure
\[
    \nu_A(B):=\int P(B\mid s,a_s)\,d_S^*(s).
\]

\begin{lem}[Anchor-transition overlap from logit full support]
\label{lem:anchor_overlap_logit}
Suppose $\mathcal A$ is finite and the expert policy is the entropy-regularized/logit policy induced by $Q^*$:
\[
    \pi^*(a\mid s)={\exp(Q^*(s,a))\over\sum_{b\in\mathcal A}\exp(Q^*(s,b))}.
\]
Then $\nu_A\ll d_S^*$.  If, in addition, $\|Q^*\|_\infty\le B_*$, then
\[
    \pi^*(a_s\mid s)\ge \underline p_A:={e^{-2B_*}\over |\mathcal A|}
    \qquad d_S^*\text{-a.s.},
\]
 and
\[
    \nu_A(B)\le {1\over \beta\underline p_A}d_S^*(B),
    \qquad \forall B.
\]
Equivalently,
\[
    \left\|{d\nu_A\over d d_S^*}\right\|_\infty
    \le {1\over \beta\underline p_A}
    \le {|\mathcal A|e^{2B_*}\over\beta}.
\]
\end{lem}

\begin{proof}
The discounted expert state occupancy satisfies the flow identity
\[
    d_S^*(B)
    =
    (1-\beta)\nu_0(B)
    +
    \beta\int \sum_{a\in\mathcal A}\pi^*(a\mid s)P(B\mid s,a)\,d_S^*(s).
\]
Therefore,
\[
    d_S^*(B)
    \ge
    \beta\int \pi^*(a_s\mid s)P(B\mid s,a_s)\,d_S^*(s).
\]
If $d_S^*(B)=0$, the right-hand side must be zero.  Since the logit policy satisfies $\pi^*(a_s\mid s)>0$ $d_S^*$-almost surely, this implies $P(B\mid s,a_s)=0$ for $d_S^*$-almost every $s$, and hence
\[
    \nu_A(B)=\int P(B\mid s,a_s)\,d_S^*(s)=0.
\]
Thus $\nu_A\ll d_S^*$.

If $\|Q^*\|_\infty\le B_*$, then for every $s$ and $a$,
\[
    \pi^*(a\mid s)
    ={\exp(Q^*(s,a))\over\sum_{b\in\mathcal A}\exp(Q^*(s,b))}
    \ge {e^{-B_*}\over |\mathcal A|e^{B_*}}
    ={e^{-2B_*}\over |\mathcal A|}
    =\underline p_A.
\]
Substituting this lower bound into the previous display gives
\[
    d_S^*(B)
    \ge
    \beta\underline p_A\int P(B\mid s,a_s)\,d_S^*(s)
    =
    \beta\underline p_A\nu_A(B),
\]
which proves the density-ratio bound.
\end{proof}

\begin{lem}[Quantitative identification of the full ERM-DDC/IRL risk]
\label{lem:erm_ddc_quant_identification}
Let $B_Q\subset\mathbb R^{d_\theta}$ be the compact finite-dimensional parameter region used in Assumption~\ref{ass:nonSingularJac}, and suppose $Q_{\theta^\dagger}=Q^*$ for some $\theta^\dagger\in B_Q$.  Suppose $\mathcal A$ is finite and the maps $\theta\mapsto Q_\theta$ are twice continuously differentiable on $B_Q$ as $L_2(d^*)$ functions, with uniformly bounded outputs.  Then there exists a finite constant $C_{\rm id}<\infty$ such that, for every $\theta\in B_Q$,
\[
    \|Q_\theta-Q^*\|_{2,*}^2
    \le
    C_{\rm id}
    \left[
    \mathcal R_{\exp}(Q_\theta)-\mathcal R_{\exp}(Q^*)
    \right].
\]
The constant $C_{\rm id}$ is the local conditioning constant of the identified DDC system.  It depends on the curvature of the NLL term, the anchor-action Bellman normalization, the compact parameter region, and the logit full-support lower bound, but it does not require a finite state-space or finite-support assumption.
\end{lem}

\begin{proof}
We prove that the full ERM-DDC/IRL risk has quadratic growth in the $L_2(d^*)$ function metric around $Q^*$, and then extend the bound from a local neighborhood to all of the compact parameter region.

Let
\[
    h(s,a):=Q_\theta(s,a)-Q^*(s,a).
\]
The excess NLL term equals
\[
    \mathbb E_{s\sim d_S^*}
    \left[\mathrm{KL}\bigl(\pi^*(\cdot\mid s)\,\|\,\pi_{Q_\theta}(\cdot\mid s)\bigr)\right].
\]
The Hessian of the state-wise log-sum-exp loss at $Q^*$ is the covariance operator under $\pi^*(\cdot\mid s)$:
\[
    v\mapsto \operatorname{Var}_{a\sim\pi^*(\cdot\mid s)}(v(a)).
\]
Therefore the NLL curvature is strictly positive on within-state action-value contrasts and has exactly the state-wise constant directions as its null directions.

Next consider the anchor Bellman residual
\[
    b_Q(s):=\mathcal TQ(s,a_s)-Q(s,a_s).
\]
At $Q^*$, $b_{Q^*}(s)=0$.  The first derivative of $b_Q$ at $Q^*$ in direction $h$ is
\[
    Db_{Q^*}[h](s)
    =
    \beta\mathbb E\!\left[\sum_{a\in\mathcal A}\pi^*(a\mid s')h(s',a)\mid s,a_s\right]
    -h(s,a_s).
\]
The second variation of the anchor Bellman-error term is therefore proportional to
\[
    \mathbb E_{s\sim d_S^*}
    \left[
    \pi^*(a_s\mid s)\,Db_{Q^*}[h](s)^2
    \right].
\]
Since $Q^*$ is bounded and $\mathcal A$ is finite, Lemma~\ref{lem:anchor_overlap_logit} gives both $\pi^*(a_s\mid s)\ge \underline p_A>0$ and $\nu_A\ll d_S^*$.

We now characterize the null space of the total second variation.  If the NLL second variation is zero, then for $d_S^*$-almost every $s$,
\[
    h(s,a)=c(s),\qquad \forall a\in\mathcal A,
\]
for some bounded state function $c$.  If the anchor Bellman second variation is also zero, then
\[
    c(s)=\beta\mathbb E[c(s')\mid s,a_s]
    \qquad d_S^*\text{-a.s.}
\]
By Lemma~\ref{lem:anchor_overlap_logit}, the conditional expectation on the right-hand side only depends on $c$ on the $d_S^*$-support.  Taking essential suprema over this support gives
\[
    \|c\|_{\infty,d_S^*}
    \le
    \beta\|c\|_{\infty,d_S^*}.
\]
Because $\beta<1$, this implies $c=0$ $d_S^*$-almost surely.  Hence the only function direction with zero total second variation is $h=0$ in $L_2(d^*)$.

The image of the finite-dimensional parameter region under the local derivative map is finite-dimensional.  On the quotient that identifies parameter directions inducing the same $L_2(d^*)$ function direction, the unit sphere is compact.  Since the total second variation is continuous and has no nonzero null direction on this sphere, it has a strictly positive minimum eigenvalue.  Consequently, there exist constants $r>0$ and $\lambda_{\rm id}>0$ such that whenever $\|Q_\theta-Q^*\|_{2,*}\le r$,
\[
    \mathcal R_{\exp}(Q_\theta)-\mathcal R_{\exp}(Q^*)
    \ge
    \lambda_{\rm id}\|Q_\theta-Q^*\|_{2,*}^2.
\]
This is the local quadratic-growth, or quantitative-identification, bound.

It remains to extend the inequality to the whole compact region $B_Q$.  For the set
\[
    K_r:=\{\theta\in B_Q:\|Q_\theta-Q^*\|_{2,*}\ge r\},
\]
continuity and the exact identification result of Theorem~\ref{thm:mainopt} imply
\[
    \delta_r:=\inf_{\theta\in K_r}
    \left[\mathcal R_{\exp}(Q_\theta)-\mathcal R_{\exp}(Q^*)\right]>0.
\]
Also, compactness and boundedness imply
\[
    M_r:=\sup_{\theta\in B_Q}\|Q_\theta-Q^*\|_{2,*}^2<\infty.
\]
Thus, on $K_r$,
\[
    \|Q_\theta-Q^*\|_{2,*}^2
    \le
    {M_r\over\delta_r}
    \left[\mathcal R_{\exp}(Q_\theta)-\mathcal R_{\exp}(Q^*)\right].
\]
Combining this global-away-from-the-minimizer bound with the local quadratic-growth bound proves the claim with
\[
    C_{\rm id}:=\max\left\{\lambda_{\rm id}^{-1},{M_r\over\delta_r}\right\}<\infty.
\]
\end{proof}

\begin{lem}[Lipschitzness of $\mathcal{R}_{\text{emp}}$]
\label{lem:EmpLipschitz} 
Suppose that $\left\|\nabla_\theta Q_\theta\right\|_2 \leq B$ and $M=:\left(R_{\max }+\beta V_{\max }+Q_{\max }\right)$ is defined as above. Then $R_{\text{emp}}\left({\theta}\right)$ is Lipschitz continuous with respect to ${\theta}$ with constant $(1+2 M(1+\beta)) B$.
\end{lem}
\begin{proof}
       
If $\nabla_\theta \mathcal{L}_{\mathrm{NLL}}\left(Q_\theta\right)(s, a)$ and $\nabla_\theta \mathcal{L}_{\mathrm{BE}}\left(Q_\theta\right)(s, a)$ is bounded, then
$$
\nabla_\theta R_{\text{emp}}(\theta)=\mathbb{E}_{(s, a)}\left[\nabla_\theta \mathcal{L}_{\mathrm{NLL}}\left(Q_\theta\right)(s, a)+ \mathbf{1}_{a=a_s} \nabla_\theta \mathcal{L}_{\mathrm{BE}}\left(Q_\theta\right)(s, a)\right]
$$
which implies that
$$
\left\|\nabla_\theta R_{\exp }(\theta)\right\|_2 \leq \mathbb{E}_{(s, a)}\left[\left\|\nabla_\theta \mathcal{L}_{\mathrm{NLL}}\left(Q_\theta\right)(s, a)\right\|_2+ \mathbf{1}_{a=a_s}\left\|\nabla_\theta \mathcal{L}_{\mathrm{BE}}\left(Q_\theta\right)(s, a)\right\|_2\right]
$$
Now note that
$$
\nabla_\theta \mathcal{L}_{\mathrm{BE}}(s, a)=2 \Psi_\theta(s, a) \nabla_\theta \Psi_\theta(s, a)
$$ 
where $\nabla_\theta \Psi_\theta(s, a)=\nabla_\theta\left(\mathcal{T} Q_\theta(s, a)\right)-\nabla_\theta Q_\theta(s, a)$. Since gradient of $Q_\theta$ is bounded, we have $\nabla_\theta\left(\mathcal{T} Q_\theta(s, a)\right)=\nabla_\theta\left(r(s, a)+\beta \mathbb{E}_{s^{\prime}}\left[V_{Q_\theta}\left(s^{\prime}\right)\right]\right)=\beta \mathbb{E}_{s^{\prime}}\left[\nabla_\theta V_{Q_\theta}\left(s^{\prime}\right)\right]$, where $\nabla_\theta V_{Q_\theta}\left(s^{\prime}\right)=\sum_{c \in \mathcal{A}} \frac{\exp \left(Q_\theta\left(s^{\prime}, c\right)\right)}{\sum_{b \in \mathcal{A}} \exp \left(Q_\theta\left(s^{\prime}, b\right)\right)} \nabla_\theta Q_\theta\left(s^{\prime}, c\right)=\sum_{c \in \mathcal{A}} \hat{p}_{Q_\theta}\left(c \mid s^{\prime}\right) \nabla_\theta Q_\theta\left(s^{\prime}, c\right)\le B$ ($\because$ $\left\|\nabla_\theta Q_\theta\right\|_2 \leq B$ due to Assumption \ref{ass:nonSingularJac}). Now, applying Jensen's inequality for norms $(\|\mathbb{E}[X] \| \leq \mathbb{E}[\|X\|])$ :
$$
\left\|\nabla_\theta\left(\mathcal{T} Q_\theta(s, a)\right)\right\|_2=\left\|\beta \mathbb{E}_{s^{\prime}}\left[\nabla_\theta V_{Q_\theta}\left(s^{\prime}\right)\right]\right\|_2 \leq \beta \mathbb{E}_{s^{\prime}}\left[\left\|\nabla_\theta V_{Q_\theta}\left(s^{\prime}\right)\right\|_2\right] \leq \beta B
$$
Therefore 
$$\left\|\nabla_\theta \mathcal{L}_{\mathrm{BE}}(s, a)\right\|_2 \leq 2\left(R_{\max }+\beta V_{\max }+Q_{\max }\right)(1+\beta) B = 2M(1+\beta)B
$$
Also, note that $$
\mathcal{L}_{\mathrm{NLL}}\left(Q_\theta\right)(s, a)=-Q_\theta(s, a)+\log \sum_{b \in \mathcal{A}} \exp \left(Q_\theta(s, b)\right)
$$
where its gradient is
$\nabla_\theta \mathcal{L}_{\mathrm{NLL}}(s, a)=-\nabla_\theta Q_\theta(s, a)+\sum_{b \in \mathcal{A}} \hat{p}_{Q_\theta}(b \mid s) \nabla_\theta Q_\theta(s, b)$. By Assumption 6.2, $\left\|\nabla_\theta Q_\theta(s, a)\right\|_2 \leq B$. Since $\sum_b \hat{p}_{Q_\theta}(b \mid s)=1$, the triangle inequality yields 
$$
\left\|\nabla_\theta \mathcal{L}_{\mathrm{NLL}}(s, a)\right\|_2 \leq\left\|\nabla_\theta Q_\theta(s, a)\right\|_2+\sum_{b \in \mathcal{A}} \hat{p}_{Q_\theta}(b \mid s)\left\|\nabla_\theta Q_\theta(s, b)\right\|_2 \leq B+\sum_b \hat{p}_{Q_\theta}(b \mid s) B=2 B
$$ 
Since $\mathbb{E}\left[\mathbf{1}_{a=a_s}\right] \leq 1$, we have
$$
\left\|\nabla_\theta R_{\exp }(\theta)\right\|_2\le 2 M(1+\beta) B + B = (1+2 M(1+\beta))B \text { for all } \theta \in \Theta
$$
By the Mean Value Theorem, for any $\theta_1, \theta_2 \in \Theta$, there exists $\tilde{\theta}$ on the line segment connecting $\theta_1$ and $\theta_2$ such that $R_{\exp }\left(\theta_1\right)-R_{\exp }\left(\theta_2\right)=\nabla_\theta R_{\exp }(\tilde{\theta}) \cdot\left(\theta_1-\theta_2\right)$. Taking the absolute value and applying the Cauchy-Schwarz inequality:
$$
\left|R_{\exp }\left(\theta_1\right)-R_{\exp }\left(\theta_2\right)\right| \leq\left\|\nabla_\theta R_{\exp }(\tilde{\theta})\right\|_2\left\|\theta_1-\theta_2\right\|_2
$$
Since $\left\|\nabla_\theta R_{\exp }(\tilde{\theta})\right\|_2 \leq L$, we conclude:
$$
\left|R_{\exp }\left(\theta_1\right)-R_{\exp }\left(\theta_2\right)\right| \leq L| | \theta_1-\theta_2 \|_2
$$
Therefore, $R_{\exp}\left({\theta}\right)$ is Lipschitz continuous with constant $(1+2 M(1+\beta)) B$. With the exactly same logic introduced in the Lemma \ref{sec:PfofbothPL}, this implies that $R_{\text{emp}}\left({\theta}\right)$ is also Lipschitz continuous with respect to ${\theta}$ with constant $(1+2 M(1+\beta)) B$.
\end{proof}

\begin{lem}[Lipschitz smoothness of $\mathcal{R}_{\text{emp}}$]\label{lem:lipsSmooth}
We continue from what we derived during the proof of the empirical Lipschitz lemma. If $\nabla_\theta \mathcal{L}_{\mathrm{NLL}}\left(Q_\theta\right)(s, a)$ and $\nabla_\theta \mathcal{L}_{\mathrm{BE}}\left(Q_\theta\right)(s, a)$ is bounded, then
$$
\nabla^2_\theta R_{\mathrm{emp}}(\theta)=\mathbb{E}_{(s, a)}\left[\nabla^2_\theta \mathcal{L}_{\mathrm{NLL}}\left(Q_\theta\right)(s, a)+ \mathbf{1}_{a=a_s} \nabla^2_\theta \mathcal{L}_{\mathrm{BE}}\left(Q_\theta\right)(s, a)\right]
$$
Recall that $\nabla_\theta \mathcal{L}_{\mathrm{NLL}}(s, a)=-\nabla_\theta Q_\theta(s, a)+\nabla_\theta V_{Q_\theta}(s)$ where $\nabla_\theta V_{Q_\theta}(s)=\sum_{b \in \mathcal{A}} \hat{p}_{Q_\theta}(b \mid s) \nabla_\theta Q_\theta(s, b)$ and $\left\|\nabla_\theta \mathcal{L}_{\mathrm{NLL}}(s, a)\right\|_2\le 2B$. With $\left\|\nabla_\theta^2 Q_\theta(s, a)\right\|_2 \leq B_2$, we have
\begin{align}
    \nabla^2_\theta V_{Q_\theta}(s)=\sum_b \nabla_\theta p_\theta(b \mid s) \nabla_\theta Q_\theta(s, b)^{\top}+\sum_b p_\theta(b \mid s) \nabla_\theta^2 Q_\theta(s, b) \notag
\end{align}
where $\nabla_\theta p_\theta(b \mid s)=p_\theta(b \mid s)\left(\nabla_\theta Q_\theta(s, b)-\sum_c p_\theta(c \mid s) \nabla_\theta Q_\theta(s, c)\right)$. Distributing, we have 
$$
\nabla_\theta^2 V_{Q_\theta}(s)=\sum_b p_\theta(b \mid s) \nabla_\theta^2 Q_\theta(s, b)+\sum_b p_\theta(b \mid s)\left(1-p_\theta(b \mid s)\right) \nabla_\theta Q_\theta(s, b) \nabla_\theta Q_\theta(s, b)^{\top}
$$
Now we have
$$
\nabla_\theta^2 \mathcal{L}_{\mathrm{NLL}}(s,a)=T_1+T_2+T_3
$$
where 
\\
$T_1=-\nabla_\theta^2 Q_\theta(s, a)$, 
\\
$T_2=\sum_b p_\theta(b \mid s) \nabla_\theta^2 Q_\theta(s, b)$ and 
\\
$T_3=\sum_b p_\theta(b \mid s)\left(1-p_\theta(b \mid s)\right) \nabla_\theta Q_\theta(s, b) \nabla_\theta Q_\theta(s, b)^{\top}$.
\\
For $T_1$ and $T_2$,
$$
\left\|T_1\right\|_2=\left\|\nabla_\theta^2 Q_\theta(s, a)\right\|_2 \leq B_2
$$
$$
\left\|T_2\right\|_2 \leq \sum_b p_\theta(b \mid s)\left\|\nabla_\theta^2 Q_\theta(s, b)\right\|_2 \leq \sum_b p_\theta(b \mid s) B_2=B_2
$$
For $T_3$, from
$\left\|\nabla_\theta Q_\theta(s, b) \nabla_\theta Q_\theta(s, b)^{\top}\right\|_2=\left\|\nabla_\theta Q_\theta(s, b)\right\|_2^2 \leq B^2$, we have
$$
\left\|T_3\right\|_2 \leq \sum_b p_\theta(b \mid s)\left(1-p_\theta(b \mid s)\right) B^2 \leq |\mathcal{A}| \cdot \frac{1}{4} B^2=\frac{|\mathcal{A}|}{4} B^2
$$
Therefore,
$$
\left\|\nabla_\theta^2 \mathcal{L}_{\mathrm{NLL}}\right\|_2 \leq\left\|T_1\right\|_2+\left\|T_2\right\|_2+\left\|T_3\right\|_2 \le 2B_2 + \frac{|\mathcal{A}|}{4}B^2
$$\;
\\
\noindent Also, continuing from
$$
\nabla_\theta \mathcal{L}_{\mathrm{BE}}(s, a)=2 \Psi_\theta(s, a) \nabla_\theta \Psi_\theta(s, a)
$$
where $\nabla_\theta \Psi_\theta(s, a)=\nabla_\theta\left(\mathcal{T} Q_\theta(s, a)\right)-\nabla_\theta Q_\theta(s, a)=\beta \mathbb{E}_{s^{\prime}}\left[\nabla_\theta V_{Q_\theta}\left(s^{\prime}\right)\right]-\nabla_\theta Q_\theta(s, a)$, we have
$$
\nabla_\theta^2 \mathcal{L}_{\mathrm{BE}}(s,a)=2\left(\nabla_\theta \Psi_\theta(s,a)\right)\left(\nabla_\theta \Psi_\theta(s,a)\right)^{\top}+2 \Psi_\theta(s,a) \nabla_\theta^2 \Psi_\theta(s,a).
$$
From 
$
\nabla_\theta^2 \Psi_\theta(s, a)=\beta \mathbb{E}_{s^{\prime}}\left[\nabla_\theta^2 V_{Q_\theta}\left(s^{\prime}\right)\mid s,a\right]-\nabla_\theta^2 Q_\theta(s, a)
$ where $
\nabla_\theta^2 V_{Q_\theta}(s)=T_2+T_3
$.
Therefore
$$
\begin{aligned}
\left\|\nabla_\theta^2 \mathcal{L}_{\mathrm{BE}}\right\|_2 & \leq 2\left\|\nabla_\theta \Psi_\theta\right\|_2^2+2\left|\Psi_\theta\right|\left\|\nabla_\theta^2 \Psi_\theta\right\|_2 \\
& \leq 2(1+\beta)^2 B^2+2 M\left[(1+2 \beta) B_2+\frac{\beta|\mathcal{A}|}{4} B^2\right]
\end{aligned}
$$
Therefore, we get
\begin{align}
    \left\|\nabla_\theta^2 R_{\operatorname{exp}}(\theta)\right\|_2 &\le \left\|\nabla_\theta^2 \mathcal{L}_{\mathrm{NLL}}\right\|_2 +  \left\|\nabla_\theta^2 \mathcal{L}_{\mathrm{BE}}\right\|_2 \notag
    \\
    & \le 2 B_2+\frac{|\mathcal{A}|}{4} B^2+\left[2(1+2 \beta) M B_2+\left(2(1+\beta)^2+\frac{\beta|\mathcal{A}|}{2} M\right) B^2\right]
\end{align}

\noindent For $\nabla_{\zeta}^2 R_{\text{exp}}(\zeta)$, it is a simple squared function and therefore the Lipschitz smoothness in $\beta$ is trivial. Now with the same mean value theorem argument as in the empirical Lipschitz lemma, we have that 
$$
\left\|\nabla_\zeta R_{\text {exp }}\left(\theta_1\right)-\nabla_\theta R_{\text {exp }}\left(\theta_2\right)\right\|_2 \leq L_{\theta}\left\|\theta_1-\theta_2\right\|_2
$$
$$
\left\|\nabla_\zeta R_{\text {exp }}\left(\zeta_1\right)-\nabla_\zeta R_{\text {exp }}\left(\zeta_2\right)\right\|_2 \leq L_{\zeta}\left\|\zeta_1-\zeta_2\right\|_2
$$
where $L_{\theta} = 2 B_2+\frac{|\mathcal{A}|}{4} B^2+\left[2(1+2 \beta) M B_2+\left(2(1+\beta)^2+\frac{\beta|\mathcal{A}|}{2} M\right) B^2\right]$ and $L_{\zeta} =  \beta$. Now with the exactly same logic introduced in the Lemma \ref{sec:PfofbothPL}, this implies that $R_{\text{emp}}$ is also Lipschitz smooth with respect to $\boldsymbol{\theta}$ with constant $L_{\theta} = 2 B_2+\frac{|\mathcal{A}|}{4} B^2+\left[2(1+2 \beta) M B_2+\left(2(1+\beta)^2+\frac{\beta|\mathcal{A}|}{2} M\right) B^2\right]$ and Lipschitz smooth with respect to $\zeta$ with constant $L_{\zeta} = \beta$.
\end{lem}

% ===== End input: app_proofs.tex =====

%\section*{Acknowledgments}

\end{appendices}
\end{document}